\documentclass{article} % For LaTeX2e

\PassOptionsToPackage{table}{xcolor}

\usepackage{iclr2027_conference,times}

\usepackage{amsmath,amsfonts,bm}

\def\eqref#1{equation~\ref{#1}}
\def\1{\bm{1}}

\DeclareMathAlphabet{\mathsfit}{\encodingdefault}{\sfdefault}{m}{sl}
\SetMathAlphabet{\mathsfit}{bold}{\encodingdefault}{\sfdefault}{bx}{n}

\DeclareMathOperator{\Tr}{Tr}

\usepackage{xcolor}
\usepackage{hyperref}
\usepackage{url}
\usepackage{latexsym}
\usepackage[T1]{fontenc}
\usepackage[utf8]{inputenc}
\usepackage{microtype}

\usepackage{graphicx}

\let\origincludegraphics\includegraphics
\renewcommand{\includegraphics}[2][]{%
  \IfFileExists{#2}{\origincludegraphics[#1]{#2}}{%
    \IfFileExists{#2.png}{\origincludegraphics[#1]{#2.png}}{%
      \IfFileExists{#2.pdf}{\origincludegraphics[#1]{#2.pdf}}{%
        \IfFileExists{#2.jpg}{\origincludegraphics[#1]{#2.jpg}}{%
          \origincludegraphics[#1]{template.png}%
        }%
      }%
    }%
  }%
}

\usepackage{amsmath}
\usepackage{amssymb}
\usepackage{amsfonts}
\usepackage{mathtools}
\usepackage{empheq}     % box around multi-line numbered equations
\usepackage{amsthm}
\usepackage{booktabs}
\usepackage{makecell}

\usepackage[capitalize,noabbrev]{cleveref}

\theoremstyle{plain}
\newtheorem{theorem}{Theorem}[section]
\newtheorem{proposition}[theorem]{Proposition}
\newtheorem{lemma}[theorem]{Lemma}

\theoremstyle{definition}
\newtheorem{definition}[theorem]{Definition}

\theoremstyle{remark}

\usepackage{framed}
\usepackage{tcolorbox}
\tcbuselibrary{skins,breakable}
\colorlet{stmtfill}{blue!5}
\colorlet{proofbar}{gray!35}
\colorlet{stmtbar}{blue!55}
\colorlet{deffill}{orange!8}

\tcbset{
  stmtpanel/.style={
    breakable, enhanced jigsaw,
    colback=stmtfill, colframe=stmtfill,
    boxrule=0pt, arc=5pt, outer arc=5pt,
    boxsep=0pt, left=8pt, right=8pt, top=8pt, bottom=8pt,
  },
}

\newtcolorbox{stmtframe}{stmtpanel}

\newtcolorbox{defframe}{stmtpanel}

\newenvironment{proofframe}{%
  \MakeFramed{\advance\hsize-\width \FrameRestore}%
}{\endMakeFramed}

\newcommand{\wrapstmt}[1]{%
  \expandafter\let\csname orig#1\expandafter\endcsname\csname #1\endcsname
  \expandafter\let\csname endorig#1\expandafter\endcsname\csname end#1\endcsname
  \renewenvironment{#1}%
    {\stmtframe\csname orig#1\endcsname}%
    {\csname endorig#1\endcsname\endstmtframe}%
}
\wrapstmt{theorem}
\wrapstmt{proposition}
\wrapstmt{lemma}
\wrapstmt{corollary}

\usepackage{etoolbox}
\AtBeginEnvironment{proof}{\proofframe}
\AtEndEnvironment{proof}{\endproofframe}

\usepackage{subcaption} % For subfigures
\usepackage{wrapfig}    % For wrap-around figures
\usepackage{comment}
\usepackage{xurl}       % Breaks URLs at more locations
\usepackage{float}
\usepackage{caption}
\usepackage{adjustbox}  % for \adjustbox
\usepackage{multirow}
\usepackage{enumitem}
\usepackage{array}
\usepackage{tabularx}   % X columns in tables/oracle_settings.tex

\makeatletter
\def\oai@toc{toc}
\def\oai@sectionstr{section}
\def\oai@subsectionstr{subsection}
\long\def\oai@grabtype#1#2#3\oai@STOP{\def\oai@TY{#2}}
\newcommand{\StartAppendixTOCcapture}{%
  \let\oai@addtocontents\addtocontents
  \renewcommand{\addtocontents}[2]{%
    \oai@addtocontents{##1}{##2}%
    \def\oai@tmp{##1}%
    \ifx\oai@tmp\oai@toc
      \oai@grabtype##2\contentsline{}{}\oai@STOP
      \ifx\oai@TY\oai@sectionstr    \oai@addtocontents{apc}{##2}\fi
      \ifx\oai@TY\oai@subsectionstr \oai@addtocontents{apc}{##2}\fi
    \fi}%
}
\newcommand{\AppendixTOC}{%
  \begingroup
    \hypersetup{pdfborder={0 0 0}}%
    \color{blue}%
    \@starttoc{apc}%
  \endgroup}
\makeatother

\usepackage[textsize=tiny]{todonotes}

\providecommand{\norm}[1]{\lVert #1 \rVert}
\newcommand{\Iproxy}{\mathcal{I}_{\mathrm{proxy}}}
\DeclareMathOperator*{\argtopk}{arg\,top\,\mathit{k}}
\DeclareMathOperator{\GumbelTopk}{Gumbel\text{-}Top\text{-}\mathit{k}}

\title{Effective Synthetic Data Curation Requires Group-Level Signals}

\author{Cathy Jiao \quad Chenyan Xiong \\
Language Technologies Institute, Carnegie Mellon University \\
\texttt{\{cljiao,cx\}@cs.cmu.edu}}

\iclrfinalcopy % Uncomment for camera-ready version, but NOT for submission.

\begin{document}

\maketitle

\begin{abstract}
Synthetic data now is essential to LLM training, used to strengthen advanced capabilities such as autonomous and long-horizon task execution. Yet recent work shows that training on it at scale can degrade model generation, making it important to decide what synthetic data is worth training on. While current data curation practices do so with \emph{individual-level} signals (i.e., estimates of each data sample's training utility in isolation), across pre-training and post-training settings we show that this is insufficient for synthetic data, and that \emph{group-level} signals (i.e., estimates of utility that account for interactions among data samples) are necessary for effective data curation. First, we show that individual-level signals are blind to how samples jointly affect training: synthetic datasets with different compositions can be indistinguishable under individual-level influence yet differ sharply under group-level influence, and curating by the latter yields better downstream performance, particularly in generative capability. Second, we find that group-level signals matter more as training pipelines become increasingly synthetic: among widely used data curation methods, only those incorporating them improve over baseline, with gains increasing when weights capturing relations among samples are amplified. Finally, we translate these findings into practice -- for model developers under a compute budget, we offer a cheap diagnostic that prioritizes which groups of synthetic data most need group-level estimation, recovering much of the benefit of full group-level scoring at a fraction of the compute cost.
\end{abstract}
\section{Introduction}

Synthetic data is crucial to large language model (LLM) development, and is now heavily used across all stages of LLM training, from synthetic corpora in pre-training \citep{li2023textbooks, abdin2024phi4technicalreport} to inducing reasoning post-training \citep{nvidia2024nemotron4340btechnicalreport, guo2025deepseekr1}. It supplies targeted training signal for advancing LLM capabilities that are difficult to learn from organic web data alone, such as extended-horizon coding autonomy \citep{pmlr-v267-pan25g}, competition-level mathematics \citep{trinh2024alphageometry, ren2025deepseekproverv2advancingformalmathematical, hubert2026olympiad}, and reliable tool use \citep{qin2023toolllm, liu2024apigen}.

Although synthetic data is essential to continued progress in LLM capabilities, using it requires care. Recent work shows that training on it at scale can degrade LLM generative capabilities (a phenomenon known as ``\emph{model collapse}'') since synthetic data exhibits narrower coverage than organic web data and inherits generator biases \citep{dohmatob2024strongmodelcollapse, guo2024curious, shumailov2024ai, chen2024diversity}. Moreover, these consequences are costly for LLM developers, who waste substantial compute to recover from collapsed training runs \citep{chowdhery2023palm, wortsman2024small, grattafiori2024llama}. Therefore, using synthetic data effectively requires understanding how it contributes to (or harms) model training, so that it can be selected carefully.

While current data curation practices guide this selection with individual-level signals (i.e., estimates of each data sample's training utility in isolation), we show that these signals are insufficient, and that effective synthetic data curation instead requires group-level signals (i.e., estimates of training utility that account for interactions among data samples). First, individual-level signals are blind to these interactions: using data influence measures to quantify the utility of synthetic datasets, we find that datasets with different compositions (e.g., domains and diversity) can be indistinguishable under individual-level influence yet differ under group-level influence. Curating training data by group-level influence also yields better downstream performance, particularly in \emph{generative capability}. 

Next, we show that as LLM training pipelines grow increasingly reliant on synthetic data, incorporating group-level signals in standard data curation methods is necessary to enhance model performance. Analyzing several data curation methods used in practice \citep{bm25, xia2024lessselectinginfluentialdata, yu2025data, wang2024greats}, we find that, in this setting, methods which use group-level signals avoid model degradation and improve over baseline performance, whereas individual-level methods do neither. Gains increase further when the weights that capture relational information among data samples are amplified, suggesting that interactions within datasets should be modeled explicitly as pipelines rely more on synthetic data.

Finally, we build on our findings to address a practical concern: group-level methods can be computationally expensive, since capturing interactions among data samples often requires additional memory and compute, and practitioners with limited resources need to know where to allocate budget for them.  To guide this decision, we offer efficient diagnostics: by leveraging within-group gradient diversity, we identify groups where individual-level signals are likely to be misleading. Allocating budget to scoring these groups recovers much of the group-level signal at a fraction of the compute cost. To summarize, our contributions and findings are:

\begin{enumerate}[topsep=0pt, itemsep=3pt, parsep=0pt]
\item \textbf{Individual-level signals are insufficient for curating synthetic data. } Datasets that are indistinguishable under individual-level influence differ greatly in group-level influence, and curating by group-level influence yields better downstream performance.
\item \textbf{Group-level signals matter more as data pipelines grow increasingly synthetic.} Among widely used data curation methods, we find only methods which incorporate group-level signals improve over baseline, with gains in generative capability.
\item \textbf{Leveraging group-level signals under a compute budget.} We offer practitioners cheap diagnostics, leveraging within-group gradient diversity, to decide where group-level signals are worth computing under a tight budget.
\end{enumerate}
\section{Preliminaries}\label{sec:background}

We begin with background on data influence, specifically oracle scoring, the instrument we use in our later analyses for synthetic data curation.

\paragraph{Oracle scoring.} Data influence measures the effect of a data sample on model behavior, and one way to quantify it is ``oracle scoring'', which compares a given model trained with and without the sample \citep{pmlr-v70-koh17a, pruthi2020estimating, basu2020second}. Formally, given a training dataset $D_{\mathrm{train}} = \{z_i\}_{i=1}^{N}$, a reference set $D_{\mathrm{ref}} = \{z'_i\}_{i=1}^{M}$ with $M \ll N$, and a model with parameters $\theta_{D_{\mathrm{train}}}$ obtained by training on $D_{\mathrm{train}}$, we define the oracle score as:

\begin{defframe}
\begin{definition}[Oracle Score: Group and Individual]
\label{def:oracle_influence}
%Let $\theta_D$ denote the parameters obtained by training on dataset $D$.
The \emph{oracle score} of a set of training examples
$G = \{z_j\}_{j=1}^{|G|}$ measures the change in model performance when $G$
is included versus excluded from the training data, evaluated with respect
to the model loss on a reference set $D_{\mathrm{ref}}$:
\begin{equation}
\mathcal{I}_{\mathrm{orcl}}(G, \theta)
= L\!\left(D_{\mathrm{ref}} \mid \theta_{D_{\mathrm{train}} \cup G}\right)
- L\!\left(D_{\mathrm{ref}} \mid \theta_{D_{\mathrm{train}}}\right).
\label{eqn:group_if}
\end{equation}
We refer to $\mathcal{I}_{\mathrm{orcl}}(G, \theta)$ as the \emph{group oracle
score}. For a single example $z$ ($|G| = 1$) we write
$\mathcal{I}_{\mathrm{orcl}}(z, \theta) \coloneqq \mathcal{I}_{\mathrm{orcl}}(\{z\}, \theta)$
and call it the \emph{individual oracle score}.
\end{definition}
\end{defframe}

In standard data influence settings, the reference set is chosen to \emph{reflect} a target capability (e.g., math or code text), so that loss on $D_{\mathrm{ref}}$ is a measurable proxy for that capability. Prior work has shown that selecting training data to reduce the oracle reference loss yields gains on the corresponding downstream tasks in LLM pre-training and post-training \citep{xia2024lessselectinginfluentialdata, yu2024mates}.

\paragraph{Oracle scoring for datasets.} Existing data curation pipelines typically rely on scorers that operate at the \emph{individual level} -- such as heuristic deduplication filters and text quality scorers -- which score each sample in isolation \citep{li2024datacomplm, penedo2024fineweb}. Conceptually, these methods approximate the group oracle score by aggregating individual oracle scores, which we formalize as the \emph{individual proxy oracle} below.

\begin{defframe}
\begin{definition}[Individual Proxy Oracle]
\label{asm:additive}
The group oracle is approximated by the mean of individual oracle scores,
\begin{equation}
\mathcal{I}_{\mathrm{orcl}}(G, \theta)
\;\approx\;
\Iproxy(G, \theta) = \frac{1}{|G|}\sum_{z \in G} \mathcal{I}_{\mathrm{orcl}}(z, \theta).
\label{eqn:additive_assumption}
\end{equation}
\end{definition}
\end{defframe}

In practice, oracle scores are expensive since they require model training. However, for the individual proxy oracle, cheaper approximation methods exist for LLM data curation, such as LESS and LoGRA \citep{choe2024logra, xia2024lessselectinginfluentialdata}, making the individual proxy oracle in \cref{eqn:additive_assumption} tractable for group-level influence estimation compared to the true group-level oracle. 

\begin{figure*}[t]
\centering
% Print subcaption labels as (1a)/(1b) rather than (a)/(b). Clearing
% \p@subfigure (normally \thefigure) keeps \ref at "1a" instead of "11a".
\renewcommand{\thesubfigure}{\thefigure\alph{subfigure}}
\makeatletter\renewcommand{\p@subfigure}{}\makeatother
\begin{subfigure}[t]{0.68\textwidth}
    \centering
    \includegraphics[width=\linewidth]{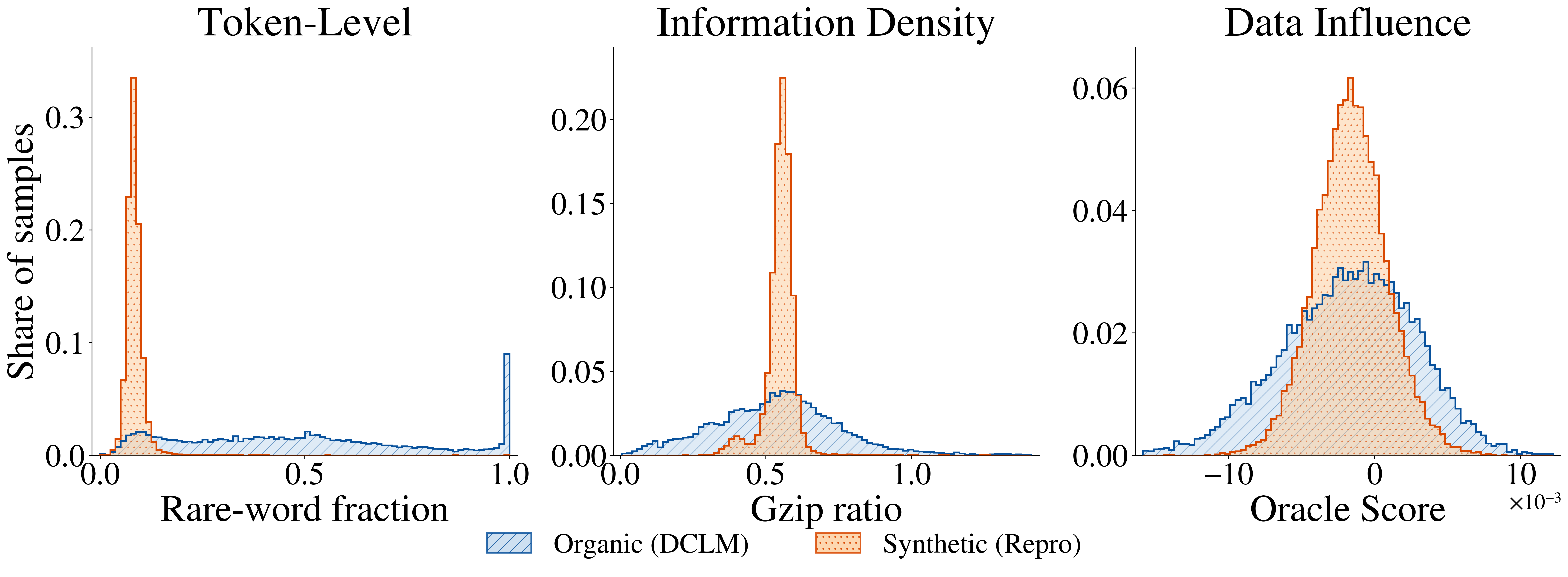}
%    \caption{Token level, information density, and data influence analysis on organic and synthetic data sequences (n=1000), from Repro~\citep{yu2025reprotraininglanguagemodels} and DCLM\citep{li2024datacomplm}, respectively.}
    \caption{\textbf{Synthetic data is more homogeneous.} Spread of token-level, information density, and data influence measures over $n{=}1000$ samples each from organic (DCLM) and synthetic (Repro) data.}
    \label{fig:bank_spread_boxes}
\end{subfigure}
\hfill
\begin{subfigure}[t]{0.29\textwidth}
    \centering
    \includegraphics[width=\linewidth]{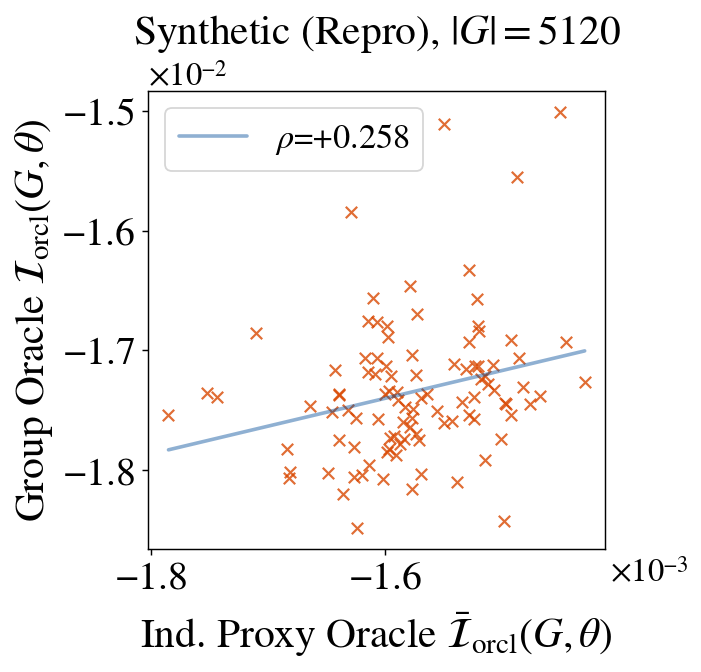}
%    \caption{Ind. proxy and group oracle correlation.}
    \caption{\textbf{Group value is not additive:} the ind. proxy oracle tracks the group oracle only weakly.}
    \label{fig:uniform_group_oracle_vs_mean_indiv_paper}
\end{subfigure}
\label{fig:background}
\vspace{-0.5cm}
\end{figure*}

With these measures in place, a natural question is why synthetic data calls for particular attention. Below, we discuss two key challenges that arise in curating synthetic datasets when we solely rely on individual-level scores.

\paragraph{Challenge 1: Synthetic data is homogeneous.} A large synthetic corpus is often generated from a few seed prompts \citep{li2023textbooks, maini2024rephrasing} (e.g., instruction-following data bootstrapped from a small pool of human-written tasks \citep{wang2023self, taori2023stanford}), making it more homogeneous than an organic corpus collected from the web. This homogeneity makes it harder to distinguish the synthetic samples that most improve model performance. As shown in Figure~\ref{fig:bank_spread_boxes}, synthetic data samples have lower variance than organic ones not only at the token-level and in information density, \emph{but also} in individual-level data influence.

\paragraph{Challenge 2: A dataset's value is non-additive.} Moving beyond individual data points, a dataset's value is shaped by the combined contribution of its members: interactions between samples can induce cancellation or amplification effects \citep{koh2019accuracy, basu2020second}, making a dataset's value non-additive in its samples. These properties are invisible to an individual scorer. Figure~\ref{fig:uniform_group_oracle_vs_mean_indiv_paper} shows an example, where the individual proxy oracle does not track the group oracle, causing a mismatch between the datasets each ranks highly for model training, which we explore in the following sections.
\section{Group-level Signals for Synthetic Data Influence}\label{sec:single_group}

The previous section showed that the individual proxy oracle and the group oracle can disagree. Here, we show through controlled experiments that this disagreement matters in practice: synthetic data groups that look equivalent under the individual proxy can differ significantly in group-level influence and downstream value. The remainder of the section describes the task settings in Section~\ref{sec:single_group_tasks}, the experimental procedure in Section~\ref{sec:single_group_setup}, and the results in Section~\ref{sec:single_group_results}.

\subsection{Task Settings}\label{sec:single_group_tasks}

We study the gap between individual and group-level influence in pre-training and post-training settings. The datasets and models used in both settings are described below.

\paragraph{Pre-training.} The first setting is synthetic data pre-training, which is increasingly being incorporated into frontier LLM training pipelines \citep{liu2024deepseek, grattafiori2024llama, yang2025qwen3}. In this setting, we train on data from the Repro-Rephrased-72B \citep{yu2025reprotraininglanguagemodels} dataset, which has been refined from low-quality web scrapes, and oracle scoring uses Repro-400M (trained on 6.59B Repro tokens) \citep{yu2025reprotraininglanguagemodels} as the base model against a 128-sample FLAN \citep{chung2024scaling} reference set, following prior oracle-scoring work \citep{yu2024mates, yu2025data, datelm}. We evaluate on eight standard pre-training benchmarks (ARC-Challenge, CommonsenseQA, CoQA, HellaSwag, LAMBADA, MMLU, PIQA, SQuAD).

\paragraph{Post-training.} The second setting is post-training using reinforcement learning with verifiable rewards (RLVR), a key factor in inducing reasoning in LLMs for coding, math, and agentic tasks. Here, we train on data from the GooseReason \citep{lu2026goldengoosesimpletrick} dataset, and oracle scoring uses Qwen2.5-1.5B-Instruct \citep{qwen2025qwen25technicalreport} as the base model against a 128-sample GPQA-extend \citep{rein2024gpqa} reference set, chosen for its graduate-level, multi-domain reasoning coverage. We evaluate on reasoning/math/coding benchmarks (MBPP, MBPP+, GPQA-D, BBH, IFEval, AIME24).

Oracle scoring parameters for both settings can be found in Appendix~\ref{appendix:oracle_scoring_parameters}

\subsection{Experiment Setup}\label{sec:single_group_setup}
Below, we discuss our general experiment setup, which we instantiate identically for both settings above and which consists of three phases: candidate pool construction, data selection, and model training. We describe the overall procedure here and \emph{defer the specific details to Appendix~\ref{appendix:tables_figures_task_settings}}. 

\begin{wrapfigure}[14]{r}{0.5\textwidth}
\centering
\vspace{-\baselineskip}
\begin{subfigure}[t]{0.45\linewidth}
    \centering
    \includegraphics[width=\linewidth]{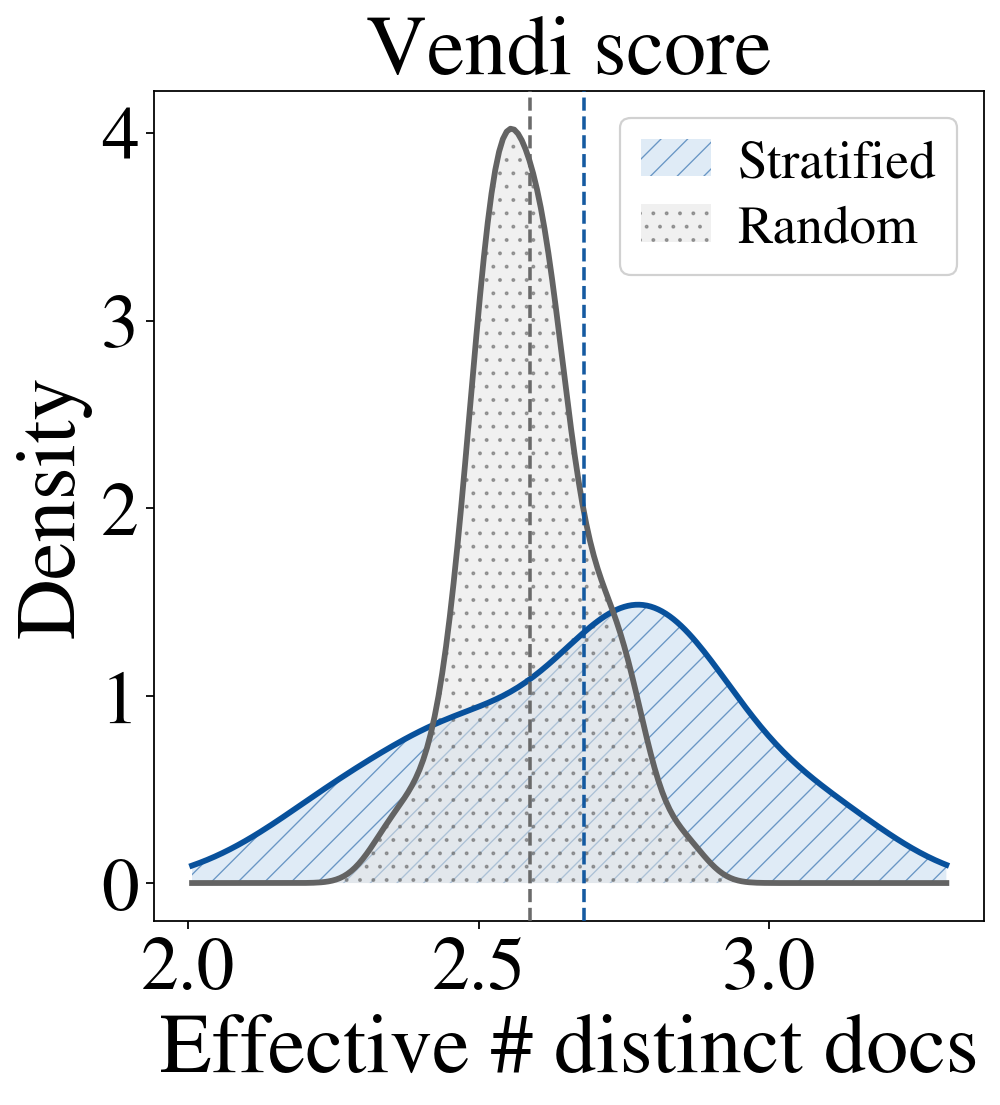}
    \caption{Semantic diversity.}
    \label{fig:divmetric_combined_vendi}
\end{subfigure}
\hfill
\begin{subfigure}[t]{0.53\linewidth}
    \centering
    \includegraphics[width=\linewidth]{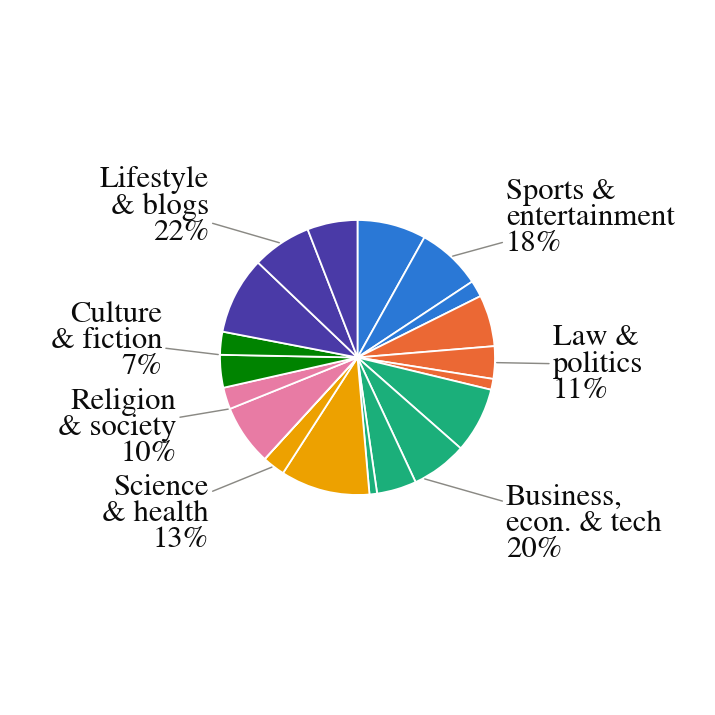}
    \caption{Topic diversity.}
    \label{fig:divmetric_combined_topic}
\end{subfigure}
\caption{Semantic/topic diversity of our candidate pool on the pre-training corpus.}
\label{fig:diversity_metrics_combined}
\end{wrapfigure}

\noindent\textbf{Phase 1: Candidate Pool Construction.} To isolate the interactions between data samples that the individual proxy fails to capture, we form groups with (i) similar individual proxy scores and (ii) different compositions. The former holds the individual proxy fixed as a control, while the latter varies how samples within a group interact and ensures the candidate pool spans a broad range of data, covering varied levels of diversity and topics.

From each corpus (Repro and GooseReason), we build a pool of $H$ equal-sized groups $\mathcal{G}=\{G_1,\dots,G_H\}$. To control individual proxy scores, we draw each group's members from above and below a target score in a narrow band such that the group's mean matches the target. To vary composition, we partition the corpus into $20$ semantic clusters (using BGE embeddings with $k$-means) and draw each group from between $1$ (homogeneous) and $20$ (diverse) clusters. Additional details on candidate pool construction are discussed in Appendix~\ref{appendix:cluster_sampling}.
 
Figure~\ref{fig:divmetric_combined_vendi} shows that our groups span a wider semantic diversity range than uniformly random groups, and Figure~\ref{fig:divmetric_combined_topic} shows the breakdown of topics across the $20$ clusters (a full breakdown of per-corpus cluster topics are in Appendix Figure~\ref{fig:cluster_topics_pie}).

\noindent\textbf{Phase 2: Data selection.} From the candidate pool $\mathcal{G}$, each scoring method selects a set of $k$ groups to form the training sets $\mathcal{D}_{\mathrm{proxy}}$ and $\mathcal{D}_{\mathrm{orcl}}$. Following prior work that uses data influence for data selection \citep{yu2024mates, datelm}, we draw these groups via \emph{Gumbel-Top-$k$} sampling \citep{kool2019stochastic} so as to better reflect the distribution of oracle scores.
\begin{align*}
  \mathcal{D}_{\mathrm{proxy}} &\leftarrow \GumbelTopk\big(\{\Iproxy(G, \theta)/\lambda \mid G \in \mathcal{G}\}\big) \\
  \mathcal{D}_{\mathrm{orcl}} &\leftarrow \GumbelTopk\big(\{\mathcal{I}_{\mathrm{orcl}}(G, \theta)/\lambda \mid G \in \mathcal{G}\}\big),
\end{align*}
where  $\lambda \in (0, \infty)$ is the temperature.

\paragraph{Phase 3: Training and evaluation.} We compare three selection methods under a matched budget: the group oracle ($\mathcal{D}_{\mathrm{orcl}}$), the individual proxy oracle ($\mathcal{D}_{\mathrm{proxy}}$), and a random baseline that draws $k$ groups uniformly from the same pool $\mathcal{G}$. Starting from the same base model (Repro-400M for pre-training and Qwen2.5-1.5B-Instruct for post-training), we train one model per set with identical hyperparameters and step counts, so differences in downstream performance come from selection alone. We report accuracy on the benchmarks in Section~\ref{sec:single_group_tasks}.

\begin{table}[t]
\centering
\caption{Results for pre- and post-training data selection. Per-task accuracy (\%) on Repro-400M (400 steps) and Qwen2.5-1.5B-Instruct (100 steps), for groups selected by random sampling, the individual proxy oracle, and the group oracle. Best average within each group size is in \textbf{bold}.}
\providecolor{pretraincolor}{RGB}{220,232,247}
\providecolor{posttraincolor}{RGB}{250,228,214}
\setlength{\tabcolsep}{1.5pt}
\footnotesize
\begin{adjustbox}{max width=\textwidth}
\begin{tabular}{@{}llccccccccc@{}}
\toprule
\rowcolor{pretraincolor}
\multicolumn{11}{@{}l}{\textit{(a) Pre-training: Repro-400M}} \\
\midrule
\rowcolor{pretraincolor!45}
\textbf{Group Size} & \textbf{Method} & \textbf{HellaSwag} & \textbf{ARC-C} & \textbf{PIQA} & \textbf{MMLU} & \textbf{CSQA} & \textbf{LAMBADA} & \textbf{SQuAD} & \textbf{CoQA} & \textbf{Avg.} \\
\midrule
%\multicolumn{2}{@{}l}{Baseline} & 30.82 & 25.85 & 64.36 & 23.94 & 19.82 & 37.30 & 7.47 & 27.61 & 29.65 \\
& Baseline & 30.82 & 25.85 & 64.36 & 23.94 & 19.82 & 37.30 & 7.47 & 27.61 & 29.65 \\
\midrule
$5120$ & Random & 31.84 & 27.43 & 66.21 & 25.53 & 20.07  & 37.29 & 9.69 & 28.83 & 30.86 \\
 & Ind. Proxy Orcl. & 32.01 & 26.96 & 66.59 & 26.01 & 21.05 & 38.09 & 10.25 & 28.62 & 31.20 \\
 & Group Orcl. & 31.78 & 26.54 & 67.30 & 25.10 & 20.56 & 38.44 & 16.19 & 29.97 & \textbf{31.98} \\
\midrule
$10240$ & Random & 31.58 & 25.94 & 66.35 & 24.74 & 20.31 & 38.66 & 10.47 & 28.49 & 30.82 \\
 & Ind. Proxy Orcl. & 31.53 & 26.54 & 66.05 & 24.30 & 20.15 & 37.63 & 9.65 & 28.49 & 30.54 \\
 & Group Orcl. & 31.40 & 26.45 & 66.38 & 24.74 & 20.07 & 38.87 & 9.65 & 30.42 & \textbf{31.00} \\
\midrule
\rowcolor{posttraincolor}
\multicolumn{11}{@{}l}{\textit{(b) Post-training: Qwen2.5-1.5B-Instruct}} \\
\midrule
\rowcolor{posttraincolor!45}
\textbf{Group Size} & \textbf{Method} & \textbf{MBPP} & \textbf{MBPP+} & \textbf{GPQA-D} & \textbf{BBH} & \textbf{IFEval} & \textbf{AIME24} & & & \textbf{Avg.} \\
\midrule
& Baseline & 41.20 & 45.50 & 25.93 & 37.46 & 44.20 & 3.00 & & & 32.88 \\
\midrule
$128$ & Random & 44.00 & 48.59 & 26.54 & 40.42 & 44.60 & 1.11 & & & 34.21 \\
 & Ind. Proxy Orcl. & 42.60 & 46.03 & 26.26 & 40.87 & 44.40 & 1.00 & & & 33.53 \\
 & Group Orcl. & 45.40 & 52.12 & 26.43 & 40.56 & 45.80 & 1.00 & & & \textbf{35.22} \\
\midrule
$512$ & Random & 43.40 & 48.41 & 25.00 & 40.39 & 45.20 & 1.33 & & & 33.96 \\
 & Ind. Proxy Orcl. & 43.80 & 49.21 & 26.43 & 37.64 & 45.20 & 2.67 & & & 34.16 \\
 & Group Orcl. & 44.20 & 50.53 & 26.94 & 40.92 & 46.00 & 2.33 & & & \textbf{35.15} \\
\bottomrule
\end{tabular}
\end{adjustbox}
\label{tab:proxy_group}
\end{table}

\subsection{Results}\label{sec:single_group_results}

\begin{wrapfigure}[13]{r}{0.5\textwidth}
\centering
\vspace{-\baselineskip}
\begin{subfigure}[t]{0.48\linewidth}
    \centering
    \includegraphics[width=\linewidth]{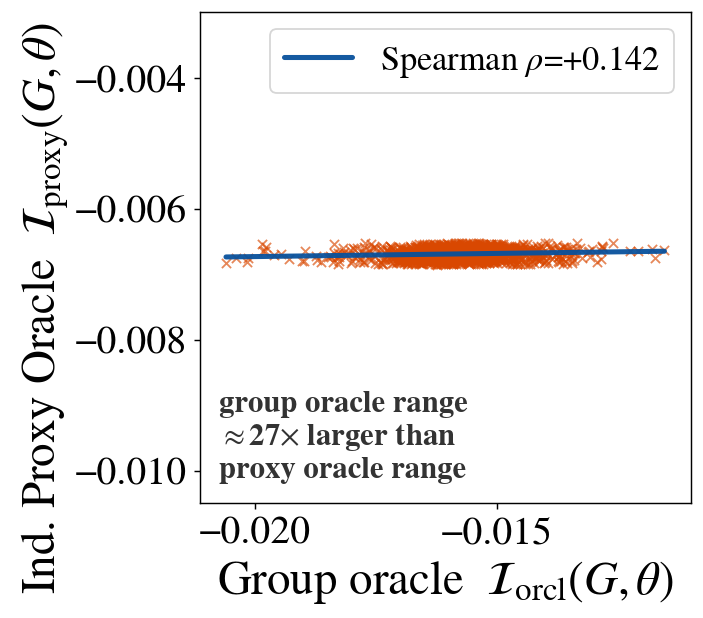}
    \caption{Pre-train.}
    \label{fig:single_vs_group_wrap_repro}
\end{subfigure}
\hfill
\begin{subfigure}[t]{0.49\linewidth}
    \centering
    \includegraphics[width=\linewidth]{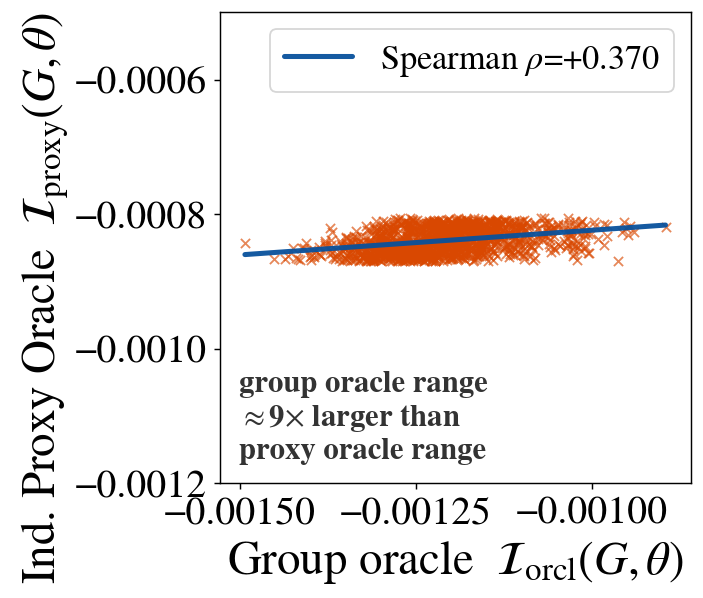}
    \caption{Post-train.}
    \label{fig:single_vs_group_wrap_goose}
\end{subfigure}
\caption{Range of group oracle and individual proxy oracle scores across our constructed groups (on groups of $5120$ and $128$ samples). See Appendix~\ref{appendix:tables_figures_diversity} for full sweep over group sizes.}
\label{fig:single_vs_group_wrap}
\end{wrapfigure}

We examine how group oracle scores vary across the candidate pool, then compare group oracle, individual proxy oracle, and random data selection on downstream model performance.

\paragraph{Finding 1: Group oracle recovers signal the individual proxy misses.} Figure~\ref{fig:single_vs_group_wrap} shows that although all groups in the candidate pool have near-identical individual proxy scores, their group oracle scores are mismatched and span a much wider range. This mismatch matters in practice: Table~\ref{tab:proxy_group} shows that, across all task settings and group sizes, training on data selected by the group oracle outperforms both random selection and individual proxy oracle selection.

\paragraph{Finding 2: Gains concentrate on generative capability.} Next, to see where the group oracle's advantage over the individual proxy oracle comes from, we examine the task types. Figure~\ref{fig:downstream_breakdown_repro} breaks these gains down by task type. On free-form generation (LAMBADA, SQuAD, CoQA), group oracle selection achieves a $+16.9\%$ relative gain over the base model, far above individual proxy oracle selection ($+6.3\%$) and random selection ($+4.7\%$). On multiple-choice tasks (HellaSwag, ARC-C, PIQA, MMLU, CSQA), all three are comparable ($+3.8\%$ to $+4.8\%$). The RLVR setting shows the same pattern (Figure~\ref{fig:downstream_breakdown_goose}): over the course of GRPO training, pass@8 increases under group oracle selection, while under individual proxy oracle selection it ends below its starting value. 

\begin{wrapfigure}[14]{r}{0.5\textwidth}
\centering
\vspace{-\baselineskip}
\begin{subfigure}[t]{0.48\linewidth}
    \centering
    \includegraphics[width=\linewidth]{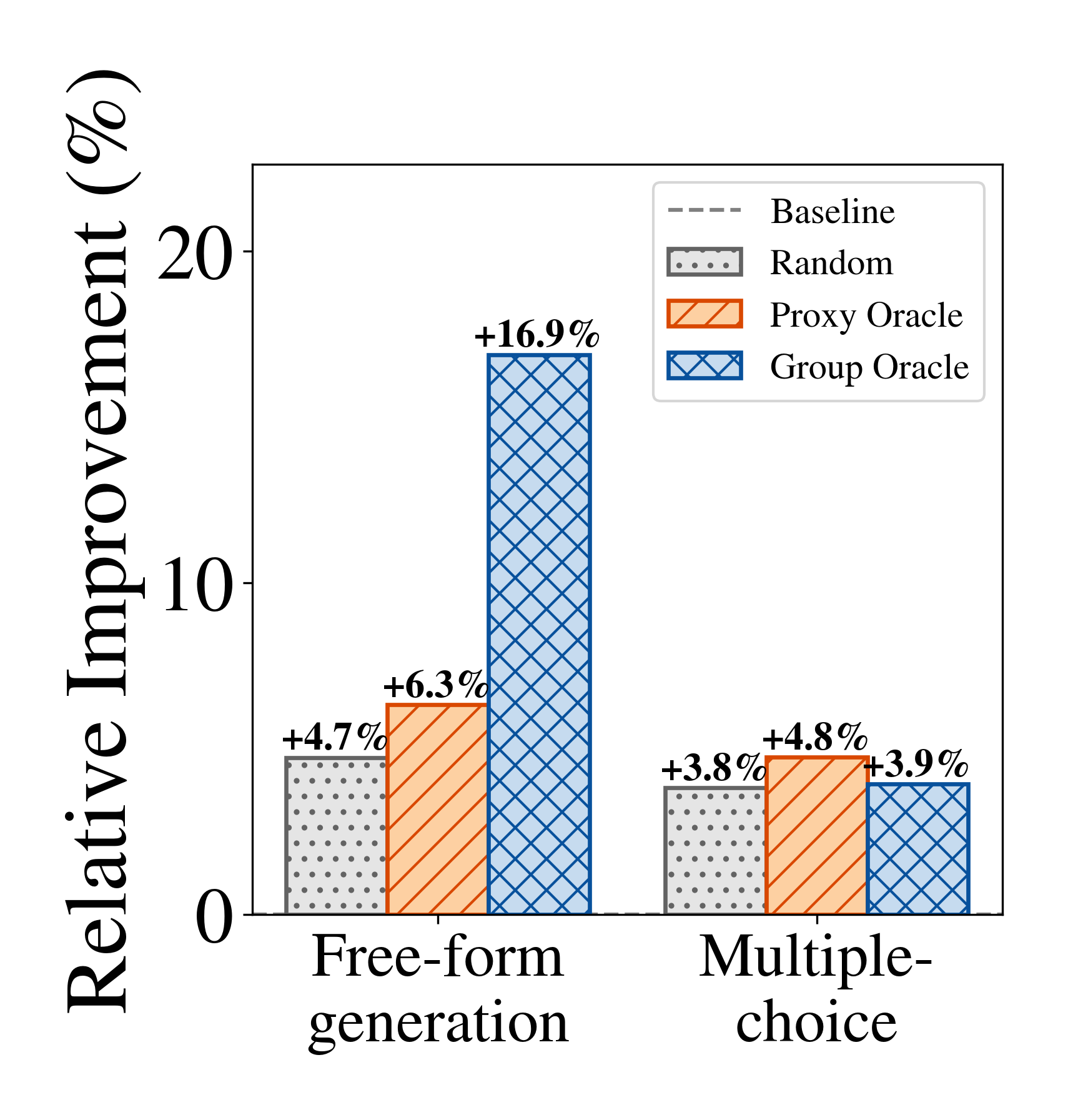}
    \caption{Pre-train (size 5120).}
    \label{fig:downstream_breakdown_repro}
\end{subfigure}
\hfill
\begin{subfigure}[t]{0.48\linewidth}
    \centering
    \includegraphics[width=\linewidth]{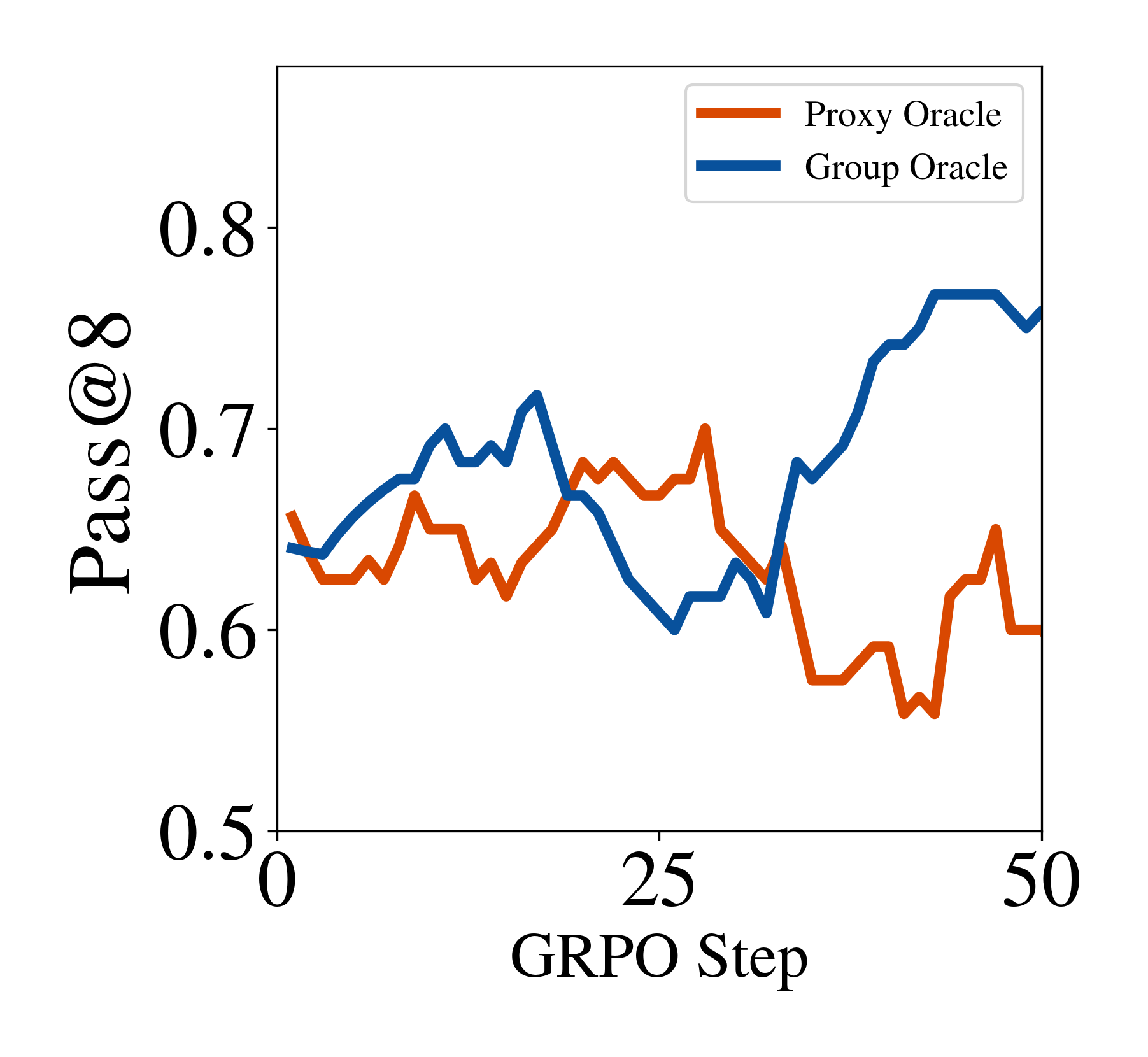}
    \caption{Post-train (size 512).}
    \label{fig:downstream_breakdown_goose}
\end{subfigure}
\caption{\emph{(a)} Relative improvement over the base model by task type in pre-training. \emph{(b)} Pass@8 over GRPO training steps in post-training. }
\label{fig:downstream_breakdown}
\end{wrapfigure}

Together, our results point to two blind spots in synthetic data curation. First, practitioners should not rely solely on individual-level scoring, as interactions between samples influence model training in ways that individual-level scoring does not capture. Instead, practitioners should leverage group-level scoring, which captures the full value of a dataset; we discuss this further in Section~\ref{sec:diagnostics}. 

Second, multiple-choice benchmarks can hide this gap, since all selection methods perform similarly on them. It is therefore important to also evaluate on generative tasks, where the difference between individual- and group-level scoring becomes visible.
\section{Group-level Signals for Synthetic Data Curation}
\label{sec:data_influence_methods}
Next, we turn to the data selection methods that model developers use in practice, and show how group-level signals become more necessary as training pipelines grow increasingly synthetic. Section~\ref{sec:collapse_setup} describes the experimental setup and Section~\ref{sec:collapse_results} the results.

\subsection{Experiment Setup}\label{sec:collapse_setup}

We study synthetic data recycling, where noisy web text is rephrased into clean training data. Recycling eases data scarcity by generating multiple samples per seed, and is used to pre-train recent LLMs such as Kimi K2 and Qwen3 \citep{team2025kimi, yang2025qwen3}. Below, we discuss our general experiment setup, consisting of three phases: synthetic pool construction, data selection, and training/evaluation. We describe the overall procedure here and \emph{defer specific details to Appendix~\ref{appendix:supp_data_curation_exp}}. 

\paragraph{Phase 1: Synthetic Pool Construction.}
We begin with a large corpus of noisy organic documents and use a rephraser model to paraphrase them, producing clean synthetic pools. Specifically, we take the organic corpus to be DCLM \citep{li2024datacomplm} and the rephraser to be Repro \citep{yu2025reprotraininglanguagemodels}, a 4B model fine-tuned to paraphrase web text. 

Each synthetic pool has the same fixed size $N = 716{,}800$ sequences but a different \emph{rephrase factor} $r \in \{1, 10, 50\}$, which controls the number of rephrases per organic data sample: for a given $r$, we sample $N/r$ source documents from the organic corpus and draw $r$ paraphrases of each. Here, $r=1$ and $r=10$ span the range used in practice (e.g., Kimi K2 rephrases its knowledge corpus up to $10$ times \citep{team2025kimi}), while $r=50$ serves as a stress test of highly synthetic training data.

\begin{wraptable}[13]{r}{0.4\textwidth}
\centering
\scriptsize
\setlength{\tabcolsep}{4pt}
\renewcommand{\arraystretch}{1.6}
\begin{tabular}{@{}c@{\hspace{4pt}}l >{\raggedright\arraybackslash}p{0.34\linewidth}@{}}
\toprule
& \textbf{Method} & \textbf{Representation} \\
\midrule
\multirow{3}{*}{\rotatebox[origin=c]{90}{\scriptsize Individual}}
& GradSim \tiny\citep{pruthi2020estimating} & Gradient \\
& FineWeb-Edu \tiny\citep{penedo2024fineweb} & Classifier score \\
& BM25 \tiny\citep{bm25} & Lexical (tokens) \\
\cmidrule(l){2-3}
\multirow{2}{*}{\rotatebox[origin=c]{90}{\tiny Group}}
& GREATS \tiny\citep{wang2024greats} & Gradient \\
& GMRel \tiny\citep{yu2025data} & Learned influence \\
\bottomrule
\end{tabular}
\caption{Taxonomy of individual-and group-level data selection methods. Details in Appendix~\ref{appendix:group_influence_methods}}
\label{fig:method_taxonomy}
\end{wraptable}

\paragraph{Phase 2: Data selection.} We compare multiple data selection methods (see Table~\ref{fig:method_taxonomy}) from two families: individual-level methods, which score each sample in isolation, and group-level methods, which also model relations among samples. Each method selects data from the Phase 1 pools to construct training sets. For efficiency at this scale, we adopt the cluster-based inference of \cite{yu2025data}, with which each method forms $20$ groups of size $5120$. See Appendix~\ref{appendix:cluster_inference} for full details.

\paragraph{Phase 3: Training and evaluation.} For each method we train the same Repro-400M base model as in Section~\ref{sec:single_group} \citep{yu2025reprotraininglanguagemodels}, on its formed groups under identical optimization: batch size $512$ for $200$ steps, peak learning rate $3\times10^{-3}$ under a cosine schedule (no warmup, weight decay $0.033$). We evaluate the trained models on the same pre-training suite as in Section~\ref{sec:single_group}.

\subsection{Results}\label{sec:collapse_results}

\begin{figure*}[t]
\centering
\begin{subfigure}[t]{0.22\textwidth}
    \centering
    \includegraphics[width=\linewidth]{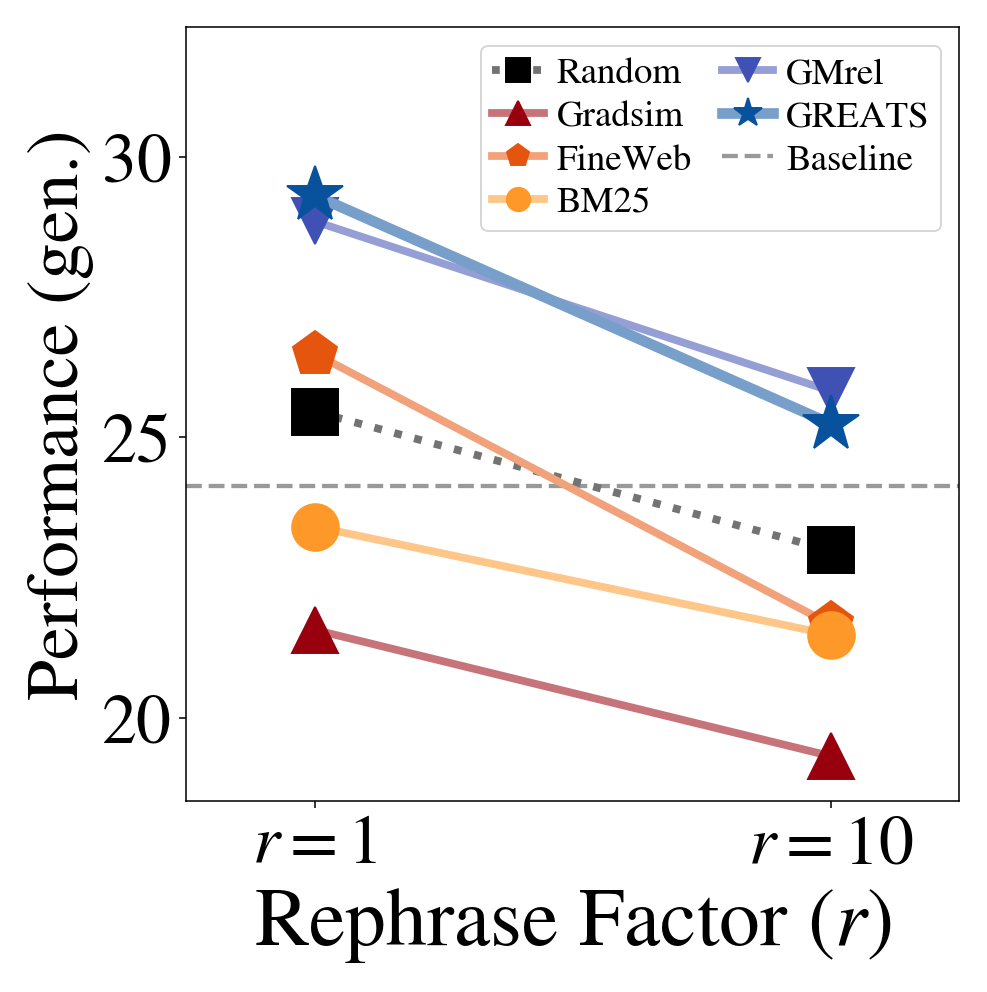}
    \caption{Generative tasks.}
    \label{fig:part2_gen_r1_r10}
\end{subfigure}
\hfill
\begin{subfigure}[t]{0.22\textwidth}
    \centering
    \includegraphics[width=\linewidth]{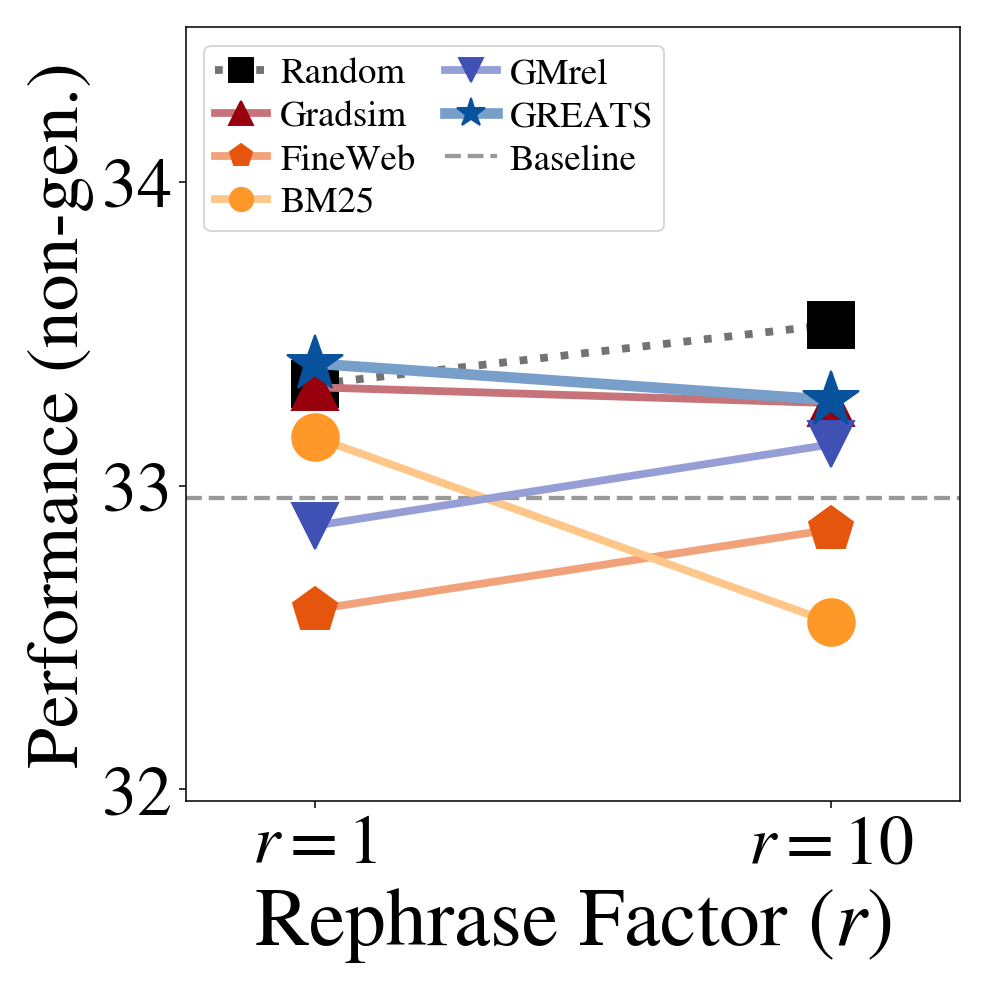}
    \caption{Non-Gen. Tasks.}
    \label{fig:part2_nongen_r1_r10}
\end{subfigure}
\hfill
\begin{subfigure}[t]{0.22\textwidth}
    \centering
    \includegraphics[width=\linewidth]{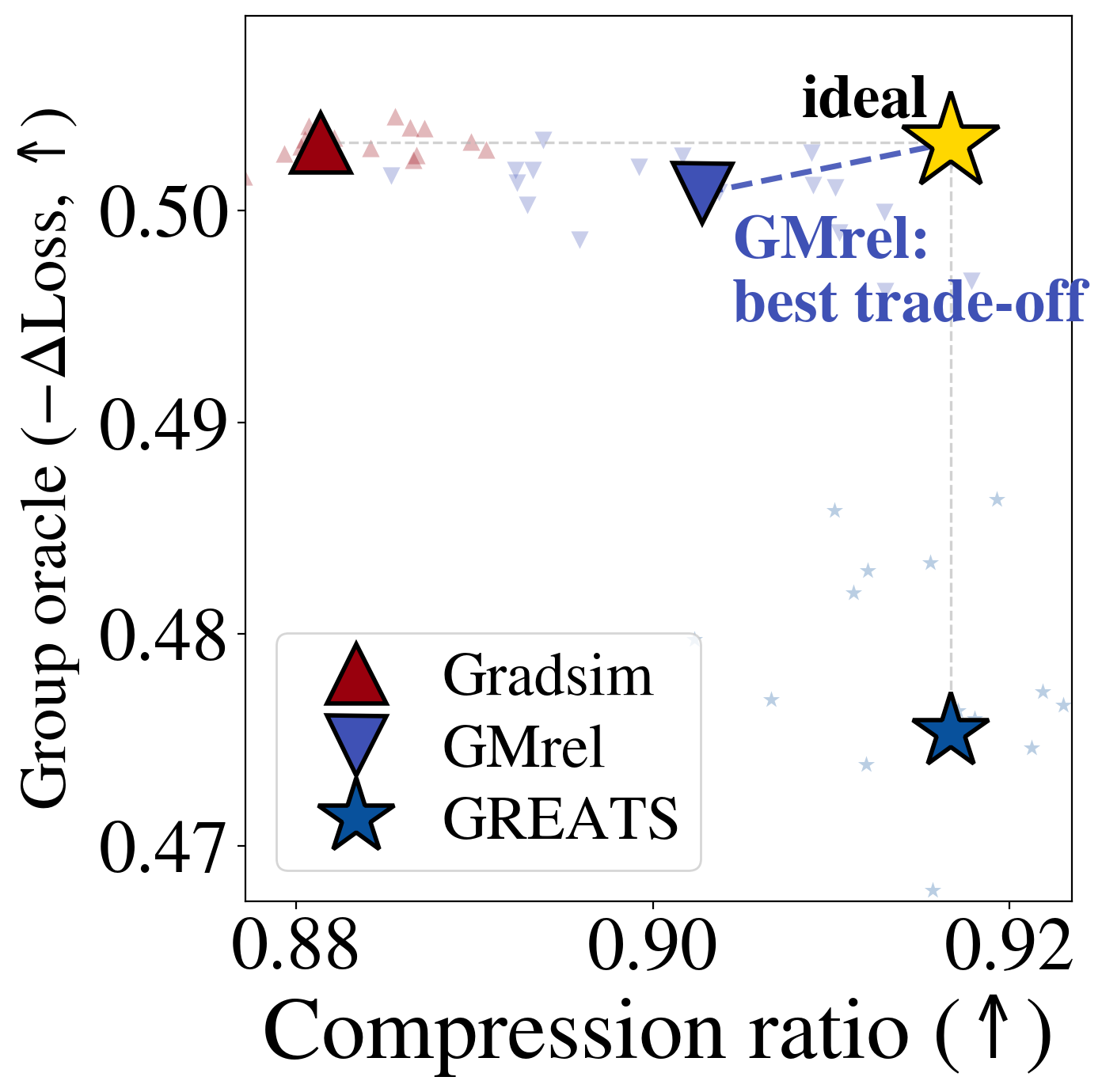}
    \caption{Value-div. tradeoff.}
    \label{fig:oracle_tradeoff_comp_ratio_k1}
\end{subfigure}
\hfill
\begin{subfigure}[t]{0.24\textwidth}
    \centering
    \includegraphics[width=\linewidth]{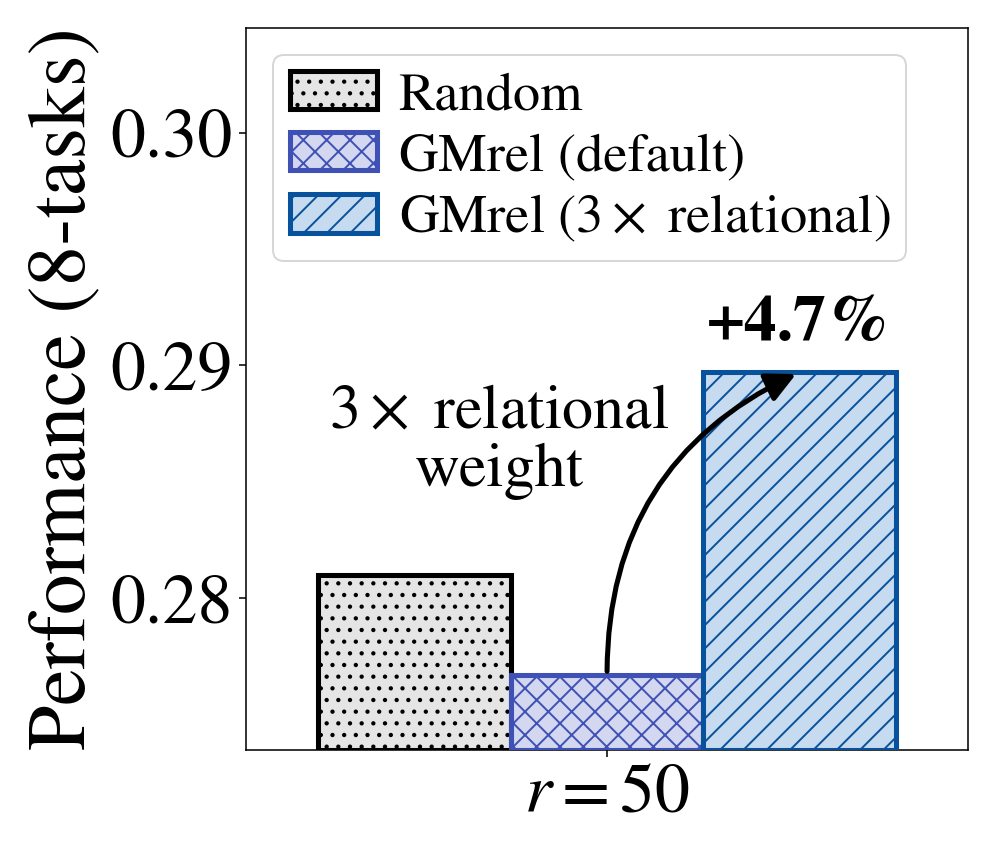}
    \caption{Relational factor.}
    \label{fig:gmrel_taus_r50}
\end{subfigure}
\caption{\emph{(a, b)} Performance of individual-/group- level selection methods as rephrase factor increases. \emph{(c)} Group-oracle value and diversity trade-off. \emph{(d)} Increasing GMRel's relational weight at $r=50$.}
\label{fig:collapse_results}
\vspace{-0.5cm}
\end{figure*}

We evaluate each selection method as the source pool becomes increasingly synthetic. We focus on the rephrase factors $r=1$ and $r=10$, which span the range used in frontier LLM pre-training \citep{team2025kimi}, and examine $r=50$ separately below, deferring the full breakdown for all $r$ to Appendix~\ref{appendix:tables_figures_collapse}.

\paragraph{Finding 1: Individual-level selection methods collapse on generative tasks.} Figures~\ref{fig:part2_nongen_r1_r10} and~\ref{fig:part2_gen_r1_r10} show downstream performance on non-generative and generative tasks, respectively, when individual-level selection methods (GradSim, FineWeb-Edu, and BM25) and group-level selection methods (GMRel and GREATS) are used to curate training data. While Figure~\ref{fig:part2_nongen_r1_r10} shows that all methods plateau on multiple-choice tasks, the difference between the two families is visible in Figure~\ref{fig:part2_gen_r1_r10}, where the group-level selection methods outperform the individual-level selection methods, similar to the observations in Section~\ref{sec:single_group_results}.

\paragraph{Finding 2: Group-level selection methods balance value and diversity.} Next, we examine the characteristics of the groups formed by individual-level and group-level selection methods. Figure~\ref{fig:oracle_tradeoff_comp_ratio_k1} plots group oracle value against diversity for the methods that approximate oracle influence (GradSim, GMRel, and GREATS) at $r=1$. GradSim greedily selects the highest-scoring samples, which gives it a high group oracle value but the least diverse groups. This lack of diversity may explain why it, like the other individual-level methods, underperforms on generative tasks. GMRel and GREATS, in contrast, form more diverse groups at a slight cost to group oracle value. This suggests that balancing value and diversity is important for effective synthetic data selection.

\paragraph{Finding 3: Stronger relational weighting helps with highly synthetic training data.} Next, we analyze the extent to which group-level selection methods are effective in highly synthetic data pipelines, using the $r=50$ stress test (see Table~\ref{tab:collapse_results} in Appendix~\ref{appendix:tables_figures_collapse}). Here, Figure~\ref{fig:gmrel_taus_r50} shows that tripling GMRel's relational weight improves its relative downstream performance by $4.7\%$, surpassing random selection and baseline. This indicates that as the training data pipelines become more synthetic, group-level selection methods should weigh inter-sample relationships more heavily.

As in Section~\ref{sec:single_group}, our findings suggest that forming groups of data to train on requires careful consideration. As data pipelines come to depend on synthetic and self-generated data, group effects in data selection methods should be modeled directly to prevent suboptimal selection.
\section{Diagnostics for Group-Level Data Selection}\label{sec:diagnostics}

In this section, we turn to the practical cost of group-level scoring. Because group-level scores capture relational information between samples, they require more compute than individual-level scores, raising the question of where practitioners with limited compute should spend a group-scoring budget. We show that this question has a principled answer: the gap between individual and group scoring is bounded by a quantity that is cheap to compute, which serves as an efficient diagnostic that reliably predicts the gap in practice and directs group oracle calls to the groups that are most poorly reflected by individual-level scoring.

\paragraph{Characterizing the oracle gap.} We start by considering the question: \emph{how far is the individual proxy oracle from the true group oracle?} To answer this, we turn to Proposition~\ref{prop:single_group_gap}, which bounds how far the individual proxy oracle can be from the true group oracle by the \emph{within-group gradient diversity} (i.e., the spread of per-sample gradients within a group).

\begin{proposition}[Oracle gap]
\label{prop:single_group_gap}
The individual proxy and true group oracles differ by at most
\begin{equation}
\big|\,\Iproxy(G, \theta) - \mathcal{I}_{\mathrm{orcl}}(G, \theta)\,\big| \;\le\; \frac{\beta\,\eta^2}{2}\,\sigma_G^2,
\label{eqn:single_group_gap}
\end{equation}
where $\beta \ge \norm{\nabla^2 L(\theta)}$ bounds the curvature of the reference loss and $\sigma_G^2 = \frac{1}{|G|}\sum_{z\in G}\norm{g_z - g_G}^2$ is the within-group gradient diversity. Proof in Appendix~\ref{appendix:oracle_scoring}.
\end{proposition}

Intuitively, the more the samples pull the model in different directions, the more their disagreements cancel in the group's mean update, an interaction invisible to a proxy that scores samples in isolation. 

\begin{wrapfigure}[15]{r}{0.55\linewidth}
\centering
\begin{subfigure}[t]{0.40\linewidth}
    \centering
    \includegraphics[width=\linewidth]{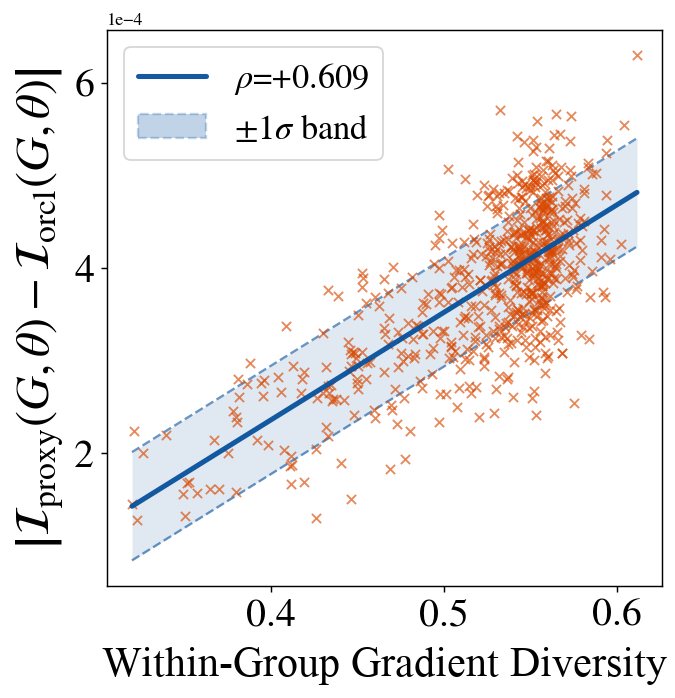}
    \caption{Oracle gap vs Diversity.}
    \label{fig:divsweep}
\end{subfigure}
\begin{subfigure}[t]{0.55\linewidth}
    \centering
    \includegraphics[width=\linewidth]{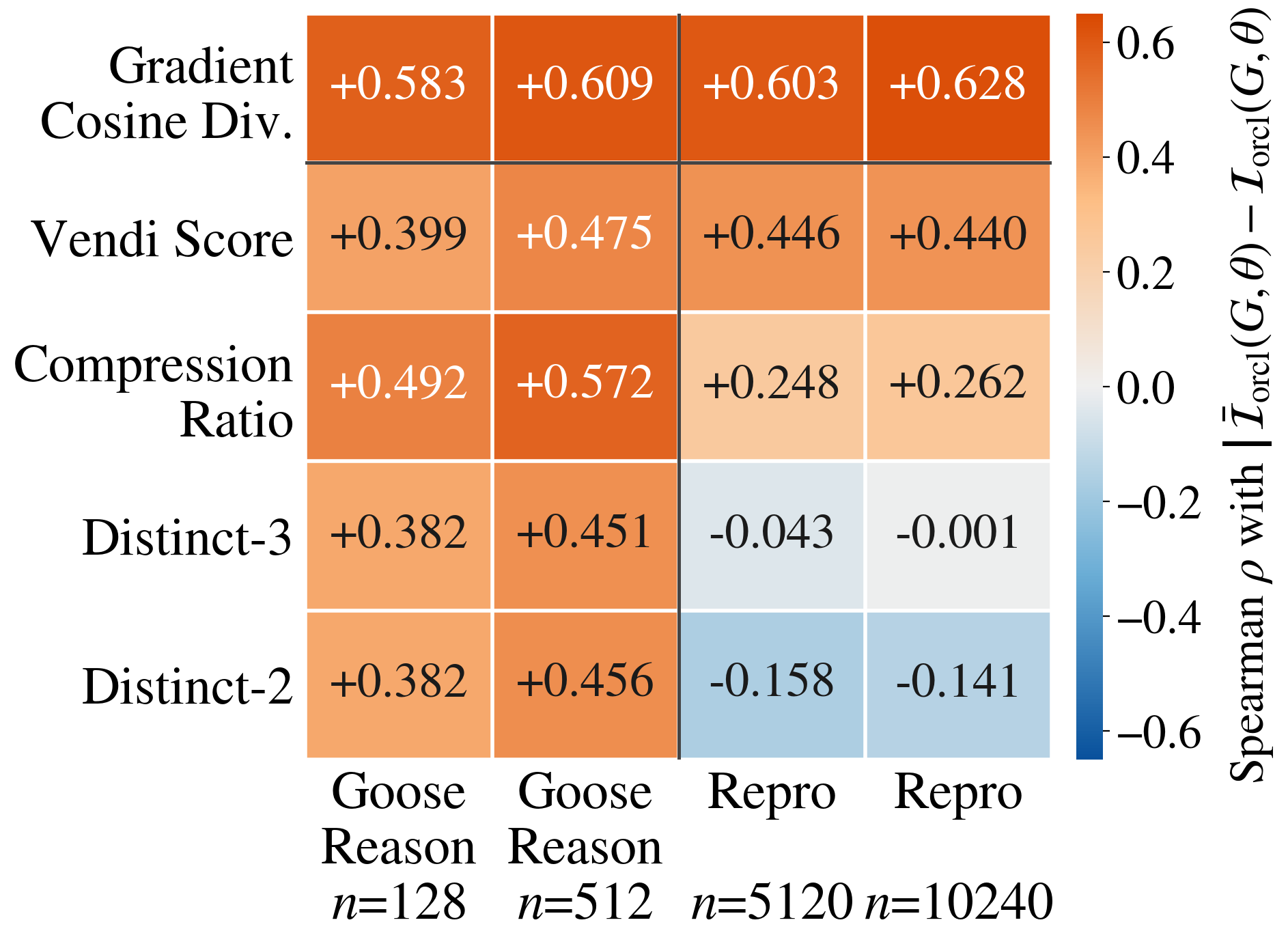}
    \caption{Oracle gap heatmap.}
    \label{fig:rho_heatmap}
\end{subfigure}
\caption{Gap between the group and individual proxy oracle diverges as within-group diversity increases.}
\label{fig:divergence_sweep_main}
\end{wrapfigure}

\begin{comment}
\begin{wrapfigure}[16]{r}{0.45\linewidth}
\includegraphics[width=\linewidth]{images/recommendations/rho_heatmap_slim.png}
\label{fig:rho_heatmap}
\caption{Gap between the group oracle $\mathcal{I}_{\mathrm{orcl}}(G, \theta)$ and individual proxy oracle $\Iproxy(G_j, \theta)$ diverges as intra-group diversity increases (on groups of $10240$ and $512$ samples).}
\label{fig:divergence_sweep_main}
\end{wrapfigure}
\end{comment}

\begin{comment}
\begin{wrapfigure}[16]{r}{0.5\linewidth}
\centering
\begin{subfigure}[t]{0.5\linewidth}
    \centering
    \includegraphics[width=\linewidth]{images/background/gradiv_cos_repro_n10240_relabel.png}
    \caption{Repro \citep{yu2025reprotraininglanguagemodels}.}
    \label{fig:divsweep_main_repro}
\end{subfigure}
\begin{subfigure}[t]{0.48\linewidth}
    \centering
    \includegraphics[width=\linewidth]{images/background/gradiv_cos_goose_n512_relabel.png}
    \caption{GooseReason \citep{lu2026goldengoosesimpletrick}.}
    \label{fig:divsweep_main_goose}
\end{subfigure}
\caption{Gap between the group oracle $\mathcal{I}_{\mathrm{orcl}}(G, \theta)$ and individual proxy oracle $\Iproxy(G_j, \theta)$ diverges as intra-group diversity increases (on groups of $10240$ and $512$ samples).}
\label{fig:divergence_sweep_main}
\end{wrapfigure}
\end{comment}

%on Repro \citep{yu2025reprotraininglanguagemodels} and GooseReason \citep{lu2026goldengoosesimpletrick}, for groups of $10{,}240$ and $512$

\paragraph{Gradient diversity predicts the oracle gap.}
The bound in Proposition~\ref{prop:single_group_gap} holds empirically: Figure~\ref{fig:divsweep} shows that as within-group gradient diversity increases, the gap between $\Iproxy$ and $\mathcal{I}_{\mathrm{orcl}}$ widens (see Appendix~\ref{appendix:tables_figures_divsweep} for additional settings).

Gradient diversity is also the strongest predictor of the gap among additional diversity metrics (see Appendix~\ref{appendix:diversity_metrics} for details), as shown in Figure~\ref{fig:rho_heatmap}. Its Spearman correlation with the gap is high ($\rho = 0.58$--$0.63$) across both datasets and all group sizes. By contrast, the best alternatives peak at $0.48$ (Vendi) and $0.57$ (compression ratio), and they vary widely. Moreover, token-level metrics are unreliable: distinct-2/-3 correlate moderately on GooseReason but not on Repro.

Given these findings, we show how gradient diversity can be used as a diagnostic to decide where to spend group oracle compute.

\paragraph{Gradient diversity as a budget allocator.} Proposition~\ref{prop:single_group_gap} implies that oracle calls are best spent on groups with high within-gradient diversity, where the proxy can be far off and is therefore most likely to misrank groups. This suggests a simple allocation strategy. Formally, given a pool of groups $\mathcal{G} = \{G_1, \dots, G_H\}$ and a budget of $B$ group oracle calls, we define a \emph{diagnostic} as a cheap score $\nu : \mathcal{G} \to \mathbb{R}$, where a larger value signals a larger predicted oracle gap. Given $\nu$, we audit the $B$ top-ranked groups, denoted by $\mathcal{A}_\nu$,
\begin{equation}
  \mathcal{A}_\nu \;=\; \argtopk_{G \in \mathcal{G}}^{\,B} \; \nu(G),
\label{eqn:audit_set}
\end{equation}
and score every group with the best oracle we can afford for it,
\begin{equation}
  \widehat{\mathcal{I}}_{\mathcal{A}}(G, \theta) \;=\;
  \begin{cases}
    \mathcal{I}_{\mathrm{orcl}}(G, \theta), & G \in \mathcal{A}_\nu,\\[2pt]
    \Iproxy(G, \theta),                     & \text{otherwise.}
  \end{cases}
\label{eqn:budget_allocation}
\end{equation}

\paragraph{Setup.} Given the above procedure, we assess the performance of different diagnostics. The diagnostics we compare are as follows: 
\begin{enumerate}[topsep=0pt, itemsep=3pt, parsep=0pt]
\item \emph{GradDiv\textsubscript{full}}: uses $\nu(G) = \sigma_G^2$ from Proposition~\ref{prop:single_group_gap} 
\item \emph{GradDiv\textsubscript{$q$}}: uses the same estimator on a uniform subsample of $\lceil q\,|G| \rceil$ samples, $q \in (0, 1]$ (i.e., $q\,|G|$ gradients per group)
\item \emph{Gradient-free}: use Vendi, compression ratio, or distinct-$n$ (Appendix~\ref{appendix:diversity_metrics}; no gradients)
\item \emph{Random allocation}: draws $\nu(G) \sim \mathrm{Unif}(0,1)$ i.i.d. 
\end{enumerate}

Moreover, in Eqn~\ref{eqn:budget_allocation} we substitute $\Iproxy$ with GradSim (from Section~\ref{sec:collapse_setup}), as it is a cheap estimate of the individual proxy oracle more commonly used in practice~\citep{pruthi2020estimating, xia2024lessselectinginfluentialdata}.

We test these diagnostics on the constructed groups from Section~\ref{sec:single_group}, where the individual proxy oracle is blind to the group effects. We measure how well each diagnostic recovers the group oracle's ranking with two metrics: \emph{top-$k\%$ recall}, the fraction of the group oracle's top-$k\%$ groups that also rank in the top $k\%$ under $\widehat{\mathcal{I}}_{\mathcal{A}}$, and Spearman $\rho$ between $\widehat{\mathcal{I}}_{\mathcal{A}}$ and $\mathcal{I}_{\mathrm{orcl}}$ over the full pool. 
\begin{figure*}[t]
\centering
\begin{subfigure}[t]{0.24\textwidth}
    \centering
    \includegraphics[width=\linewidth]{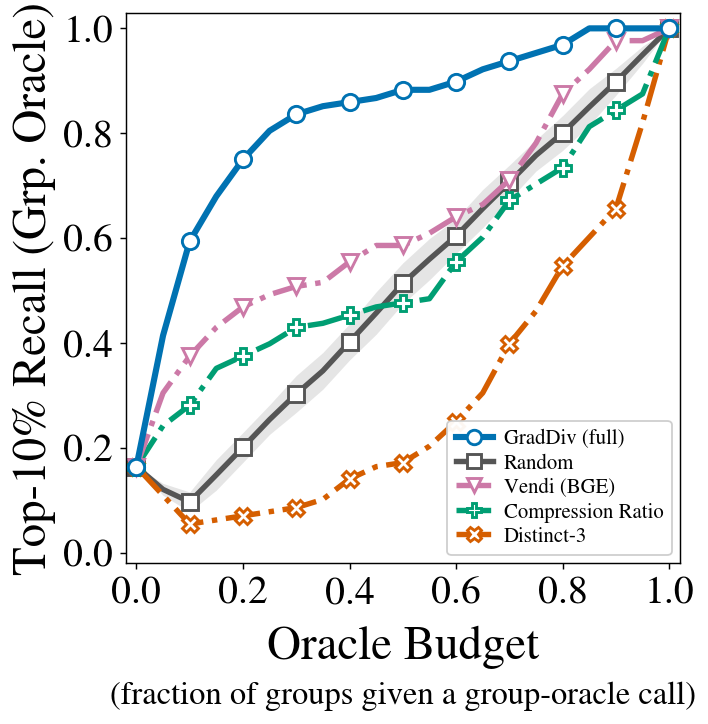}
    \caption{Choice of diagnostic.}
    \label{fig:allocation_heuristics_repro_n5120_recall10}
\end{subfigure}
\hfill
\begin{subfigure}[t]{0.24\textwidth}
    \centering
    \includegraphics[width=\linewidth]{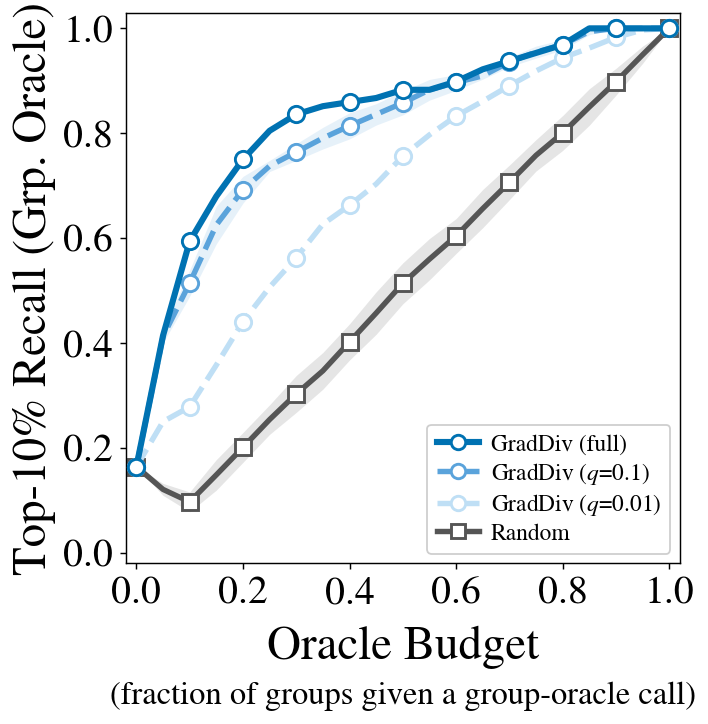}
    \caption{Gradient subsampling.}
    \label{fig:allocation_fraction_repro_n5120_recall10}
\end{subfigure}
\hfill
\begin{subfigure}[t]{0.24\textwidth}
    \centering
    \includegraphics[width=\linewidth]{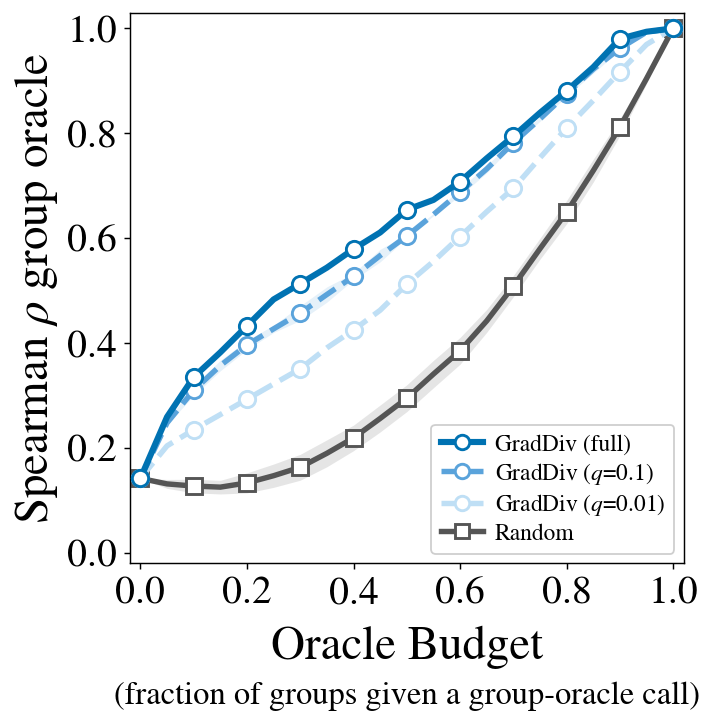}
    \caption{Rank agreement.}
    \label{fig:allocation_fraction_repro_n5120_spearman}
\end{subfigure}
\hfill
\begin{subfigure}[t]{0.24\textwidth}
    \centering
    \includegraphics[width=\linewidth]{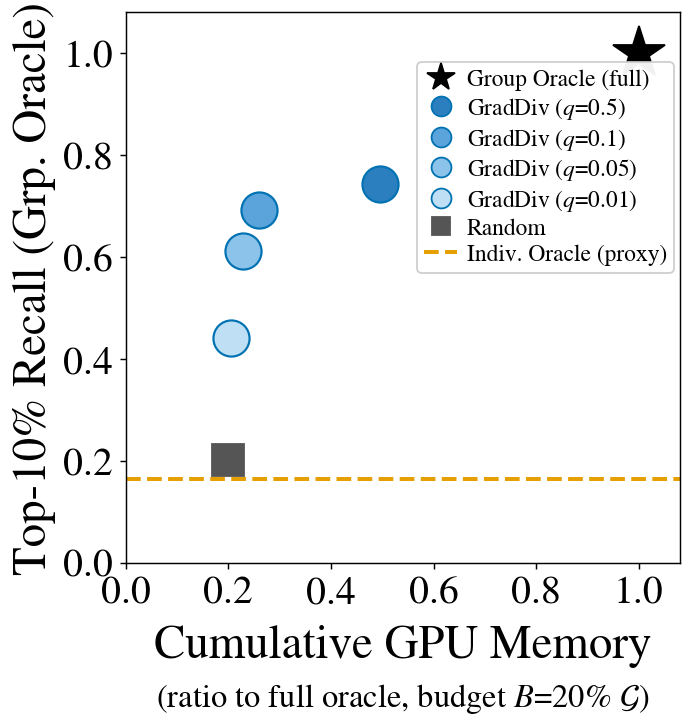}
    \caption{Memory versus quality.}
    \label{fig:memtime_vs_quality_f02_repro_n5120_recall10}
\end{subfigure}
\caption{Allocating a limited group-oracle budget with cheap diagnostics. See Appendix~\ref{appendix:budget_plots} for additional analysis plots.}
\label{fig:budget}
\vspace{-0.5cm}
\end{figure*}

\paragraph{Finding 1: Gradient diversity can recover the group oracle.} Figure~\ref{fig:allocation_heuristics_repro_n5120_recall10} shows that GradDiv\textsubscript{full} is best at retrieving the top-ranked groups under a fixed budget: at a budget of $20\%$ group-oracle calls, it recovers $75\%$ of the top-$10\%$ groups, compared to $20\%$ under random allocation. In contrast, the gradient-free diagnostics recover only part of this signal: Vendi and compression ratio reach $0.44$ and $0.40$ recall at the same budget, while the token-level distinct-$3$ falls below the random baseline.

\paragraph{Finding 2: Smaller gradient subsamples suffice.} Next, we examine whether the diagnostic stage itself can be made cheaper. Figures~\ref{fig:allocation_fraction_repro_n5120_recall10} and~\ref{fig:allocation_fraction_repro_n5120_spearman} show that GradDiv\textsubscript{$q$} gives similar benefit at $q = 0.1$ (only sampling $10\%$ of a group's gradients) compared to GradDiv\textsubscript{full} on both recall ($0.69$ versus $0.75$ at a $20\%$ budget) and Spearman correlation ($0.40$ versus $0.44$). This indicates that only a coarse estimate of GradDiv\textsubscript{full} is needed to rank groups, allowing the diagnostic to avoid much of the per-sample gradient cost.

\paragraph{GPU memory analysis.} Finally, we examine the compute cost. Figure~\ref{fig:memtime_vs_quality_f02_repro_n5120_recall10} reports the total cumulative memory usage at a $20\%$ budget (relative to scoring all groups using the group oracle), showing that GradDiv\textsubscript{$q$} exhibits a better memory-performance trade-off compared to full group oracle scoring and random allocation. In particular, at $q = 0.1$, GradDiv\textsubscript{$q$} recovers $74\%$ of the group oracle's recall while using $27\%$ of its memory. Full details on memory computation can be found in Appendix~\ref{appendix:budget_analysis}.
\section{Conclusion}
In this paper, we asked which signals identify what synthetic data is worth training on, and showed that \emph{group-level} signals, which capture interactions among samples, are needed. Across pre-training and post-training, synthetic datasets indistinguishable under individual-level influence differ under group-level influence, which also yields better downstream performance, with gains concentrated in generative capability. Among widely used data curation methods, only those using group-level signals improve over baseline, with further gains when relational weights are amplified. Finally, we recommend within-group gradient diversity as a diagnostic for when group-level signals are worth their cost. Our findings motivate future work on scalable group-level curation methods that model this relational structure directly.
\section{Acknowledgments}
This research was partially funded by the National Institute of Standards and Technology (ror.org/05xpvk416) and Carnegie Mellon University (https://ror.org/05x2bcf33) AI Measurement Science and Engineering Center (AIMSEC). 

We sincerely appreciate Saahith Janapati and Chen Xu for providing helpful feedback on this work.
\bibliography{references}
\bibliographystyle{iclr2027_conference}

%%%%%%%%%%%%%%%%%%%%%%%%%%%%%%%%%%%%%%%%%%%%%%%%%%%%%%%%%%%%
% APPENDIX
%%%%%%%%%%%%%%%%%%%%%%%%%%%%%%%%%%%%%%%%%%%%%%%%%%%%%%%%%%%%
\appendix
\newpage
\section*{Appendix}

% Appendix-only table of contents. Capture must start BEFORE the sections below
% so their TOC entries are mirrored into the .apc file; \AppendixTOC then prints
% them (sections + subsections only).
\StartAppendixTOCcapture
\AppendixTOC
\vspace{1em}

\section{Related Work}\label{appendix:related_work}

\textbf{Synthetic data for LLM Training.} Synthetic data is now used at every stage of LLM training. In pre-training, it supplements or rewrites web text to raise its information density, either through textbook-style generation \citep{li2023textbooks, abdin2024phi4technicalreport} or by rephrasing and recycling existing documents \citep{maini2024rephrasing, nguyen2025recycling, yu2025reprotraininglanguagemodels}. In mid-training, it is used to target specific capabilities before post-training, such as long-context, math, and code \citep{grattafiori2024llama3herdmodels, yang2024qwen25mathtechnicalreportmathematical, olmo20252olmo2furious, zeng2026davincidevagentnativemidtrainingsoftware}. In post-training, it supplies instruction-following and reasoning supervision that is difficult to source from humans at scale, through self-generated instructions \citep{wang2023self, taori2023stanford, mitra2024agentinstructgenerativeteachingagentic}, distillation from larger teachers \citep{adler2024nemotron, guo2025deepseekr1}, and self-improvement over verified model outputs \citep{zelikman2022star, singh2023beyond, yuan2023scalingrelationshiplearningmathematical}. It is also used to construct the environments and tasks that post-training runs over, including tool-use APIs and trajectories \citep{qin2023toolllm, liu2024apigen} and verifiable tasks for reinforcement learning \citep{lu2026goldengoosesimpletrick}.

\textbf{Synthetic data selection.} While past work has shown that individual-level data selection is effective for organic data (including representation-based \citep{pezeshkpour2021empirical}, gradient-based \citep{grosse2023studying, kwon2024datainf, xia2024lessselectinginfluentialdata}, and lexical methods \citep{bm25}), synthetic data often exhibits different characteristics. These characteristics include lower information density \citep{maini2024rephrasing, kang2025demystifying}, under-representation of rare entities or domains \citep{zhu2024synthesize, yu2026generating}, and more templated phrasing \citep{niklaus2026can}. Synthetic data curation therefore requires additional considerations, such as task complexity, environment construction, and model biases \citep{long-etal-2024-llms, Nad__2025}. Selection strategies for synthetic data largely score each generation on its own, filtering by quality or reward models \citep{chen2024alpagasus, wettig2024qurating} or by similarity to a target task \citep{xia2024lessselectinginfluentialdata, bm25}. A smaller line of work instead scores generations against one another, promoting diversity or limiting redundancy across the retained set \citep{chen2024diversity, zhu2024synthesize}.

\textbf{Group-Level Data Influence.} Since data inclusion decisions are typically made over groups rather than individual samples, a line of work has examined how influence estimates behave at the group level. Group influence is not the sum of its parts: the effects of individual samples can amplify or cancel one another, and the change induced by any single sample is small relative to that of the group it belongs to \citep{koh2019accuracy, basu2020second, NEURIPS2024_d8684e49}. The most common approach nonetheless scores a group by aggregating the influence estimates of its samples \citep{wang2023farewell, zhangharnessing, zhang2025fairshare}. Cooperative-game formulations instead average a sample's marginal contribution across subsets \citep{ghorbani2019data, jia2019towards, wang2023data}, though the resulting scores remain per-sample and are again summed to value a group. A third line of work models the group directly, either by predicting model behavior as a function of the training subset \citep{ilyas2022datamodels, park2023trak, engstrom2024dsdm} or by estimating influence alongside relational information between samples in the same group \citep{wang2024greats, sun2025enhancingtrainingdataattribution, yu2025data}.
\section{Supplementary Materials}\label{appendix:supplements}

This section presents additional details on the experiments from Sections \ref{sec:background}, \ref{sec:single_group}, \ref{sec:data_influence_methods}, and \ref{sec:diagnostics}.

\subsection{Oracle Scoring Configuration}\label{appendix:oracle_scoring_parameters}
In this section, we present the parameters for the oracle scoring procedure from Section~\ref{sec:single_group_tasks}, for both the pre-training and post-training settings. Across all oracle scoring settings, we take a candidate group $G = \{z_j\}_{j=1}^{|G|}$ and perform one AdamW step on $G$, starting from a model with parameters $\theta_{D_{\mathrm{train}}}$ previously obtained by training on $D_{\mathrm{train}}$ (in our setting, $\theta_{D_{\mathrm{train}}}$ is a pre-trained model checkpoint). We evaluate both $\theta_{D_{\mathrm{train}}}$ and $\theta_{D_{\mathrm{train}} \cup G}$ on a reference set $D_{\mathrm{ref}} = \{z'_i\}_{i=1}^{M}$. We reset the weights to $\theta_{D_{\mathrm{train}}}$ between candidate groups. The parameter details are shown in Table~\ref{tab:oracle-scoring-params}.

\begin{table}[ht]
\centering
\small
\setlength{\tabcolsep}{4pt}
\renewcommand{\arraystretch}{1.15}
\begin{tabularx}{\textwidth}{l >{\raggedright\arraybackslash}X >{\raggedright\arraybackslash}X}
\toprule
\textbf{Parameter} & \textbf{Pre-Training Setting} & \textbf{Post-training Setting} \\
\midrule
Probe model      & Repro-400M (trained on 6.59B Repro tokens \citep{yu2025data}) & Qwen2.5-1.5B-Instruct \\
Tokenizer        & Pythia-410m & Qwen2.5-1.5B-Instruct \\
Optimizer        & AdamW, 1 step & AdamW, 1 step \\
Learning rate    & \texttt{2e-5} &  1e-5 \\
Precision        & \texttt{bf16} &  \texttt{bf16} \\
Grad.\ accumulation & Yes & Yes \\
Micro batch size & \texttt{2} & \texttt{1} \\
Candidate group sizes & $5120$, $10240$ samples per group & $128$, $512$ samples per group \\
Sequence length  & 2048 & 2048 \\
Reference set    & FLAN (128 samples) & GPQA (128 held-out samples) \\
\bottomrule
\end{tabularx}
\caption{Oracle scoring configurations for the pre-training and post-training settings.}
\label{tab:oracle-scoring-params}
\end{table}

\subsection{Proxy-Matched Diversity Sampling}\label{appendix:cluster_sampling}

This section details the sampling procedure summarized in Section~\ref{sec:single_group_setup}. The goal is to create a pool of groups in which every group falls within an individual proxy oracle score range, and diversity is the free variable. We describe the process below.

\paragraph{Embedding and clustering.} Given a large corpus, $D_{train}$, we first partition $D_{train}$ into multiple clusters. The cluster partition $\mathcal{C} = \{C_1, \dots, C_k\}$ that defines coverage is built once per setting, before any group creation. We embed each example with BGE-base-en-v1.5 \citep{xiao2024c}. BGE accepts at most $512$ tokens whereas our sequences are $2048$ tokens long, so each sequence is split into four consecutive $512$-token chunks, each chunk is encoded independently, and we take the $768$-dimensional [CLS] hidden state as the chunk embedding. The four chunk embeddings are averaged to give one $768$-dimensional vector per sequence. We then partition the resulting matrix with $k$-means  at $k = 20$, using Euclidean distance on the embeddings. Figure~\ref{fig:cluster_topics_pie} shows the topics covered by the resulting clusters for each corpus.

\paragraph{Band selection.} We first score a reference pool of $100$ uniformly-random groups from $D_{train}$ and take the central half of their individual proxy-oracle distribution as the band $[\tau_{\mathrm{lo}}, \tau_{\mathrm{hi}}]$. For the pre-training setting this gives $[-6.85, -6.57] \times 10^{-3}$ (center $\tau = -6.71 \times 10^{-3}$, width $2.8 \times 10^{-4}$); for the post-training setting, $[-8.7, -8.1] \times 10^{-4}$ (center $\tau = -8.4 \times 10^{-4}$, width $6 \times 10^{-5}$). Restricting to the central half rather than the full reference range keeps every group's individual proxy oracle inside the distribution a random selector would produce, so the pool is not itself an outlier population.  

\paragraph{Coverage schedule.} Each group $G_j$, where $j = 1, \dots, H$, is assigned a coverage level

\[   r_j = \mathrm{round}\!\left(     \exp\!\Big( \log r_{\min} + \tfrac{j-1}{H-1}\,       \big(\log r_{\max} - \log r_{\min}\big) \Big)   \right),   \qquad r_{\min} = 1,\ r_{\max} = k = 20 . \] 

\paragraph{Target assignment and decorrelation.} We form the grid $\beta = \mathrm{linspace}(\tau_{\mathrm{lo}}, \tau_{\mathrm{hi}}, H)$ and set $\tau_j = \beta_{\pi(j)}$ for a seeded uniform random permutation $\pi$. 

\paragraph{Member selection.} Given $(r_j, \tau_j)$, we allocate the group's examples across its $r_j$ clusters and select within each cluster by \emph{stratified two-sided sampling}: within a cluster, examples are sorted by individual proxy score, split at $\tau_j$, and drawn from both sides in the mixing proportion that makes the subset mean equal $\tau_j$, with picks spread uniformly over each side's rank range. This deliberately allows a \emph{wide} spread of individual scores while pinning their mean to $\tau_j$. Table~\ref{tab:divsweep_pools} reports the four pools. In every case the realized individual proxy-oracle residual is within tolerance for all $H$ groups.

\begin{table}[t] \centering \small \begin{tabular}{llrrrr} \toprule Setting & size & Quantity & $\max_j |\Iproxy(G_j,\theta) - \tau_j|$ & $\epsilon$ &   diversity range \\ \midrule Pretraining & $5120$  & $1000$ & $5.0 \times 10^{-5}$ & $5 \times 10^{-5}$ & $0.095$--$0.212$ \\ Pretraining & $10240$ & $500$  & $5.0 \times 10^{-5}$ & $5 \times 10^{-5}$ & $0.096$--$0.200$ \\ Post-training                  & $128$   & $1000$ & $5.0 \times 10^{-6}$ & $5 \times 10^{-6}$ & $0.156$--$0.435$ \\  Post-training                 & $512$   & $500$  & $4.8 \times 10^{-6}$ & $5 \times 10^{-6}$ & $0.157$--$0.425$ \\ \bottomrule \end{tabular} \caption{Diversity-swept candidate pools. Diversity is $1 - \overline{\cos}(G)$. Every group's individual proxy oracle lands within $\epsilon$ of its assigned band target, so the pools vary in diversity at (near-)fixed individual proxy oracle.} 
\label{tab:divsweep_pools} \end{table}
\newpage
\subsection{Data Influence Setup}\label{appendix:tables_figures_task_settings}

Table~\ref{tab:task_settings_full} lists the full settings for the pre-training and post-training experiments in Section~\ref{sec:single_group_setup}, organized by phase. Each training setting is run under two configurations that differ in group size. In the \emph{small}-group configuration, the candidate pool contains more, smaller groups ($H{=}1000$), and more groups are selected ($k{=}40$ for pre-training, $k{=}25$ for post-training). In the \emph{large}-group configuration, the pool contains fewer, larger groups ($H{=}500$), and fewer groups are selected ($k{=}20$ and $k{=}7$). 

Group sizes are integer multiples of the training batch size, so each group spans a whole number of optimizer steps (10 or 20 for pre-training, 4 or 16 for post-training). The base model is used both for oracle scoring (Phase 2) and as the training initialization (Phase 3).

\begin{table}[ht]
\centering
\small
\caption{Task settings for the pre-training and post-training experiments, organized by the three phases in Section~\ref{sec:single_group_setup}. Each training setting is run under two configurations: a \emph{small}-group and a \emph{large}-group setting.}
\label{tab:task_settings_full}
\begin{tabular}{lll}
\toprule
 & \textbf{Pre-training} & \textbf{Post-training} \\
\midrule
\multicolumn{3}{@{}l}{\textit{Phase 1: Candidate pool construction}} \\
\midrule
Source dataset                      & Repro-Rephrased-72B            & GooseReason \\
Semantic clusters                   & 20                             & 20 \\
Group sizes $s$                     & $5120$ {\scriptsize(small)}, $10240$ {\scriptsize(large)} & $128$ {\scriptsize(small)}, $512$ {\scriptsize(large)} \\
Candidate pool $H$                  & $1000$ {\scriptsize(small)}, $500$ {\scriptsize(large)} & $1000$ {\scriptsize(small)}, $500$ {\scriptsize(large)} \\
\midrule
\multicolumn{3}{@{}l}{\textit{Phase 2: Data selection}} \\
\midrule
Reference set                       & FLAN (128 samples)                     & GPQA-extend (128 samples) \\
Selected groups $k$                 & $40$ {\scriptsize(small)}, $20$ {\scriptsize(large)}    & $25$ {\scriptsize(small)}, $7$ {\scriptsize(large)} \\
Gumbel Temperature $\lambda$        & 0.5                     & 0.5 \\
\midrule
\multicolumn{3}{@{}l}{\textit{Phase 3: Training and evaluation}} \\
\midrule
Base model $M_{\theta_0}$           & Repro-400M checkpoint          & Qwen2.5-1.5B-Instruct \\
Batch size                          & 512 sequences                  & 32 prompts ($\times 8$ completions) \\
Optimizer                           & cosine (wd $0.033$, no warmup) & GRPO (warmup $0.05$, $\beta{=}0$, temp $0.6$) \\
Learning rate                       & $3\times10^{-3}$               & $1\times10^{-6}$ \\
Training steps                      & 400                            & 100 \\
Benchmarks                          & \makecell[l]{ARC-C, CSQA, CoQA, HellaSwag,\\LAMBADA, MMLU, PIQA, SQuAD} & \makecell[l]{MBPP, MBPP+, GPQA-D,\\BBH, IFEval, AIME24} \\
\bottomrule
\end{tabular}
\end{table}

\newpage
\subsection{Data Curation Setup}\label{appendix:supp_data_curation_exp}

Table~\ref{tab:collapse_settings} lists the full settings for the data curation experiments in Section~\ref{sec:data_influence_methods}, organized by phase.

\begin{table}[h]
\centering
\small
\caption{Experimental settings for the data curation experiments, organized by the three phases in Section~\ref{sec:collapse_setup}.
The pool is held at a fixed size while the rephrase factor $r$ controls the degree of
near-duplicate redundancy.}
\label{tab:collapse_settings}
\begin{tabular}{ll}
\toprule
\multicolumn{2}{@{}l}{\textit{Phase 1: Synthetic pool construction}} \\
\midrule
Source documents                & Organic web text (DCLM-RefinedWeb) \\
Rephraser                       & Repro rephraser-4B \citep{yu2025reprotraininglanguagemodels} \\
Sampling                        & temperature $1.0$, top-$p$ $0.9$ \\
Pool size $N$                   & $716{,}800$ sequences \\
Rephrase factor $r$             & $\{1, 10, 50\}$ \\
Unique source docs ($N/r$)      & $716{,}800\,/\,71{,}680\,/\,14{,}336$ \\
Sequence length                 & $2048$ tokens \\
\midrule
\multicolumn{2}{@{}l}{\textit{Phase 2: Data selection}} \\
\midrule
Selection methods               & Random, GradSim, FineWeb-Edu, BM25, GMRel, GREATS \\
Reference set                   & FLAN (128 samples) \citep{chung2024scaling} \\
Embedding model $\phi(\cdot)$   & BGE \\
Clusters                        & $30$ ($k$-means), largest $20$ retained \\
Group size $s$                  & $5{,}120$ sequences \\
Groups per run                  & $20$ (one per retained cluster), $102{,}400$ sequences \\
\midrule
\multicolumn{2}{@{}l}{\textit{Phase 3: Training and evaluation}} \\
\midrule
Base model $M_{\theta_0}$       & Repro-400M \citep{yu2025reprotraininglanguagemodels} \\
Batch size                      & $512$ sequences \\
Training steps                  & $200$ (one epoch over the selection) \\
Learning rate                   & $3\times10^{-3}$ (cosine, no warmup) \\
Weight decay                    & $0.033$ \\
Seed                            & $124$ \\
Benchmarks                      & \makecell[l]{ARC-C, CSQA, CoQA, HellaSwag,\\LAMBADA, MMLU, PIQA, SQuAD} \\
\bottomrule
\end{tabular}
\end{table}

\newpage
\subsection{Data Selection Methods}\label{appendix:group_influence_methods}

This section presents additional details regarding the data selection methods from Section~\ref{sec:collapse_setup}.

\subsubsection{FineWeb-Edu}
The FineWeb-Edu \citep{penedo2024fineweb} classifier is a lightweight encoder-based regressor (a BERT-based encoder with a single-output regression head) that rates a document's educational value from $0$ to $5$. We score every sequence in the pool independently. The classifier accepts at most $512$ tokens. Therefore, for a given data sample, we score each chunk, and average the chunk scores to obtain a single score per sequence. For groups of samples, we average the individual sequence scores to obtain a single score per group.

\subsubsection{BM25}
BM25 \citep{bm25} is a standard information retrieval method that ranks training samples based on lexical similarity to a query. The reference set sequences are decoded to text and whitespace-tokenized to form a BM25-Okapi document collection of documents, built once; each candidate sequence is then decoded and whitespace-tokenized in the same way and issued as a query against that collection. The score of a sequence is the mean of its BM25 scores, so higher scores indicate greater lexical relevance to the reference set and sequences are again ranked in descending order.

\subsubsection{GradSim}\label{appendix:gradsim}
GradSim \citep{pruthi2020estimating} estimates influence with the first-order inner product between the gradient of the loss on a training sample and the gradient of the loss on the reference set. We first compute a single reference gradient by back-propagating the token-weighted mean loss over the reference set at the base checkpoint $\theta$; the base checkpoint is the same one used for oracle scoring in each setting (Table~\ref{tab:oracle-scoring-params}). We then compute, for each candidate sequence $z$, the gradient $\nabla_\theta \mathcal{L}(z, \theta)$ at that same checkpoint. Each candidate is scored by the cosine similarity between its gradient and the reference gradient.

\subsubsection{GMRel}
GMRel \citep{yu2025data} trains a relational data influence model that estimates oracle data influence by weighting individual influence scores according to relationships among training data. We train one relational model per task setting, in three stages.

\textbf{Architecture.} The relational model is a BGE-base-en-v1.5 encoder $\psi(\cdot)$ with a
one-dimensional linear head $h(\cdot)$, plus two scalars learned jointly with the encoder: a
relationship temperature $\tau$ and a relational weight $\alpha$. A sequence's representation
$\psi(z)$ is the $\ell_2$-normalized \texttt{[CLS]} state of each of its $512$-token chunks,
averaged over the four chunks of a $2048$-token sequence, and its individual influence is
$\hat{I}(z) = h(\psi(z))$. For an ordered group $G = (z_1, \dots, z_n)$, the \emph{relationship}
between two members is the temperature-scaled representation inner product
\[
  R(z_i, z_j) \;=\; \frac{\langle \psi(z_i), \psi(z_j) \rangle}{\tau},
\]
and the group-aware influence of the $t$-th member is its individual influence modulated by its
average relationship to the members preceding it,
\[
  \hat{I}_{\mathrm{rel}}\big(z_t \mid z_1, \dots, z_{t-1}\big)
  \;=\; \alpha\, \hat{I}(z_t) \left( 1 - \frac{1}{t-1} \sum_{i < t} R(z_t, z_i) \right),
  \qquad t \ge 2,
\]
with $\hat{I}_{\mathrm{rel}}(z_1) = \hat{I}(z_1)$ for the first member. 

\textbf{Training data.} The relational model is trained on \emph{oracle rollouts} generated on the same base model used for oracle scoring
in that setting. In the pre-training setting, rollout candidates are drawn from the tokenized Repro-Rephrased-72B corpus (excluding data used for training in Section~\ref{sec:data_influence_methods}) \citep{yu2025reprotraininglanguagemodels}, such that each candidate contains four $2048$-token sequences, giving $800{,}000$ candidate sequences; the reference set is
the same $128$-sample FLAN set used for oracle scoring. In the post-training setting, candidates are
drawn from GooseReason \citep{lu2026goldengoosesimpletrick} against the
$128$-sample GPQA-extend reference set.

Each rollout samples $K = 10$ candidate sequences without replacement, resets the model to its base
weights, and then trains on the candidates \emph{sequentially} with AdamW, re-evaluating reference-set
loss after each one. The recorded label for a candidate is the cumulative change in reference loss. We collect $2{,}400$ rollouts ($24{,}000$ labeled
candidates) in the pre-training setting and $800$ rollouts ($8{,}000$ labeled candidates) in the
post-training setting.

\textbf{Training parameters.} The relational model is fit by MSE between its group-aware scores
$\hat{I}_{\mathrm{rel}}$ and these per-step deltas, holding out $200$ rollouts for evaluation. We
select the best checkpoint by Spearman correlation on the held-out rollouts. Table~\ref{tab:gmrel-params} lists the full configuration for both the rollout
generation and the model fitting.

\begin{table}[h]
\centering
\small
\setlength{\tabcolsep}{4pt}
\renewcommand{\arraystretch}{1.15}
\begin{tabularx}{\textwidth}{l >{\raggedright\arraybackslash}X >{\raggedright\arraybackslash}X}
\toprule
\textbf{Parameter} & \textbf{Pre-Training Setting} & \textbf{Post-training Setting} \\
\midrule
\multicolumn{3}{l}{\emph{Rollout generation (supervision)}} \\
Base model            & Repro-400M (as in Table~\ref{tab:oracle-scoring-params}) & Qwen2.5-1.5B-Instruct \\
Candidate pool        & Repro-Rephrased-72B & GooseReason \\
Source tokenizer      & Pythia-410m & Qwen2.5-1.5B-Instruct \\
Reference set         & FLAN (128 samples) & GPQA-extend (128 samples) \\
Candidates per rollout $K$ & $10$ & $10$ \\
Inner optimizer       & AdamW, 4 steps per candidate & AdamW, 1 step per candidate \\
Inner learning rate   & \texttt{2e-5} & \texttt{1e-5} \\
$(\beta_1, \beta_2)$, weight decay & $(0.9, 0.95)$, $0$ & $(0.9, 0.999)$, $0$ \\
Gradient clipping     & $1.0$ & none \\
Precision             & \texttt{bf16} & \texttt{bf16} \\
Rollouts collected    & $2{,}400$ ($24{,}000$ labeled candidates) & $800$ ($8{,}000$ labeled candidates) \\
\midrule
\multicolumn{3}{l}{\emph{Relational model fitting}} \\
Encoder               & BGE-base-en-v1.5 & BGE-base-en-v1.5 \\
Chunking per sequence & $4 \times 512$ tokens & $4 \times 512$ tokens \\
Objective             & MSE on per-step reference-loss deltas & MSE on per-step reference-loss deltas \\
Learning rate         & \texttt{5e-5} & \texttt{5e-5} \\
Weight decay          & $0.01$ & $0.01$ \\
Epochs                & $5$ & $5$ \\
Batch size            & $1$, gradient accumulation $4$ & $1$, gradient accumulation $4$ \\
Warmup steps          & $50$ & $50$ \\
Precision             & \texttt{bf16} & \texttt{bf16} \\
Held-out rollouts     & $200$ & $200$ \\
\bottomrule
\end{tabularx}
\caption{GMRel relational model configurations: oracle rollout generation (which produces the
supervision) and fitting of the relational influence model.}
\label{tab:gmrel-params}
\end{table}

At selection time each pool sequence's representation and individual influence are computed once and cached, and the greedy of Appendix~\ref{appendix:cluster_inference} uses the checkpoint's learned scalars ($\alpha \approx 0.996$, $\tau \approx 1.010$ for the pre-training model). Because the labels are reference-loss changes, more useful data has a \emph{more negative} score and the greedy selects the arg\,min. To vary relational strength we additionally form groups with $\tau$ overridden to $0.30$, roughly tripling the weight on the relational term; this is the comparison in Figure~\ref{fig:gmrel_taus_r50}.

\subsubsection{GREATS}
GREATS \citep{wang2024greats} augments the GradSim individual score with a penalty against the sequences already selected. Writing $g_z = \nabla_\theta \mathcal{L}(z,\theta)$ and $g_{\mathrm{ref}}$ for the reference gradient above, the individual term is $s_z = \langle g_z, g_{\mathrm{ref}}\rangle$ and the marginal gain of adding $z$ to a partially formed group $G$ is
\[
  \phi_z = s_z - \eta \Big\langle g_z, \sum_{z^* \in G} g_{z^*} \Big\rangle .
\]
Selection starts from $G = \emptyset$, takes $\arg\max_z s_z$ in the first round, and repeatedly adds $\arg\max_z \phi_z$, so a candidate that points in the same gradient direction as the current members is discounted even if its individual score is high. Because $\phi_z$ has to be recomputed against $G$ at every round, we cache each sequence's gradient once as a projection to $d = 8192$ dimensions (correlation $\ge 0.9999$ with the exact gradient inner products). We set step size $\eta$ to the training learning rate ($\eta = 0.003$). The gradients are the same per-sequence gradients GradSim selects on, so the two methods differ only in the selection rule.

\subsection{Group Formation}\label{appendix:cluster_inference}
In this section, we describe the group formation procedure used in Section~\ref{sec:collapse_setup}. For each data selection method considered in that section, we construct groups using a \emph{cluster-based inference} procedure \citep{yu2025data}. We adopt this procedure to reduce computational cost, since some data selection methods (e.g., GMRel \citep{yu2025data} and GREATS \citep{wang2024greats}) require iterative rounds of scoring over the entire dataset.

Specifically, for each synthetic dataset from Section~\ref{sec:data_influence_methods}, we embed every sequence with BGE-base-en-v1.5 \citep{xiao2024c} and partition the pool into $30$ clusters via $k$-means. We then sample $20$ clusters and use each data selection method to form one group of size $s = 5120$ from each sampled cluster, yielding $20$ groups per method.

Individual-level scorers (i.e., GradSim \citep{pruthi2020estimating}, FineWeb-Edu \citep{penedo2024fineweb}, and BM25 \citep{bm25}) form each group greedily by taking the top-scored samples within a cluster. GMRel \citep{yu2025data} and GREATS \citep{wang2024greats} instead form each group through a relational greedy process: starting from an empty group, they repeatedly add the candidate with the highest score after \emph{re-scoring against the partially formed group}, so that each addition accounts for its relation to the samples already selected, as described in Appendix~\ref{appendix:group_influence_methods}.

\subsection{Diversity Metrics}\label{appendix:diversity_metrics}

In this section, we describe the diversity metrics evaluated in Section~\ref{sec:diagnostics}. Each metric maps a group to a scalar, and in every case a larger value indicates a more diverse (less redundant) group. We then report the Spearman rank correlation between each metric and the oracle gap $\big|\,\Iproxy(G,\theta) - \mathcal{I}_{\mathrm{orcl}}(G,\theta)\,\big|$ over the groups in each setting.

\paragraph{Gradient diversity.} This is the gradient-based metric. For each sample $z$ we take the per-sample gradient $g_z = \nabla_\theta \mathcal{L}(z, \theta)$ at the same base checkpoint used for oracle scoring (Table~\ref{tab:oracle-scoring-params}), and JL-project it to $d = 8192$ dimensions with the TRAK Rademacher projector \citep{park2023trak}, as in Appendix~\ref{appendix:group_influence_methods}. Writing $u_z = p_z / \norm{p_z}$ for the unit-normalized projected gradients, the metric is one minus the mean pairwise cosine similarity,
\[
  \mathrm{GradDiv}(G) \;=\; 1 - \frac{1}{|G|(|G|-1)} \sum_{z \neq z'} \langle u_z, u_{z'} \rangle .
\]
We use the cosine form so that the metric depends only on gradient \emph{directions} and is not dominated by a few high-magnitude samples.

\paragraph{Vendi score.} The Vendi score \citep{friedman2022vendi} is an embedding-based metric that measures the \emph{effective number of distinct samples} in a group. We embed each sample with BGE-base-en-v1.5 \citep{xiao2024c} (averaging over a sample's sequences) and $\ell_2$-normalize, giving $X \in \mathbb{R}^{|G| \times 768}$. Letting $\lambda_1, \dots, \lambda_{|G|}$ be the eigenvalues of the normalized cosine-similarity kernel $K = \frac{1}{|G|} X X^\top$, the score is the exponential of the Shannon entropy of the eigenspectrum,
\[
  \mathrm{Vendi}(G) \;=\; \exp\!\Big(-\sum_{i} \lambda_i \log \lambda_i\Big).
\]
It ranges from $1$ (all samples identical) to $|G|$ (mutually orthogonal embeddings), and unlike a mean pairwise similarity it accounts for higher-order redundancy, such as several tight clusters of near-duplicates.

\paragraph{Compression ratio.} This is the information-theoretic metric, and captures redundancy \emph{across} samples that a compressor can exploit \citep{shaib-etal-2025-standardizing}. Writing $b_z$ for the raw byte string of sample $z$ and $\mathrm{zlib}(\cdot)$ for the compressed length at level $6$, we take the ratio of the compressed length of the concatenated group to the sum of the individually compressed lengths,
\[
  \mathrm{CompRatio}(G) \;=\; \frac{\big|\mathrm{zlib}(b_{z_1} \Vert \cdots \Vert b_{z_{|G|}})\big|}{\sum_{z \in G} \big|\mathrm{zlib}(b_z)\big|} .
\]
A value near $1$ means concatenating the samples yields no compression gain, i.e.\ they share little material; values well below $1$ indicate that samples are largely restatements of one another.

\paragraph{Distinct $n$-grams.} These are the token-level metrics. We pool the token sequences of all samples in the group (stripping padding), enumerate all token $n$-grams, and report the fraction that are unique,
\[
  \mathrm{Distinct}\text{-}n(G) \;=\; \frac{\big|\{\text{unique } n\text{-grams in } G\}\big|}{\big|\{\text{all } n\text{-grams in } G\}\big|} .
\]
We report Distinct-$2$ and Distinct-$3$. Because $n$-grams are pooled across the whole group rather than computed per sample, the metric reflects repetition between samples and not just within them.

\subsection{Budget Analysis}\label{appendix:budget_analysis}

This section details the memory computation from Section~\ref{sec:diagnostics}. We consider a pool of data groups $\mathcal{G} = \{G_1, \dots, G_H\}$, where each group contains $|G|$ sequences, and an allocation that audits $B$ of the $H$ groups with the group oracle. We report \emph{cumulative} GPU memory, i.e., memory integrated over the run (GiB$\cdot$hours). A run has two phases, group-oracle calls and per-sample gradient passes, so its cumulative memory is
\begin{equation}
  \mathrm{Mem} \;=\; \underbrace{\mathrm{mem}_{\mathrm{orcl}} \times T_{\mathrm{orcl}}}_{\text{oracle phase}} \;+\; \underbrace{\mathrm{mem}_{\mathrm{grad}} \times T_{\mathrm{grad}}}_{\text{gradient phase}} ,
  \label{eqn:cumulative_memory}
\end{equation}
where $\mathrm{mem}$ is the GPU memory resident during a phase and $T$ is the time spent in it. Since both phases are compute-bound on the same GPU, we substitute FLOPs for time. We detail the FLOPs analysis in Section~\ref{appendix:budget_flops} and the cumulative memory analysis in Section~\ref{appendix:budget_memory}.

\subsubsection{FLOPs Analysis}\label{appendix:budget_flops}
We first define the following variables for FLOPs analysis:
\begin{itemize}[topsep=2pt, itemsep=2pt, parsep=0pt]
\item $F_{\mathrm{fwd}}$: FLOPs for one forward pass of the probe model over one sequence. We report all FLOPs in units of $F_{\mathrm{fwd}}$.
\item $F_{\mathrm{bwd}} \approx 2F_{\mathrm{fwd}}$: FLOPs for one backward pass.
\item $R$: number of sequences in the reference set $D_{\mathrm{ref}}$.
\end{itemize}
Section~\ref{sec:diagnostics} applies two operations to a group $G$. A \emph{group-oracle call} (Appendix~\ref{appendix:oracle_scoring_parameters}) performs one AdamW step over $G$ and then evaluates the updated model on $D_{\mathrm{ref}}$. A \emph{gradient pass} computes the per-sample gradient $g_z$ of every $z \in G$, each requiring one forward and backward pass. Their FLOPs are
\begin{align*}
  F_{\mathrm{orcl}} \;&=\; \underbrace{|G|\,(F_{\mathrm{fwd}} + F_{\mathrm{bwd}})}_{\text{AdamW step over } G} \;+\; \underbrace{R\, F_{\mathrm{fwd}}}_{\text{reference evaluation}} \;=\; 3|G| + R, \\
  F_{\mathrm{grad}} \;&=\; \underbrace{|G|\,(F_{\mathrm{fwd}} + F_{\mathrm{bwd}})}_{\text{per-sample gradients over } G} \;=\; 3|G| .
\end{align*}
A gradient pass over a fraction $q$ of the group, as in GradDiv\textsubscript{$q$}, costs $q\,F_{\mathrm{grad}}$.\footnote{GradDiv\textsubscript{$q$} samples $\lceil q|G| \rceil$ sequences; the rounding does not affect any reported value.} 

\subsubsection{Cumulative Memory Analysis}\label{appendix:budget_memory}
Next, we define the following variables for memory analysis:
\begin{itemize}[topsep=2pt, itemsep=2pt, parsep=0pt]
\item $W$: GPU memory of one copy of the model parameters.
\item $A$: GPU memory of the activations of one micro-batch.
\end{itemize}
A group-oracle call holds the weights, the gradient, and the two AdamW moments, plus activations. A gradient pass holds only the weights and the gradient, plus activations; no optimizer state is instantiated. The resident memory of each phase is therefore
\begin{align*}
  \mathrm{mem}_{\mathrm{orcl}} \;&=\; \underbrace{W}_{\text{weights}} + \underbrace{W}_{\text{gradient}} + \underbrace{2W}_{\text{AdamW moments}} + \underbrace{A}_{\text{activations}} \;=\; 4W + A, \\
  \mathrm{mem}_{\mathrm{grad}} \;&=\; \underbrace{W}_{\text{weights}} + \underbrace{W}_{\text{gradient}} + \underbrace{A}_{\text{activations}} \;=\; 2W + A,
\end{align*}
and we write $\kappa = \mathrm{mem}_{\mathrm{grad}} / \mathrm{mem}_{\mathrm{orcl}} = (2W + A)/(4W + A)$ for the memory discount of a gradient pass relative to an oracle call.

Every allocation in Section~\ref{sec:diagnostics} audits $B$ groups with the group oracle and runs a gradient pass over a fraction $q$ of each of the $H$ groups, so all of them share one cost expression:
\begin{equation}
  \mathrm{Mem}(B, q) \;=\; \underbrace{(4W + A)\, B\, F_{\mathrm{orcl}}}_{\text{oracle calls}} \;+\; \underbrace{(2W + A)\, H\, q\, F_{\mathrm{grad}}}_{\text{gradient passes}} .
  \label{eqn:memtime_cost}
\end{equation}
The allocations differ only in $q$:
\begin{itemize}[topsep=2pt, itemsep=2pt, parsep=0pt]
\item \textbf{Group oracle.} Every group receives an oracle call and no gradients are computed.
\item \textbf{Random.} No diagnostic is computed, and the allocation is charged only for its oracle calls.
\item \textbf{GradDiv\textsubscript{full}.} All per-sample gradients of every group are computed to rank the groups; the top $B$ are audited and the rest are scored with GradSim, which reuses the same gradients.
\item \textbf{GradDiv\textsubscript{$q$}.} As above, with a uniform subsample of a fraction $q$ of each group.
\end{itemize}

\textbf{Ratio to the group oracle.} Figure~\ref{fig:memtime_vs_quality_f02_repro_n5120_recall10} plots each allocation's cumulative memory relative to the group oracle, $\mathrm{Mem}_{\mathrm{orcl}} = \mathrm{Mem}(H, 0)$.

\begin{proposition}[Cumulative memory ratio]
\label{prop:memtime_ratio}
For an allocation that audits $B$ groups and runs a gradient pass over a fraction $q$ of each group,
\begin{equation}
  \frac{\mathrm{Mem}(B, q)}{\mathrm{Mem}_{\mathrm{orcl}}}
  \;=\; \frac{B}{H} \;+\; \kappa\, q\, \frac{3|G|}{3|G| + R}
  \;\approx\; \frac{B}{H} \;+\; \kappa\, q ,
  \label{eqn:memtime_ratio}
\end{equation}
where the approximation holds when $R \ll 3|G|$.
\end{proposition}

\begin{proof}
Substituting $F_{\mathrm{orcl}} = 3|G| + R$ and $F_{\mathrm{grad}} = 3|G|$ into \cref{eqn:memtime_cost}, and dividing by $\mathrm{Mem}_{\mathrm{orcl}} = (4W + A)\, H\, (3|G| + R)$,
\begin{align*}
  \frac{\mathrm{Mem}(B, q)}{\mathrm{Mem}_{\mathrm{orcl}}}
  \;&=\; \frac{(4W + A)\, B\, (3|G| + R) + (2W + A)\, H\, q\, 3|G|}{(4W + A)\, H\, (3|G| + R)} \\[2pt]
  \;&=\; \frac{B}{H} \;+\; \frac{2W + A}{4W + A}\cdot q\, \frac{3|G|}{3|G| + R} ,
\end{align*}
and the middle factor is $\kappa$ by definition. Only the ratio of the two memory levels enters, so the result does not depend on the absolute model size or GPU. When $R \ll 3|G|$, the last fraction is $\approx 1$.
\end{proof}
Table~\ref{tab:budget_memory_model} summarises the resulting ratios.

\begin{table}[ht]
\centering
\footnotesize
\setlength{\tabcolsep}{6pt}
\begin{tabular}{lllcc}
\toprule
Allocation & $B$ & $q$ & Ratio (Eq.~\ref{eqn:memtime_ratio}) & Calculated Ratio \\
\midrule
Group oracle & $H$ & $0$ & $1$ & $1.000$ \\
GradDiv\textsubscript{full} & $B$ & $1$ & $\approx B/H + \kappa$ & $0.790$ \\
GradDiv\textsubscript{$q$} & $B$ & $0.5$ & \multirow{4}{*}{$\approx B/H + \kappa q$} & $0.495$ \\
GradDiv\textsubscript{$q$} & $B$ & $0.1$ & & $0.259$ \\
GradDiv\textsubscript{$q$} & $B$ & $0.05$ & & $0.229$ \\
GradDiv\textsubscript{$q$} & $B$ & $0.01$ & & $0.206$ \\
Random & $B$ & $0$ & $B/H$ & $0.200$ \\
\bottomrule
\end{tabular}
\caption{Cumulative GPU memory of each allocation relative to the full group oracle. The last column is the exact ratio of \cref{eqn:memtime_ratio} on the pre-training pool of Figure~\ref{fig:memtime_vs_quality_f02_repro_n5120_recall10} ($H = 1000$, $|G| = 5120$, $R = 128$, $B/H = 0.2$, $\kappa = 0.593$).}
\label{tab:budget_memory_model}
\end{table}

\newpage
\section{Proofs: Invidividual Proxy Oracle vs Group Oracle}
\label{appendix:oracle_scoring}

In this section, we prove \cref{prop:single_group_gap} and analyze the gap between the two ways of
assigning an oracle score to a group: averaging the per-example oracles, and scoring the group
jointly.

\subsection{Setup}

We adopt the notation of \cref{sec:background} and use the following quantities throughout:
\begin{center}
\begin{tabular}{ll}
$g_z = \nabla_\theta \ell(z;\theta)$ & gradient of the training loss on example $z$ at $\theta$ \\[2pt]
$g_G = \frac{1}{|G|}\sum_{z\in G} g_z$ & mean gradient of the group \\[2pt]
$H = \nabla^2 L(\theta)$ & Hessian of the reference loss at $\theta$ \\[2pt]
$\Sigma_G = \frac{1}{|G|}\sum_{z\in G} (g_z-g_G)(g_z-g_G)^\top$ & within-group gradient covariance \\[2pt]
$\sigma_G^2 = \frac{1}{|G|}\sum_{z\in G} \norm{g_z-g_G}^2 = \Tr(\Sigma_G)$ & within-group gradient diversity \\[2pt]
$\beta$ & smoothness constant of $L$ ($\norm{H}\le\beta$)
\end{tabular}
\end{center}

Let $\theta$ denote the model parameters,
$D_{\mathrm{ref}}$ the reference set, and $G=\{z_1,\dots,z_{|G|}\}$ a group of training examples. Recall from \cref{def:oracle_influence} in \cref{sec:background} the oracle scores for a data sample $z$ and a data group $G$:
\begin{align*}
&\text{(individual)} \quad \mathcal{I}_{\mathrm{orcl}}(z, \theta)
= \mathcal{L}\!\left(D_{\mathrm{ref}} \mid \theta_{D_{\mathrm{train}} \cup \{z\}}\right) - \mathcal{L}\!\left(D_{\mathrm{ref}} \mid \theta_{D_{\mathrm{train}}}\right) \\
&\text{(group)} \quad 
\mathcal{I}_{\mathrm{orcl}}(G, \theta)
= \mathcal{L}\!\left(D_{\mathrm{ref}} \mid \theta_{D_{\mathrm{train}} \cup G}\right) - \mathcal{L}\!\left(D_{\mathrm{ref}} \mid \theta_{D_{\mathrm{train}}}\right)
\end{align*}

which represents the change in model performance on reference set $D_{\mathrm{ref}}$ when a training sample $z$ or group $G$ is included versus excluded from the training set. 

Next, we introduce a helper function to capture the loss-change via parameter displacement.
\begin{defframe}
\textit{\paragraph{Loss-Change Function:} Let $d$ denote the number of model parameters, so that $\theta, g_z, g_G\in\mathbb{R}^d$. For a parameter displacement $v\in\mathbb{R}^d$, we define the loss-change function as:
\begin{equation}
\Phi(v) \;\overset{\text{def}}{=}\; L(\theta-v) - L(\theta),
\qquad
\Phi:\mathbb{R}^d\to\mathbb{R},
\label{eqn:appendix_phi}
\end{equation}}
\end{defframe}

%START
\begin{proofframe}
We rewrite the definitions in \cref{def:oracle_influence} above by abbreviating the reference loss as $L(\theta) \;\overset{\text{def}}{=}\; \mathcal{L}\!\left(D_{\mathrm{ref}}\mid\theta\right)$, and including the gradient steps $\theta\to\theta-\eta g_z$ and  $\theta\to\theta-\eta g_G$:
\begin{align*}
\mathcal{I}_{\mathrm{orcl}}(z,\theta) 
&\;=\; L(\theta-\eta g_z) - L(\theta)  \\
\mathcal{I}_{\mathrm{orcl}}(G,\theta) 
&\;=\; L(\theta-\eta g_G) - L(\theta)
\end{align*}

Furthermore, the individual proxy oracle of \cref{asm:additive} in \cref{sec:background} becomes:
\begin{equation*}
\Iproxy(G,\theta) = \frac{1}{|G|}\sum_{z\in G} L(\theta-\eta g_z) - L(\theta).
\label{eqn:appendix_oracle_proxy}
\end{equation*}
%END
With the loss-change function, for $\Iproxy$ and $\mathcal{I}_{\mathrm{orcl}}$ the relevant displacements are the gradients $\eta g_z\in\mathbb{R}^d$ and $\eta g_G\in\mathbb{R}^d$, respectively. Thus, $\Iproxy$ and $\mathcal{I}_{\mathrm{orcl}}$ are written as:
\begin{empheq}[box=\fbox]{align}
\Iproxy(G,\theta) &= \frac{1}{|G|}\sum_{z\in G} \Phi(\eta g_z) \label{eqn:orc_single_loss_change} \\
\mathcal{I}_{\mathrm{orcl}}(G,\theta) &= \Phi(\eta g_G). \label{eqn:orc_group_loss_change}
\end{empheq}
Thus $\Iproxy(G,\theta)$ is the average of $\Phi$ over the points
$\{\eta g_z\}$, while $\mathcal{I}_{\mathrm{orcl}}(G,\theta)$ is $\Phi$ evaluated at the mean group gradient $g_G$.
\end{proofframe}

\subsection{Leading-order characterization of the gap}

In this section, we characterize the difference between $\Iproxy$ and $\mathcal{I}_{\mathrm{orcl}}$. We first identify the leading-order behavior of the gap via a Taylor expansion of $\Phi$ about
$v=0$, then promote it to an exact bound in \cref{prop:single_group_gap}.

\begin{proposition}[Leading-order gap]
\label{prop:leading_order_gap}
Let $\eta$ be the gradient-step size, $H=\nabla^2 L(\theta)$ the Hessian of the reference loss at $\theta$, and $\Sigma_G$ the within-group gradient covariance. Then the leading-order gap between $\mathcal{I}_{\mathrm{orcl}}(G,\theta)$ and $\Iproxy(G,\theta)$ is
\begin{equation}
\Iproxy(G,\theta) - \mathcal{I}_{\mathrm{orcl}}(G,\theta)
\;=\; \frac{\eta^2}{2}\,\Tr\!\big(H\,\Sigma_G\big) \;+\; O(\eta^3).
\label{eqn:appendix_leading_gap}
\end{equation}
In particular, the individual proxy oracle $\Iproxy$ and the true group oracle $\mathcal{I}_{\mathrm{orcl}}$
agree to first order in $\eta$, and their leading discrepancy is a second-order effect that grows
with both the loss curvature $H$ and the within-group gradient diversity $\Sigma_G$.
\end{proposition}

\begin{proof} We begin by expanding $\mathcal{I}_{\mathrm{orcl}}(G,\theta)$ and $\Iproxy(G,\theta)$ using Lemma~\ref{lem:phi_expansion}.

\emph{Step 1. Expanding $\mathcal{I}_{\mathrm{orcl}}(G,\theta)$}. Using Lemma~\ref{lem:phi_expansion} we can write \cref{eqn:orc_group_loss_change} as:
\begin{align}
\mathcal{I}_{\mathrm{orcl}}(G,\theta) 
&\;=\; \Phi(\eta g_G) \nonumber \\
&\;=\; -\nabla L(\theta)^\top(\eta g_G) \;+\; \tfrac12(\eta g_G)^\top H(\eta g_G) \;+\; O(\norm{\eta g_G}^3) \quad \text{(Lemma~\ref{lem:phi_expansion})} \nonumber \\
&\;=\; -\eta\,\nabla L(\theta)^\top g_G \;+\; \frac{\eta^2}{2}\,g_G^\top H g_G \;+\; O(\eta^3). \label{eqn:expand_oracle_group}
\end{align}

\emph{Step 2. Expanding $\Iproxy(G,\theta)$.} Similarly, using Lemma~\ref{lem:phi_expansion} we can write \cref{eqn:orc_single_loss_change} as:

\begin{align}
\Iproxy(G,\theta) 
&\;=\; \frac{1}{|G|}\sum_{z\in G}\Phi(\eta g_z) \nonumber \\
&\;=\; \frac{1}{|G|}\sum_{z\in G} \left( -\eta\,\nabla L(\theta)^\top g_z \;+\; \frac{\eta^2}{2}\, g_z^\top H g_z \;+\; O(\eta^3) \right) \quad \text{(Lemma~\ref{lem:phi_expansion})} \nonumber \\
&\;=\; -\eta\,\nabla L(\theta)^\top g_G \;+\; \frac{\eta^2}{2}\cdot\frac{1}{|G|}\sum_{z\in G} g_z^\top H g_z \;+\; O(\eta^3). \label{eqn:expand_oracle_single}
\end{align}

To obtain the gap between $\Iproxy(G,\theta)$ and $\mathcal{I}_{\mathrm{orcl}}(G,\theta)$, we
subtract their expanded forms.

\emph{Step 3. Gap between $\Iproxy(G,\theta)$ and $\mathcal{I}_{\mathrm{orcl}}(G,\theta)$.} Subtracting
the group-oracle expansion  (Equation~\ref{eqn:expand_oracle_group}) from the individual proxy-oracle expansion (Equation~\ref{eqn:expand_oracle_single}) , we have:
\begin{align}
\Iproxy(G,\theta) - \mathcal{I}_{\mathrm{orcl}}(G,\theta)
&\;=\; \left( -\eta\,\nabla L(\theta)^\top g_G \;+\; \frac{\eta^2}{2}\cdot\frac{1}{|G|}\sum_{z\in G} g_z^\top H g_z \;+\; O(\eta^3) \right) \nonumber \\
&\qquad\;-\; \left( -\eta\,\nabla L(\theta)^\top g_G \;+\; \frac{\eta^2}{2}\,g_G^\top H g_G \;+\; O(\eta^3) \right) \nonumber \\
&\;=\; \frac{\eta^2}{2}\cdot\frac{1}{|G|}\sum_{z\in G} g_z^\top H g_z \;-\; \frac{\eta^2}{2}\,g_G^\top H g_G \;+\; O(\eta^3) \nonumber \\
&\;=\; \frac{\eta^2}{2}\left[\frac{1}{|G|}\sum_{z\in G} g_z^\top H g_z \;-\; g_G^\top H g_G\right] + O(\eta^3),
\label{eqn:appendix_bracket}
\end{align}

Note that the linear terms $-\eta\,\nabla L(\theta)^\top g_G$ are identical in both expansions (i.e., the two quantities agree to first order) and
therefore cancel.

\emph{Step 4. Within-group gradient variance.} Next, we simplify the bracketed term in \cref{eqn:appendix_bracket}. We begin by writing each gradient as the group mean plus a deviation,
\begin{equation}
g_z \;=\; g_G + \delta_z, \qquad \delta_z \overset{\text{def}}{=} g_z - g_G,
\label{eqn:appendix_deviation}
\end{equation}
where, by the definition $g_G = \frac{1}{|G|}\sum_{z\in G} g_z$, the deviations necessarily sum to zero: $\frac{1}{|G|}\sum_{z\in G}\delta_z = g_G - g_G = 0$, a fact we use below.

Next, we expand $g_z^\top H g_z$ using the symmetry of $H$ (note: the Hessian
$H=\nabla^2 L(\theta)$ is symmetric because $L$ is twice continuously differentiable),
\begin{align}
g_z^\top H g_z
&\;=\; (g_G+\delta_z)^\top H(g_G+\delta_z) \nonumber \\
&\;=\; g_G^\top H g_G \;+\; g_G^\top H\delta_z \;+\; \delta_z^\top H g_G \;+\; \delta_z^\top H\delta_z \nonumber \\
&\;=\; g_G^\top H g_G \;+\; 2\,\delta_z^\top H g_G \;+\; \delta_z^\top H\delta_z .
\end{align}
Averaging over $z\in G$, we obtain:
\begin{align}
\frac{1}{|G|}\sum_{z\in G} g_z^\top H g_z
&\;=\; g_G^\top H g_G \;+\; 2\Big(\tfrac{1}{|G|}\sum_{z\in G}\delta_z\Big)^{\!\top}\! H g_G \;+\; \frac{1}{|G|}\sum_{z\in G}\delta_z^\top H\delta_z \nonumber \\
&\;=\; g_G^\top H g_G \;+\; \frac{1}{|G|}\sum_{z\in G}\delta_z^\top H\delta_z \quad \left(\text{since $\frac{1}{|G|}\sum_{z\in G}\delta_z$ = 0}\right).
\end{align}
%Substituting this into the bracket of \cref{eqn:appendix_bracket} cancels $g_G^\top H g_G$ and isolates the average of $\delta_z^\top H\delta_z$. 
Applying the cyclic trace identity
$a^\top H a = \Tr(H\,a a^\top)$ with $a=\delta_z$, recasts the above equation as a trace:
\begin{align}
\frac{1}{|G|}\sum_{z\in G} g_z^\top H g_z \;-\; g_G^\top H g_G
&\;=\; \frac{1}{|G|}\sum_{z\in G}\delta_z^\top H\delta_z \nonumber \\
&\;=\; \frac{1}{|G|}\sum_{z\in G}\Tr\!\big(H\,\delta_z\delta_z^\top\big) \nonumber \\
&\;=\; \Tr\!\Big(H\cdot\tfrac{1}{|G|}\sum_{z\in G}\delta_z\delta_z^\top\Big) \nonumber \\
&\;=\; \Tr\!\big(H\,\Sigma_G\big),
\end{align}
where $\Sigma_G \overset{\text{def}}{=} \tfrac{1}{|G|}\sum_{z\in G}\delta_z\delta_z^\top = \frac{1}{|G|}\sum_{z\in G}(g_z-g_G)(g_z-g_G)^\top$ is the
within-group gradient covariance. Substituting this back into \cref{eqn:appendix_bracket},
\begin{equation}
\Iproxy(G,\theta) - \mathcal{I}_{\mathrm{orcl}}(G,\theta)
\;=\; \frac{\eta^2}{2}\,\Tr\!\big(H\,\Sigma_G\big) \;+\; O(\eta^3),
\end{equation}
which is \cref{eqn:appendix_leading_gap}. As the linear terms cancelled in Step~3, the leading
discrepancy is the $O(\eta^2)$ curvature term, so the two quantities agree to first order.
\end{proof}

The gap is therefore a second-order effect, governed by the product of the reference-loss
curvature $H$ and the within-group gradient covariance $\Sigma_G$: it vanishes when the group is
perfectly coherent ($\Sigma_G=0$) or the loss is locally flat ($H=0$).

\begin{lemma}[Quadratic expansion of $\Phi$]
\label{lem:phi_expansion}
Writing $H \;\overset{\text{def}}{=}\; \nabla^2 L(\theta)$ for the Hessian of the reference loss at
$\theta$, the function $\Phi$ in \cref{eqn:appendix_phi} admits the second-order expansion
\begin{equation}
\Phi(v) \;=\; -\,\nabla L(\theta)^\top v \;+\; \tfrac12\, v^\top H\, v \;+\; O(\norm{v}^3).
\label{eqn:appendix_phi_taylor}
\end{equation}
\end{lemma}

\begin{proof}
Recall from \cref{eqn:appendix_phi} that the loss-change function is
\begin{equation*}
\Phi(v) \;=\; L(\theta-v) - L(\theta).
\end{equation*}
The only term that depends on $v$ is the loss $L$ evaluated at the displaced parameters $\theta-v$,
so we expand $L$ via its general second-order Taylor expansion about the base point $\theta$.
Assuming $L$ is three-times continuously differentiable in a neighborhood of $\theta$, for any point
$x$ near $\theta$, the general second-order Taylor expansion is
\begin{equation}
L(x) \;=\; L(\theta) \;+\; \nabla L(\theta)^\top(x-\theta) \;+\; \tfrac12\,(x-\theta)^\top \nabla^2 L(\theta)\,(x-\theta) \;+\; O(\norm{x-\theta}^3).
\label{eqn:appendix_taylor_general}
\end{equation}
Setting $x=\theta-v$ we have
\begin{align*}
L(\theta-v)
&\;=\; L(\theta) \;+\; \nabla L(\theta)^\top(-v) \;+\; \tfrac12\,(-v)^\top \nabla^2 L(\theta)\,(-v) \;+\; O(\norm{v}^3) \\
&\;=\; L(\theta) \;-\; \nabla L(\theta)^\top v \;+\; \tfrac12\, v^\top H\, v \;+\; O(\norm{v}^3),
\end{align*}
where $H \overset{\text{def}}{=} \nabla^2 L(\theta)$ is the Hessian of the reference loss $L$ evaluated
at $\theta$.
Substituting this back into $\Phi(v) = L(\theta-v) - L(\theta)$ from \cref{eqn:appendix_phi} gives
\begin{equation}
\Phi(v) \;=\; -\,\nabla L(\theta)^\top v \;+\; \tfrac12\, v^\top H\, v \;+\; O(\norm{v}^3),
\end{equation}
which is \cref{eqn:appendix_phi_taylor}.
\end{proof}
\subsection{Proof of the exact bound}

The leading-order statement of \cref{prop:leading_order_gap} becomes the exact inequality of
\cref{prop:single_group_gap} under $\beta$-smoothness. 

%The argument rests on an exact, second-order form of the Jensen gap of the loss-change function $\Phi$ over the group, which we isolate as a lemma.

\begin{stmtframe}
\noindent\textbf{Proposition~\ref{prop:single_group_gap} (Individual-vs-Group Oracle gap, restated).}\itshape\quad
If the reference loss is $\beta$-smooth (i.e.\ $\norm{H}\le\beta$, where $H$ is its Hessian), then
\begin{equation*}
\big|\,\Iproxy(G, \theta) - \mathcal{I}_{\mathrm{orcl}}(G, \theta)\,\big| \;\le\; \frac{\beta\,\eta^2}{2}\,\sigma_G^2,
\end{equation*}
where $\sigma_G^2 = \frac{1}{|G|}\sum_{z\in G}\norm{g_z - g_G}^2$ is the within-group diversity of the per-example gradients.
\end{stmtframe}

\begin{proof}[Proof]
First, write $m\overset{\text{def}}{=}\eta g_G$ for the mean of the points. Substituting
the boxed identities \cref{eqn:orc_single_loss_change,eqn:orc_group_loss_change} for $\Iproxy(G,\theta)$
and $\mathcal{I}_{\mathrm{orcl}}(G,\theta)$, the Jensen gap of $\Phi$ over those points is:
\begin{align}
\Iproxy(G,\theta) - \mathcal{I}_{\mathrm{orcl}}(G,\theta)
&\;=\; \frac{1}{|G|}\sum_{z\in G}\Phi(\eta g_z) \;-\; \Phi(\eta g_G) \nonumber \\
&\;=\; \frac{1}{|G|}\sum_{z\in G}\Phi(\eta g_z) \;-\; \Phi(m) \label{eqn:appendix_jensen_gap} \\
&\;=\; \frac{1}{2|G|}\sum_{z\in G}(\eta g_z - m)^\top \nabla^2\Phi(\xi_z)\,(\eta g_z - m). \quad \text{(Lemma~\ref{lem:jensen_gap})} \label{eqn:appendix_quadratic_remainder}
\end{align}

Here, Lemma~\ref{lem:jensen_gap} applies directly to \cref{eqn:appendix_jensen_gap}, and $\xi_z = m + t_z(\eta g_z - m)$ with $t_z\in[0,1]$ are points on the segment between $m$ and
$\eta g_z$.

\emph{Step 1. Bound each quadratic form.} By \cref{lem:phi_expansion} the Hessian of $\Phi$ is
$\nabla^2\Phi(\cdot)=H(\cdot)$, and $\beta$-smoothness of $L$ means $\norm{H(\cdot)}\le\beta$ at every
point, including each $\xi_z$; thus $\norm{\nabla^2\Phi(\xi_z)}\le\beta$. For any symmetric matrix $A$
and vector $u$, the operator-norm inequality gives $|u^\top A u|\le\norm{A}\,\norm{u}^2$, so with
$u=\eta g_z - m$,
\begin{equation}
\big|(\eta g_z - m)^\top \nabla^2\Phi(\xi_z)\,(\eta g_z - m)\big|
\;\le\; \beta\,\norm{\eta g_z - m}^2.
\end{equation}
Applying the triangle inequality to the average in \cref{eqn:appendix_quadratic_remainder},
\begin{align}
\big|\Iproxy(G,\theta) - \mathcal{I}_{\mathrm{orcl}}(G,\theta)\big|
&\;\le\; \frac{1}{2|G|}\sum_{z\in G}\beta\,\norm{\eta g_z - m}^2 \nonumber \\
&\;=\; \frac{\beta}{2}\cdot\frac{1}{|G|}\sum_{z\in G}\norm{\eta g_z - \eta g_G}^2.
\end{align}

\emph{Step 2. Identify the variance.} Factoring out $\eta^2$ and using the definition
$\sigma_G^2 \overset{\text{def}}{=} \frac{1}{|G|}\sum_{z\in G}\norm{g_z-g_G}^2$ of the within-group
gradient diversity,
\begin{align}
\frac{1}{|G|}\sum_{z\in G}\norm{\eta g_z - \eta g_G}^2
&\;=\; \eta^2\cdot\frac{1}{|G|}\sum_{z\in G}\norm{g_z - g_G}^2 \nonumber \\
&\;=\; \eta^2\,\sigma_G^2.
\end{align}
Combining with Step~1,
\begin{equation}
\big|\Iproxy(G,\theta) - \mathcal{I}_{\mathrm{orcl}}(G,\theta)\big|
\;\le\; \frac{\beta\,\eta^2}{2}\,\sigma_G^2,
\end{equation}
which is \cref{eqn:single_group_gap}. The bound is exact in $\eta$; the agreement to first order
follows separately from \cref{prop:leading_order_gap}.
\end{proof}

\begin{lemma}[Exact second-order Jensen gap]
\label{lem:jensen_gap}
Let $\Phi:\mathbb{R}^d\to\mathbb{R}$ be the loss-change function \cref{eqn:appendix_phi}, and assume
it is twice continuously differentiable. For a group $G$ with per-example gradients $\{g_z\}_{z\in G}$,
mean gradient $g_G=\frac{1}{|G|}\sum_{z\in G}g_z$, and step $m\overset{\text{def}}{=}\eta g_G$, there
exist points $\xi_z = m + t_z(\eta g_z - m)$ with $t_z\in[0,1]$, each on the segment between $m$ and
$\eta g_z$, such that
\begin{equation}
\frac{1}{|G|}\sum_{z\in G}\Phi(\eta g_z) \;-\; \Phi(m)
\;=\; \frac{1}{2|G|}\sum_{z\in G}(\eta g_z - m)^\top \nabla^2\Phi(\xi_z)\,(\eta g_z - m).
\label{eqn:appendix_jensen_lemma}
\end{equation}
\end{lemma}

\begin{proof}
We recall Taylor's theorem with the Lagrange form of the remainder. For a twice continuously differentiable $f:\mathbb{R}^d\to\mathbb{R}$ and any two points
$a,b\in\mathbb{R}^d$, there exists a point $\xi$ on the line segment joining $a$ and $b$, that is,
$\xi=a+t(b-a)$ for some $t\in[0,1]$, such that
\begin{equation}
f(b) \;=\; f(a) \;+\; \nabla f(a)^\top(b-a) \;+\; \tfrac12\,(b-a)^\top \nabla^2 f(\xi)\,(b-a).
\label{eqn:appendix_taylor_lagrange}
\end{equation}
We apply \cref{eqn:appendix_taylor_lagrange} to $f=\Phi$, with $a=m$ and $b=\eta g_z$ for each
$z\in G$. This produces, for each $z$, a point $\xi_z = m + t_z(\eta g_z - m)$, $t_z\in[0,1]$, on the
segment between $m$ and $\eta g_z$ such that
\begin{equation}
\Phi(\eta g_z)
\;=\; \Phi(m) \;+\; \nabla\Phi(m)^\top(\eta g_z - m) \;+\; \tfrac12\,(\eta g_z - m)^\top \nabla^2\Phi(\xi_z)\,(\eta g_z - m).
\label{eqn:appendix_exact_taylor}
\end{equation}
No higher-order terms are dropped: this is an equality, with the second-order term evaluated at
$\xi_z$. Averaging over $z\in G$, the linear term vanishes because, with $m=\eta g_G$, the deviations
sum to zero:
\begin{align*}
\frac{1}{|G|}\sum_{z\in G}\nabla\Phi(m)^\top(\eta g_z - m)
&\;=\; \nabla\Phi(m)^\top\Big(\tfrac{1}{|G|}\sum_{z\in G}\eta g_z - m\Big) \\
&\;=\; \nabla\Phi(m)^\top(\eta g_G - m) \\
&\;=\; 0.    
\end{align*}

Hence $\frac{1}{|G|}\sum_{z\in G}\Phi(\eta g_z) = \Phi(m) + \frac{1}{2|G|}\sum_{z\in G}(\eta g_z - m)^\top \nabla^2\Phi(\xi_z)(\eta g_z - m)$,
and subtracting $\Phi(m)$ gives \cref{eqn:appendix_jensen_lemma}.
\end{proof}

\newpage
\section{Additional Figures and Tables}\label{appendix:tables_figures}

This section contains supporting plots and figures.

\subsection{Synthetic Data Curation Full Results}\label{appendix:tables_figures_collapse}
This section shows the full results from the experiments in Section~\ref{sec:data_influence_methods}.

\begin{table}[h]
\centering
\footnotesize
\setlength{\tabcolsep}{4pt}
\caption{Per-task accuracy at the final step (200 steps) for the Repro-400M continuation, by selection method and rephrase factor $r$. Best average within each $r$ is in \textbf{bold}, and second best is \underline{underlined}.}
\label{tab:collapse_results}
\begin{tabular}{lccccccccc}
\toprule
Method & HellaSwag & ARC-C & PIQA & MMLU & CSQA & LAMBADA & SQuAD & CoQA & Avg \\
\midrule
\multicolumn{10}{l}{\textit{Rephrase factor $r=1$}} \\
\midrule
Random & 30.95 & 26.96 & 66.16 & 23.19 & 19.41 & 41.04 & 9.82 & 25.52 & 30.38 \\
GradSim & 30.81 & 26.28 & 65.34 & 22.98 & 21.21 & 36.37 & 3.65 & 24.69 & 28.92 \\
FineWeb-Edu & 30.89 & 25.51 & 64.53 & 22.87 & 19.16 & 38.09 & 17.79 & 23.60 & 30.31 \\
BM25 & 31.14 & 26.45 & 65.34 & 23.21 & 19.66 & 38.15 & 5.40 & 26.68 & 29.50 \\
GMRel ($\tau\!\approx\!1.01$) & 30.41 & 25.94 & 65.02 & 23.15 & 19.82 & 36.76 & 25.17 & 24.62 & \underline{31.36} \\
GREATS & 30.76 & 26.79 & 65.13 & 23.85 & 20.48 & 38.87 & 24.20 & 24.87 & \textbf{31.87} \\
\midrule
\multicolumn{10}{l}{\textit{Rephrase factor $r=10$}} \\
\midrule
Random & 31.09 & 26.45 & 65.13 & 24.10 & 20.88 & 35.49 & 6.23 & 27.24 & 29.58 \\
GradSim & 30.77 & 27.39 & 64.15 & 24.48 & 19.57 & 30.76 & 5.86 & 21.34 & 28.04 \\
FineWeb-Edu & 30.75 & 26.54 & 63.87 & 22.97 & 20.15 & 34.29 & 7.24 & 23.47 & 28.66 \\
BM25 & 30.41 & 26.30 & 63.35 & 22.81 & 19.88 & 33.95 & 7.13 & 23.37 &  28.40 \\
GMRel ($\tau\!\approx\!1.01$) & 30.77 & 26.45 & 64.20 & 24.44 & 19.82 & 32.66 & 20.61 & 24.21 & \textbf{30.40} \\
GREATS & 30.52 & 26.09 & 65.89 & 24.17 & 19.74 & 35.55 & 16.49 & 23.69 & \underline{30.27} \\
\midrule
\multicolumn{10}{l}{\textit{Rephrase factor $r=50$}} \\
\midrule
Random & 31.77 & 27.13 & 64.09 & 25.66 & 21.62 & 28.00 & 6.13 & 20.42 & \textbf{28.10} \\
GradSim & 31.12 & 26.19 & 62.24 & 24.28 & 23.42 & 25.33 & 7.24 & 16.09 & 26.99 \\
FineWeb-Edu & 30.22 & 25.85 & 62.84 & 26.03 & 20.97 & 25.58 & 6.47 & 16.36 & 26.79 \\
BM25 & 29.97 & 26.02 & 63.06 & 24.38 & 20.39 & 29.67 & 11.66 & 19.03 & \underline{28.02} \\
GMRel ($\tau\!\approx\!1.01$) & 31.01 & 26.88 & 63.44 & 24.24 & 20.23 & 28.53 & 3.92 & 23.09 & 27.67 \\
%GMRel ($\tau\!\approx\!0.3$) & 30.53 & 25.85 & 63.06 & 24.18 & 21.54 & 28.68 & 16.92 & 21.00 & \textbf{28.97} \\
GREATS & 30.50 & 26.28 & 62.95 & 25.13 & 21.13 & 29.54 & 6.50 & 21.20 & 27.90 \\
\midrule
Baseline & 30.82 & 25.85 & 64.36 & 23.94 & 19.82 & 37.30 & 7.47 & 27.61 & 29.65 \\
\bottomrule
\end{tabular}
\end{table}
\newpage

\subsection{Group versus Individual Proxy Oracle Gap Across Group Sizes}\label{appendix:tables_figures_divsweep}
This section contains the full sweep of the discrepancy between the group and individual proxy oracle scores across group sizes, summarized in Figure~\ref{fig:divergence_sweep_main} of the main body.
\begin{figure}[h]
\centering
\begin{subfigure}[t]{0.24\linewidth}
    \centering
    \includegraphics[width=\linewidth]{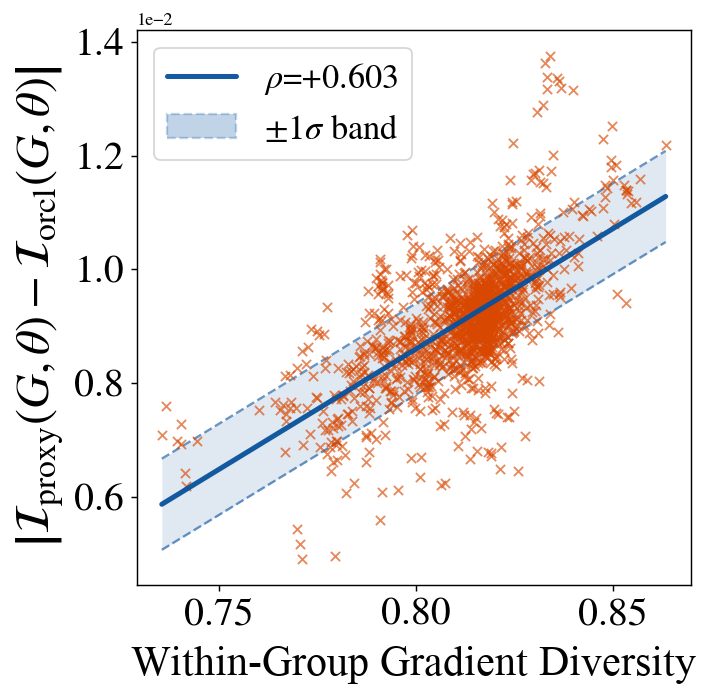}
    \caption{Pre-train (size 5120).}
    \label{fig:divsweep_repro_n5120}
\end{subfigure}
\begin{subfigure}[t]{0.24\linewidth}
    \centering
    \includegraphics[width=\linewidth]{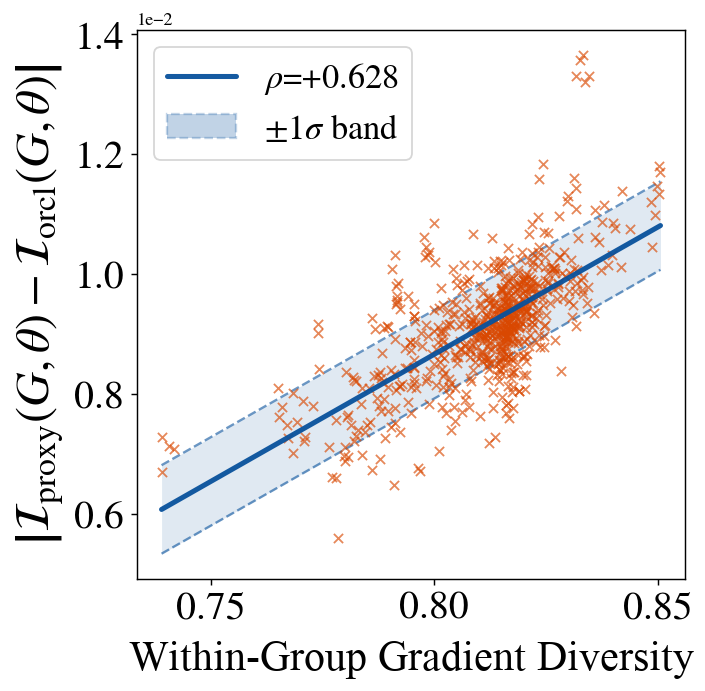}
    \caption{Pre-train (size 10240).}
    \label{fig:divsweep_repro_n10240}
\end{subfigure}
\begin{subfigure}[t]{0.24\linewidth}
    \centering
    \includegraphics[width=\linewidth]{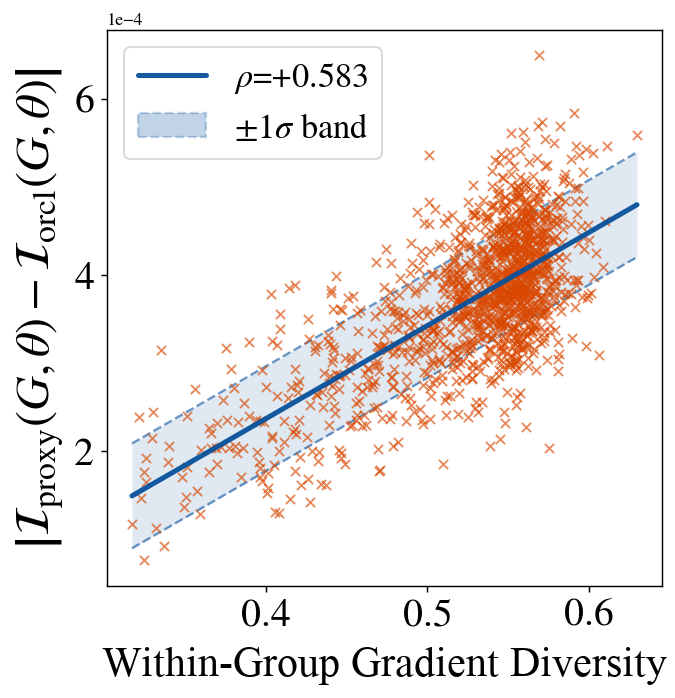}
    \caption{Post-train (size 128).}
    \label{fig:divsweep_goose_n128}
\end{subfigure}
\begin{subfigure}[t]{0.24\linewidth}
    \centering
    \includegraphics[width=\linewidth]{images/background/gradiv_cos_goose_n512_relabel.png}
    \caption{Post-train (size 512).}
    \label{fig:divsweep_goose_n512}
\end{subfigure}
\caption{Discrepancy between the group oracle $\mathcal{I}_{\mathrm{orcl}}(G, \theta)$ and individual proxy oracle $\Iproxy(G_j, \theta)$ scores across varying sizes of both the synthetic pretraining and the RLVR data groups.}
\label{fig:divergence_sweep}
\end{figure}

\begin{figure}[h]
\centering
\begin{subfigure}[t]{0.24\linewidth}
    \centering
    \includegraphics[width=\linewidth]{images/background/single_vs_group_repro_n5120_relabel.png}
    \caption{Pre-train (size 5120).}
    \label{fig:single_vs_group_repro_n5120}
\end{subfigure}
\begin{subfigure}[t]{0.24\linewidth}
    \centering
    \includegraphics[width=\linewidth]{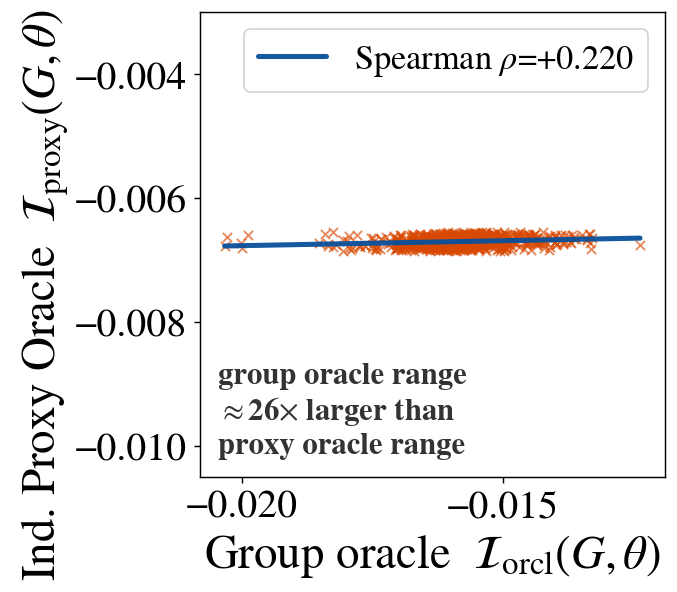}
    \caption{Pre-train (size 10240).}
    \label{fig:single_vs_group_repro_n10240}
\end{subfigure}
\begin{subfigure}[t]{0.24\linewidth}
    \centering
    \includegraphics[width=\linewidth]{images/background/single_vs_group_goose_n128_relabel.png}
    \caption{Post-train (size 128).}
    \label{fig:single_vs_group_goose_n128}
\end{subfigure}
\begin{subfigure}[t]{0.24\linewidth}
    \centering
    \includegraphics[width=\linewidth]{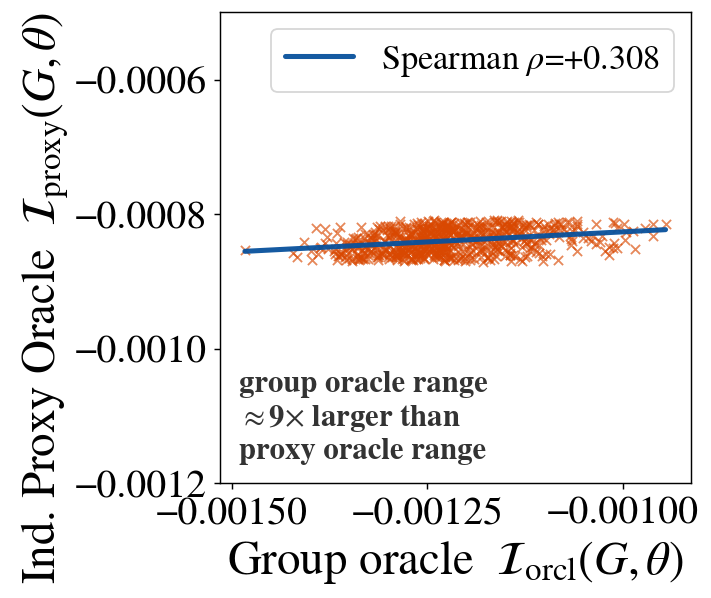}
    \caption{Post-train (size 512).}
    \label{fig:single_vs_group_goose_n512}
\end{subfigure}
\caption{Range of group oracle and individual proxy oracle scores across our constructed groups, for each group size in the pre-training and post-training settings.}
\label{fig:single_vs_group_divergence}
\end{figure}

\subsection{Cluster Topic Composition}\label{appendix:tables_figures_cluster_topics}
\begin{figure*}[h]
\centering
\begin{subfigure}[t]{0.45\textwidth}
    \centering
    \includegraphics[width=\linewidth]{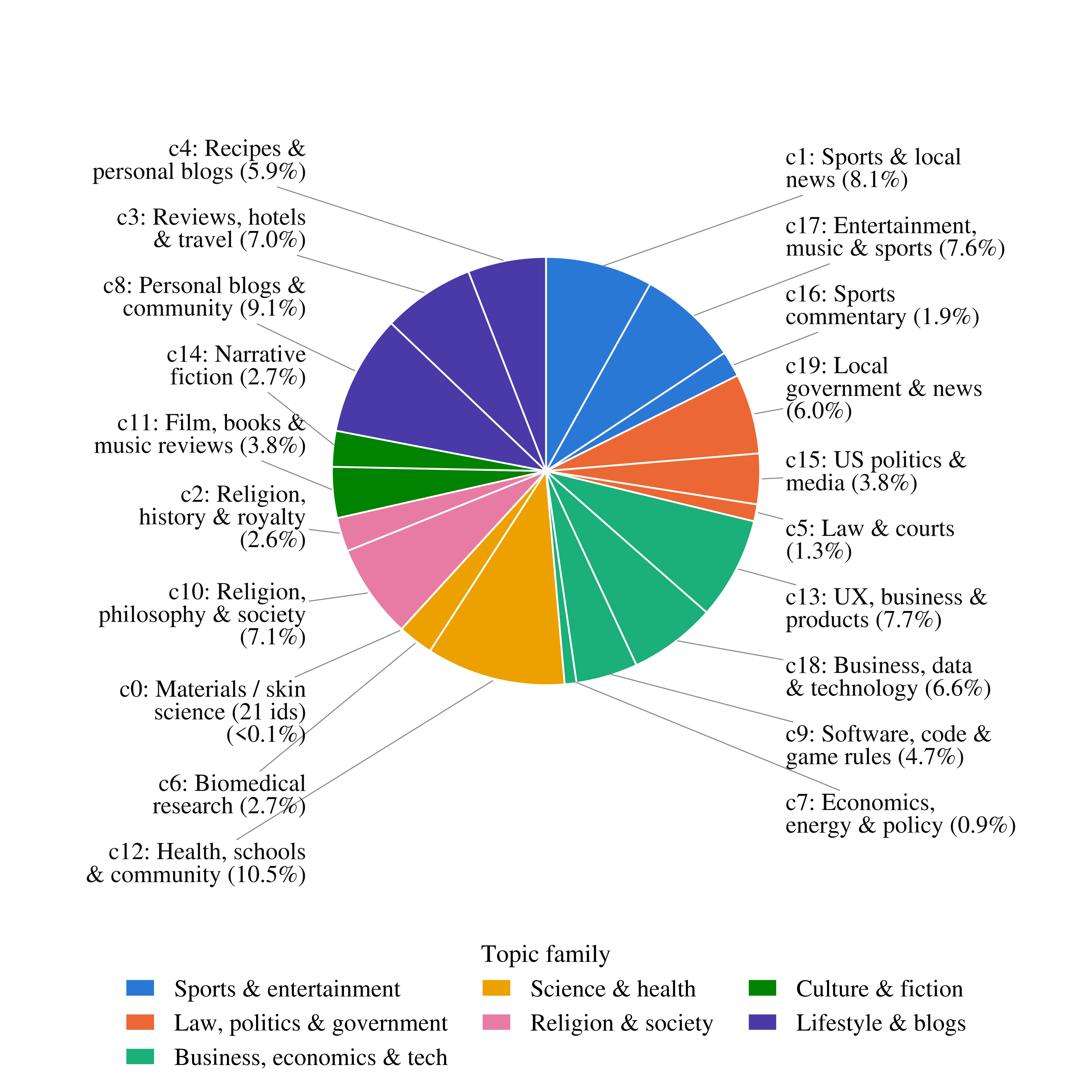}
    \caption{Pre-training (Repro).}
    \label{fig:cluster_topics_pie_repro}
\end{subfigure}
\begin{subfigure}[t]{0.45\textwidth}
    \centering
    \includegraphics[width=\linewidth]{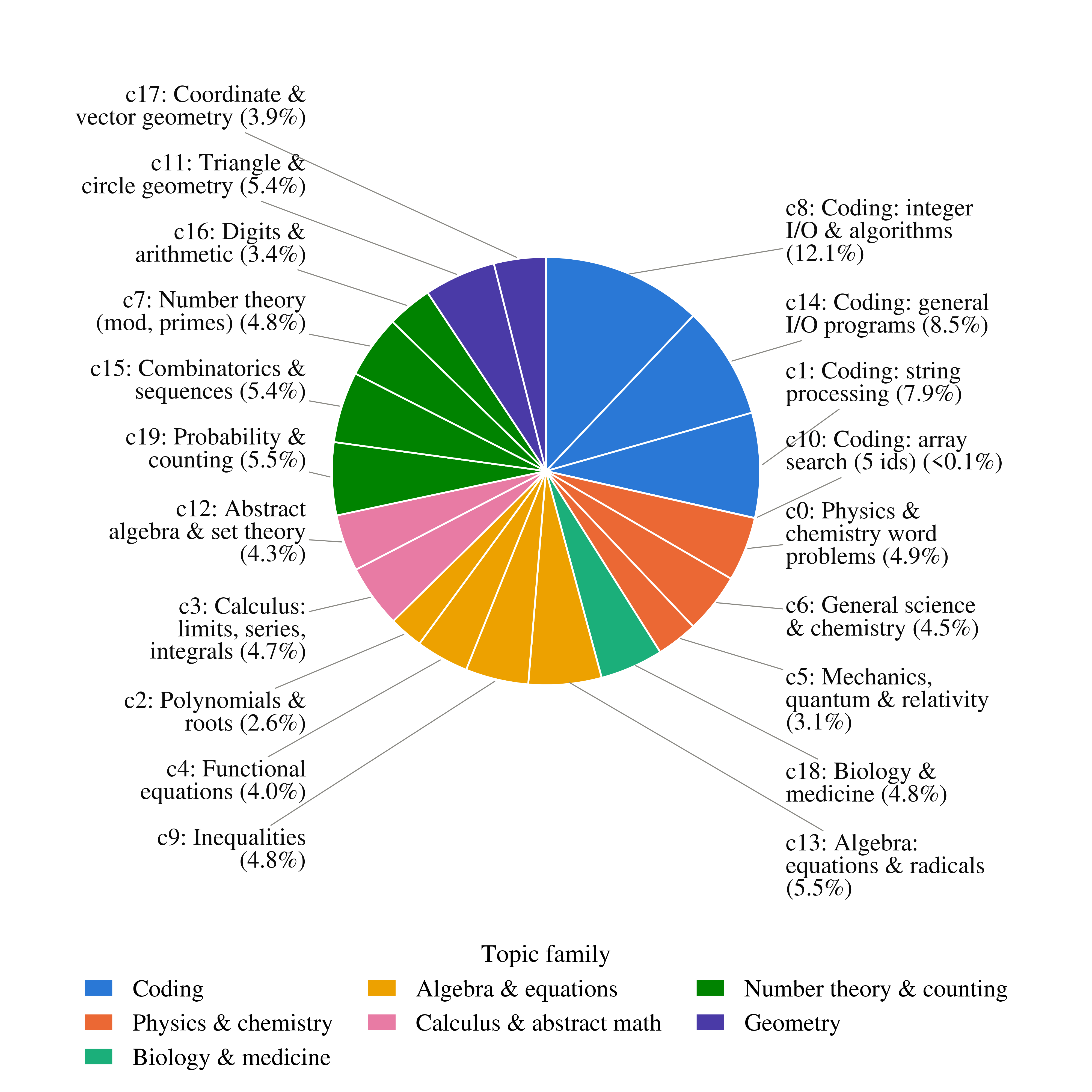}
    \caption{RLVR (GooseReason).}
    \label{fig:cluster_topics_pie_goose}
\end{subfigure}
\caption{Topic composition of the $k=20$ semantic clusters used for proxy-matched diversity sampling. Each slice is one cluster, labeled with its topic and share of the corpus, and colored by topic family.}
\label{fig:cluster_topics_pie}
\end{figure*}

\newpage

\subsection{Additional Diversity Plots}\label{appendix:tables_figures_diversity}
\begin{figure*}[h]
\centering
\begin{subfigure}[t]{0.22\textwidth}
    \centering
    \includegraphics[width=\linewidth]{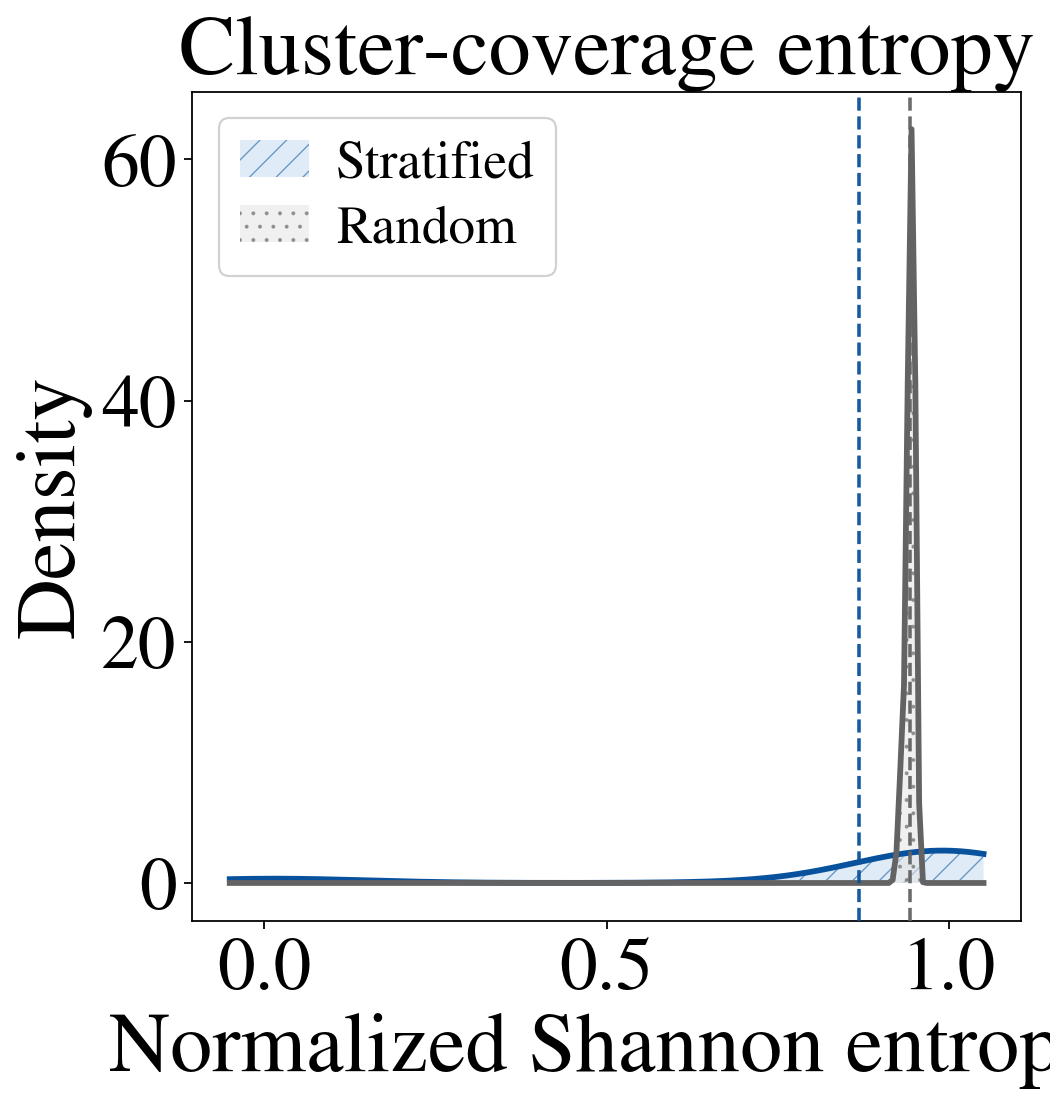}
    \caption{Cluster entropy.}
    \label{fig:divmetric_repro_cluster_entropy}
\end{subfigure}
\hfill
\begin{subfigure}[t]{0.26\textwidth}
    \centering
    \includegraphics[width=\linewidth]{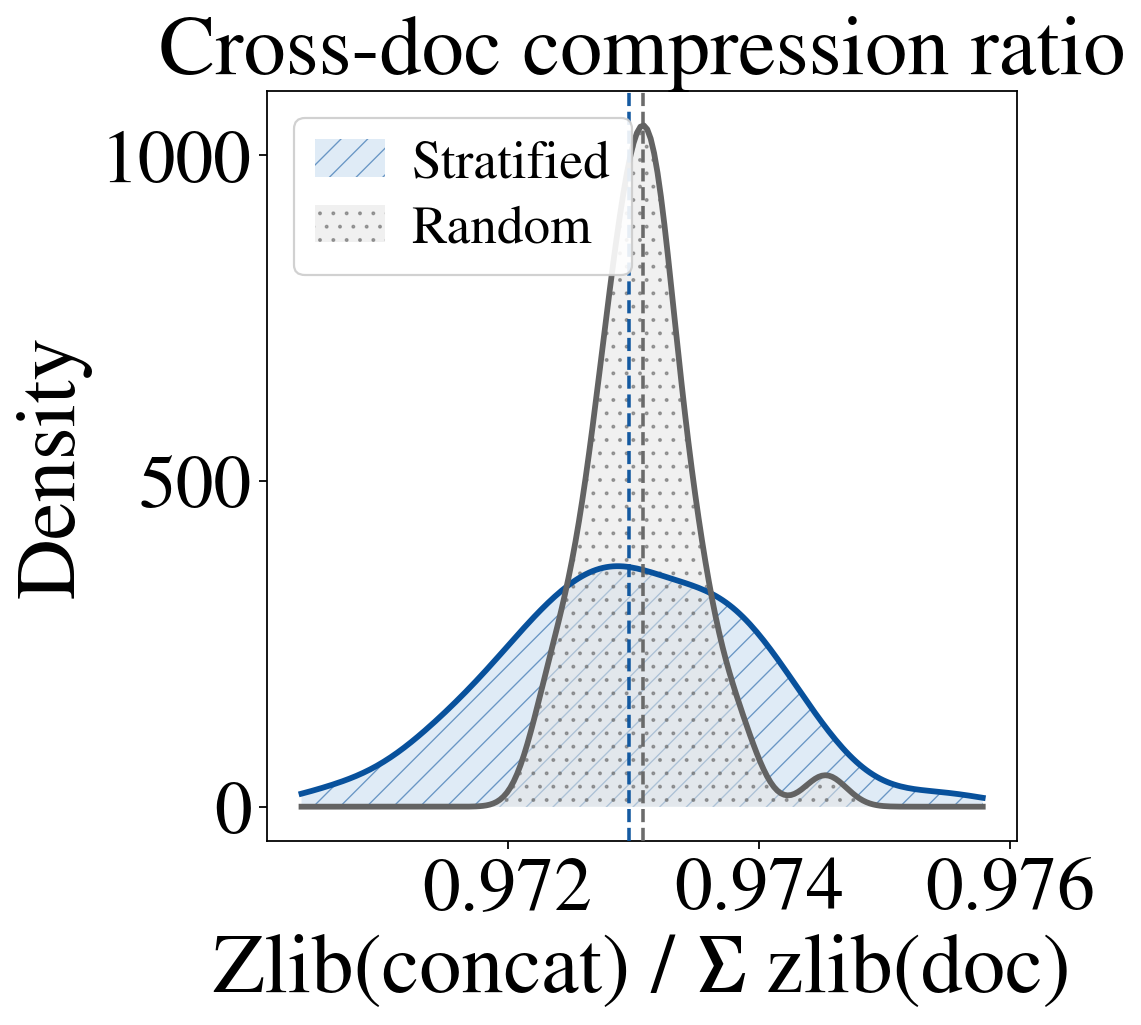}
    \caption{Compression ratio.}
    \label{fig:divmetric_repro_comp_ratio}
\end{subfigure}
\hfill
\begin{subfigure}[t]{0.24\textwidth}
    \centering
    \includegraphics[width=\linewidth]{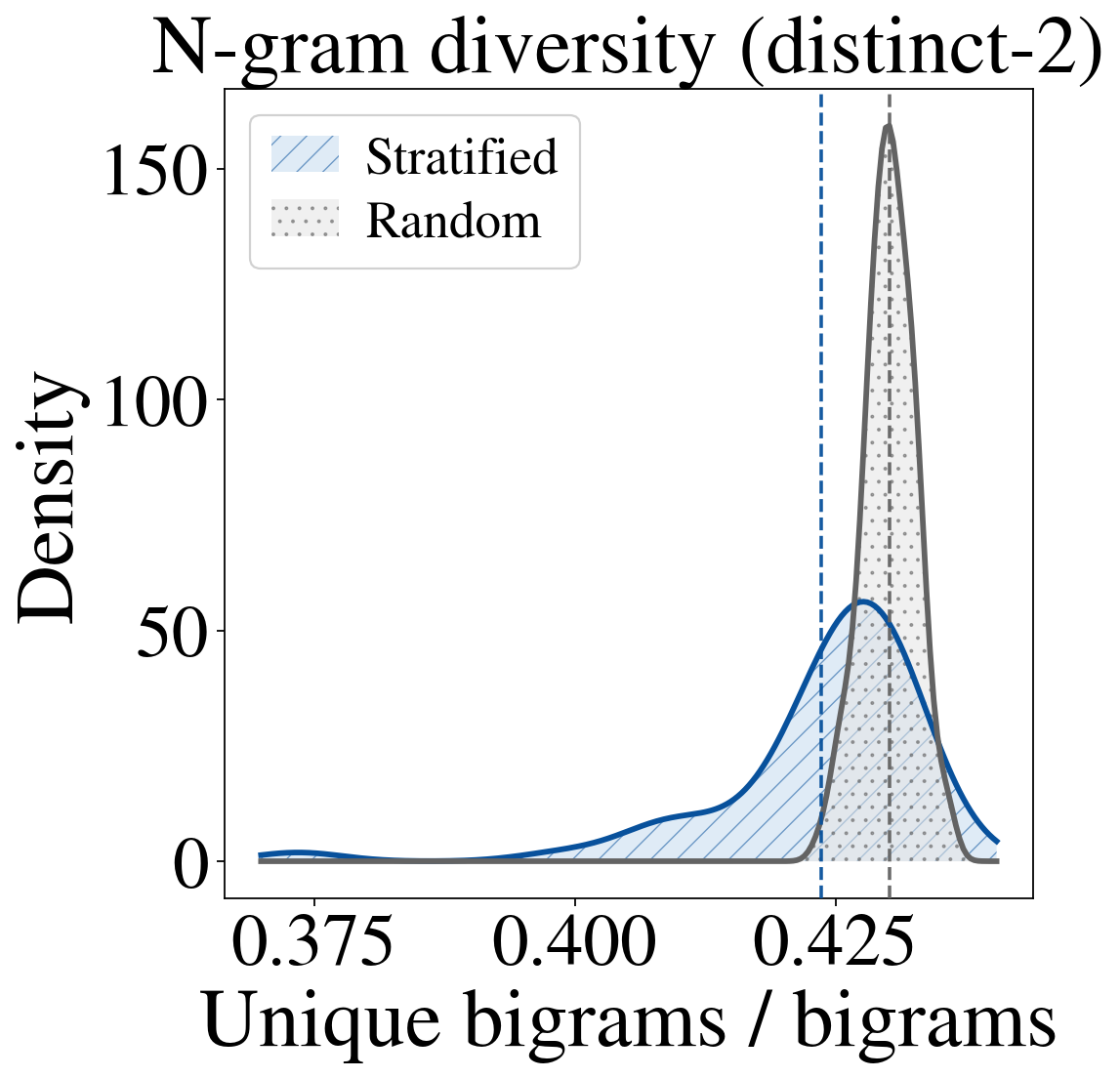}
    \caption{Distinct-2.}
    \label{fig:divmetric_repro_distinct_2}
\end{subfigure}
\hfill
\begin{subfigure}[t]{0.24\textwidth}
    \centering
    \includegraphics[width=\linewidth]{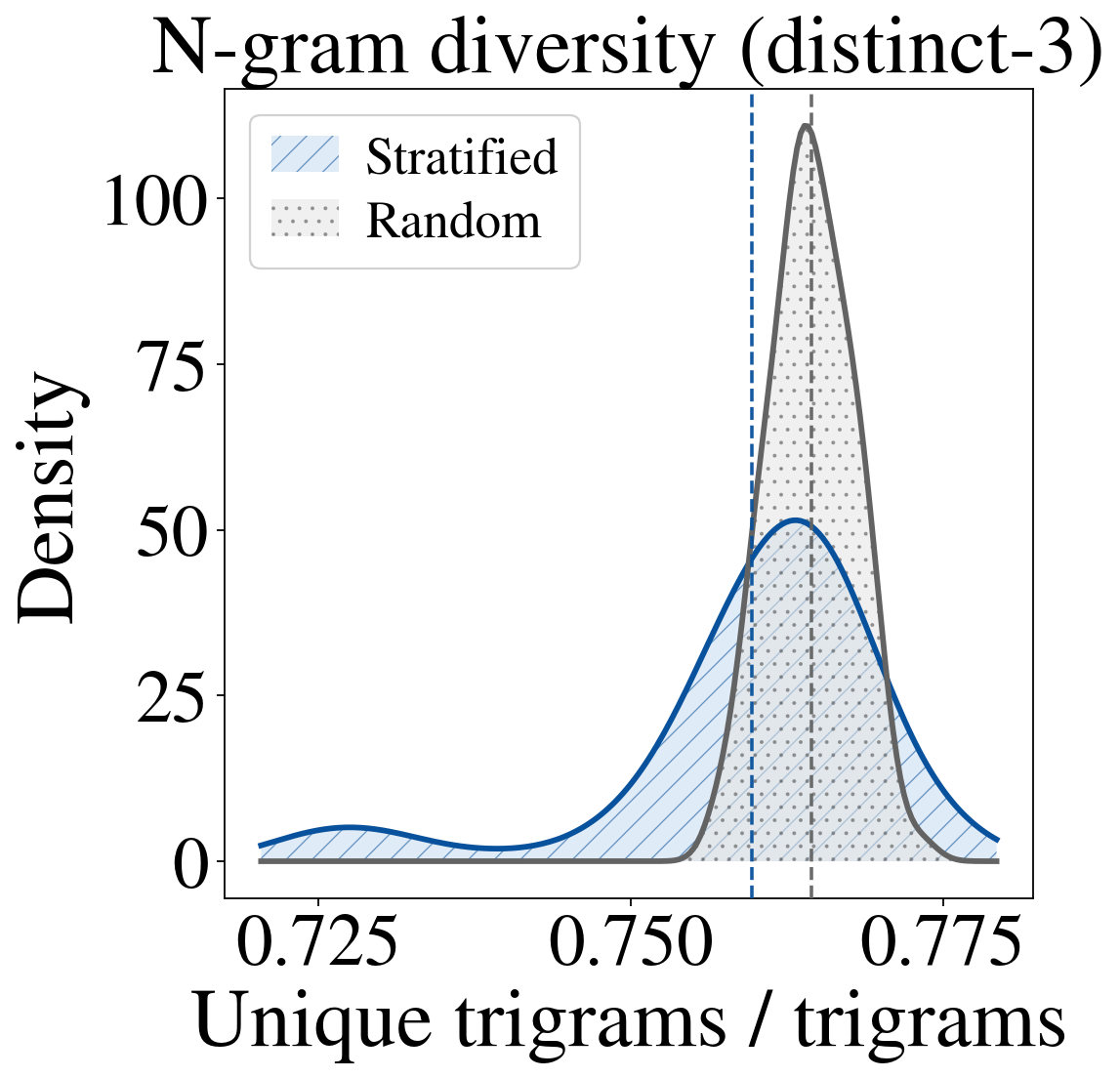}
    \caption{Distinct-3.}
    \label{fig:divmetric_repro_distinct_3}
\end{subfigure}

\vspace{0.5em}
\begin{subfigure}[t]{0.24\textwidth}
    \centering
    \includegraphics[width=\linewidth]{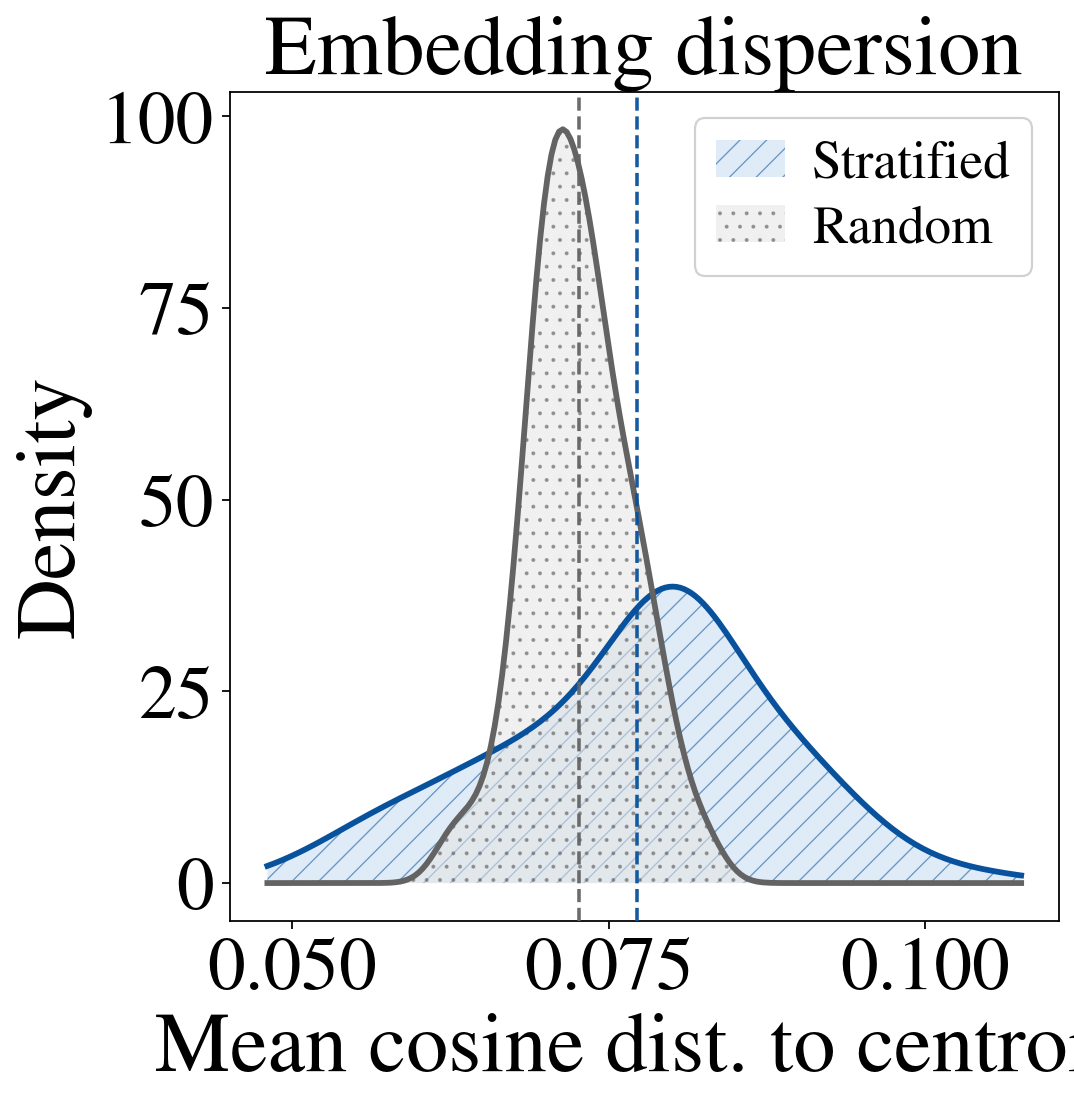}
    \caption{Embedding dispersion.}
    \label{fig:divmetric_repro_emb_dispersion}
\end{subfigure}
\hfill
\begin{subfigure}[t]{0.25\textwidth}
    \centering
    \includegraphics[width=\linewidth]{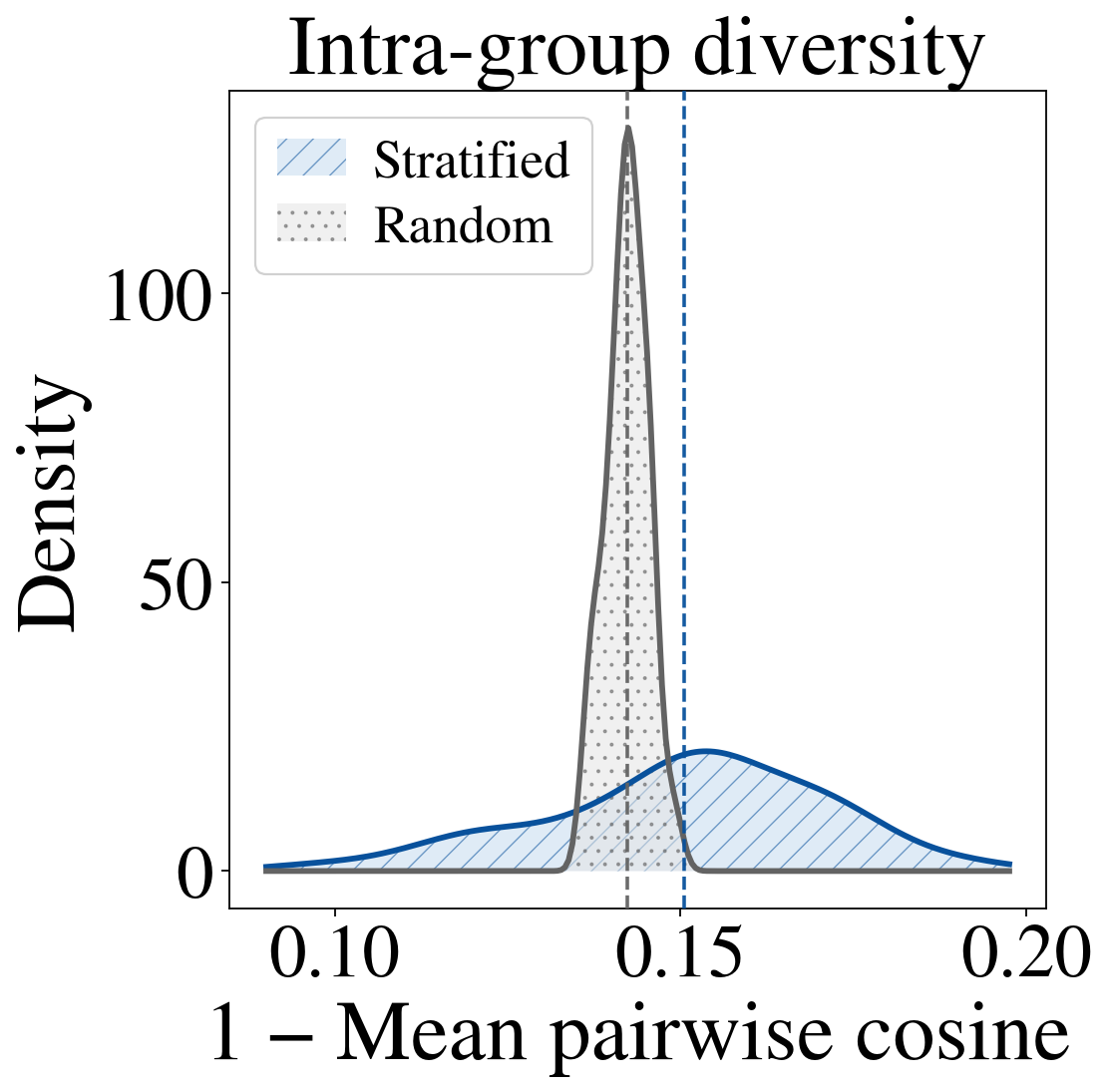}
    \caption{Intra-group diversity.}
    \label{fig:divmetric_repro_intragroup_diversity}
\end{subfigure}
\hfill
\begin{subfigure}[t]{0.24\textwidth}
    \centering
    \includegraphics[width=\linewidth]{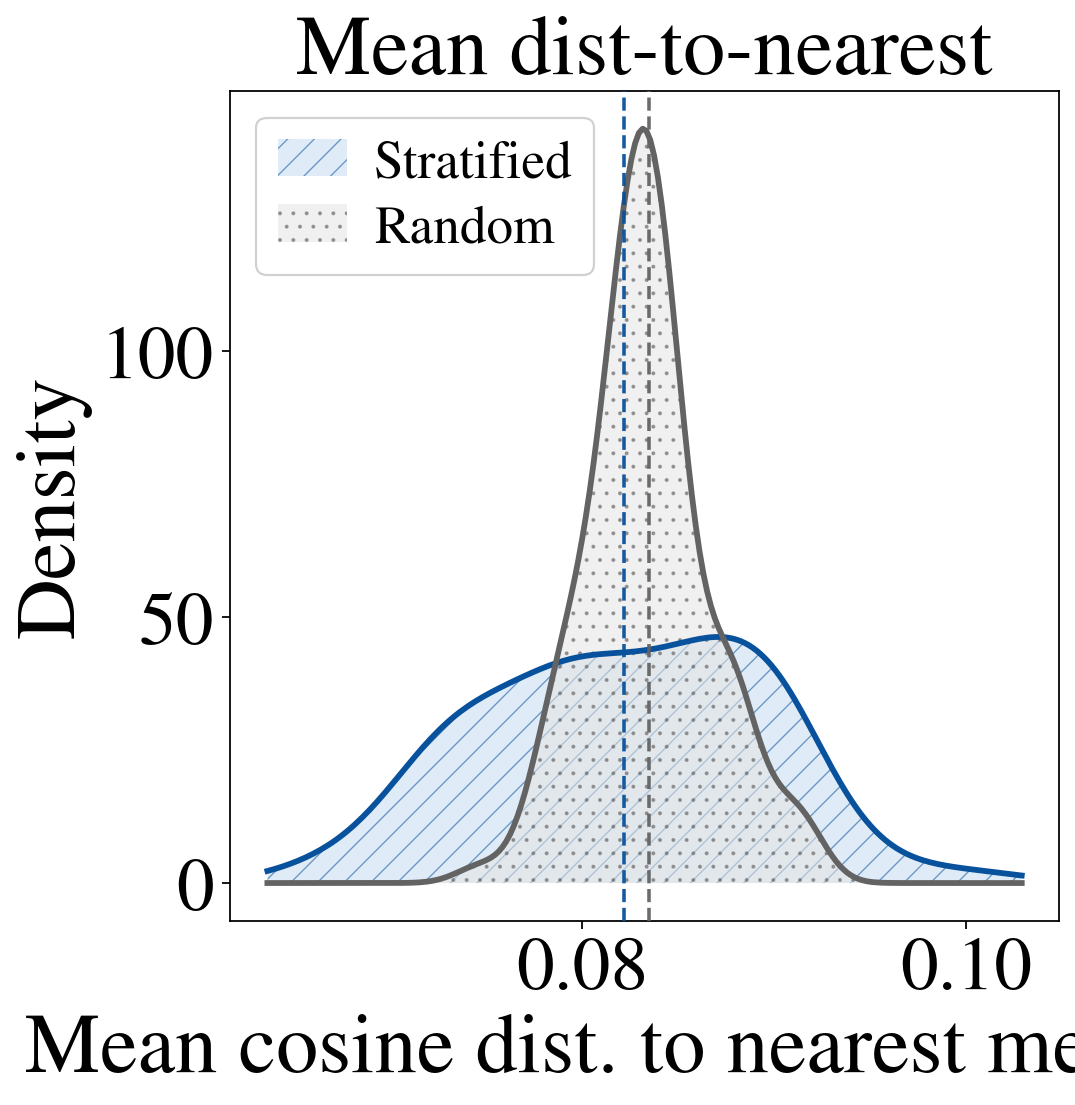}
    \caption{Mean distance to nearest.}
    \label{fig:divmetric_repro_mean_dist_to_nearest}
\end{subfigure}
\hfill
\begin{subfigure}[t]{0.23\textwidth}
    \centering
    \includegraphics[width=\linewidth]{images/background/diversity_metrics/repro/vendi.png}
    \caption{Vendi score.}
    \label{fig:divmetric_repro_vendi}
\end{subfigure}
\caption{Diversity metrics on our constructed groups versus randomly formed groups in the pre-training setting.}
\label{fig:diversity_metrics_repro}
\end{figure*}

\begin{figure*}[h]
\centering
\begin{subfigure}[t]{0.23\textwidth}
    \centering
    \includegraphics[width=\linewidth]{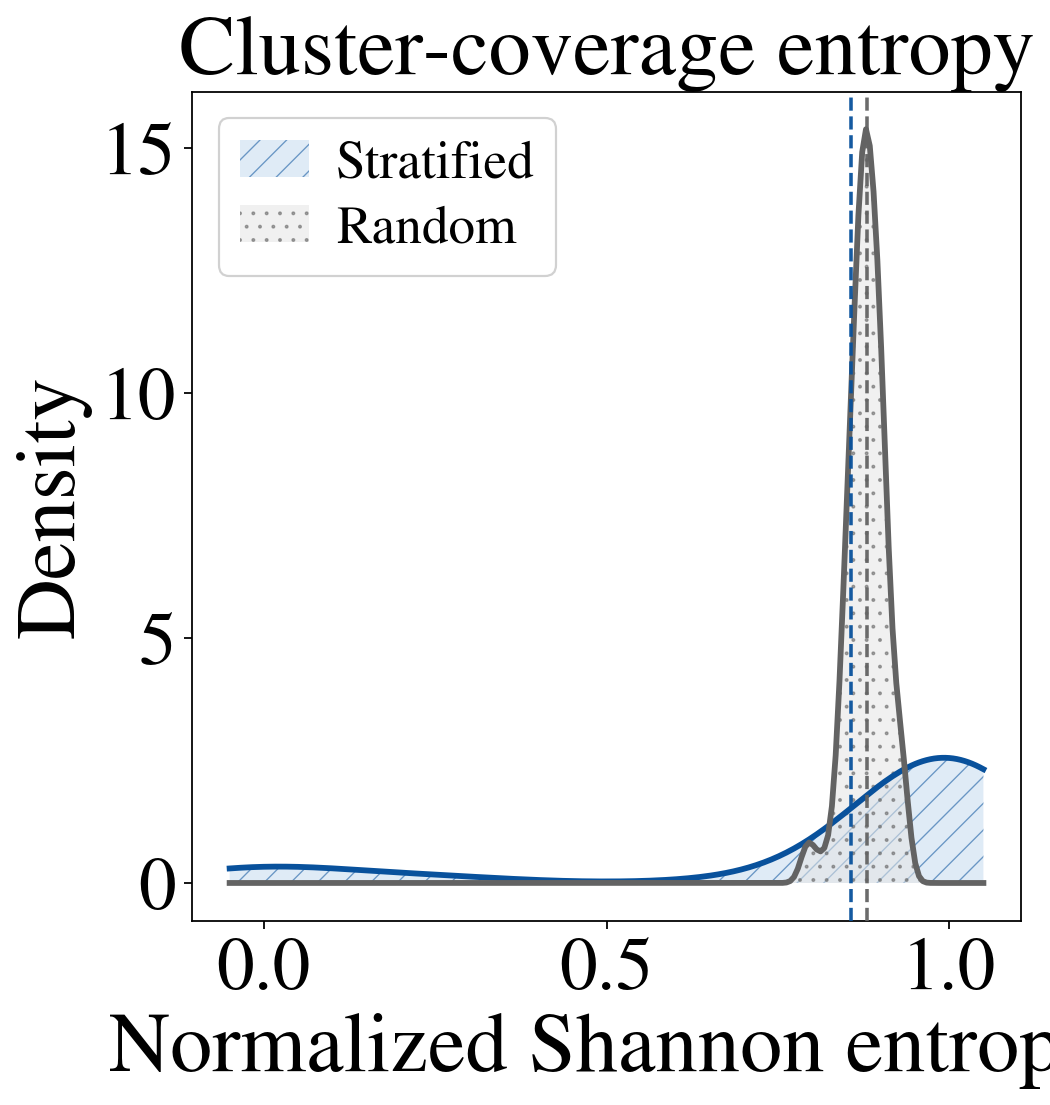}
    \caption{Cluster entropy.}
    \label{fig:divmetric_goose_cluster_entropy}
\end{subfigure}
\hfill
\begin{subfigure}[t]{0.24\textwidth}
    \centering
    \includegraphics[width=\linewidth]{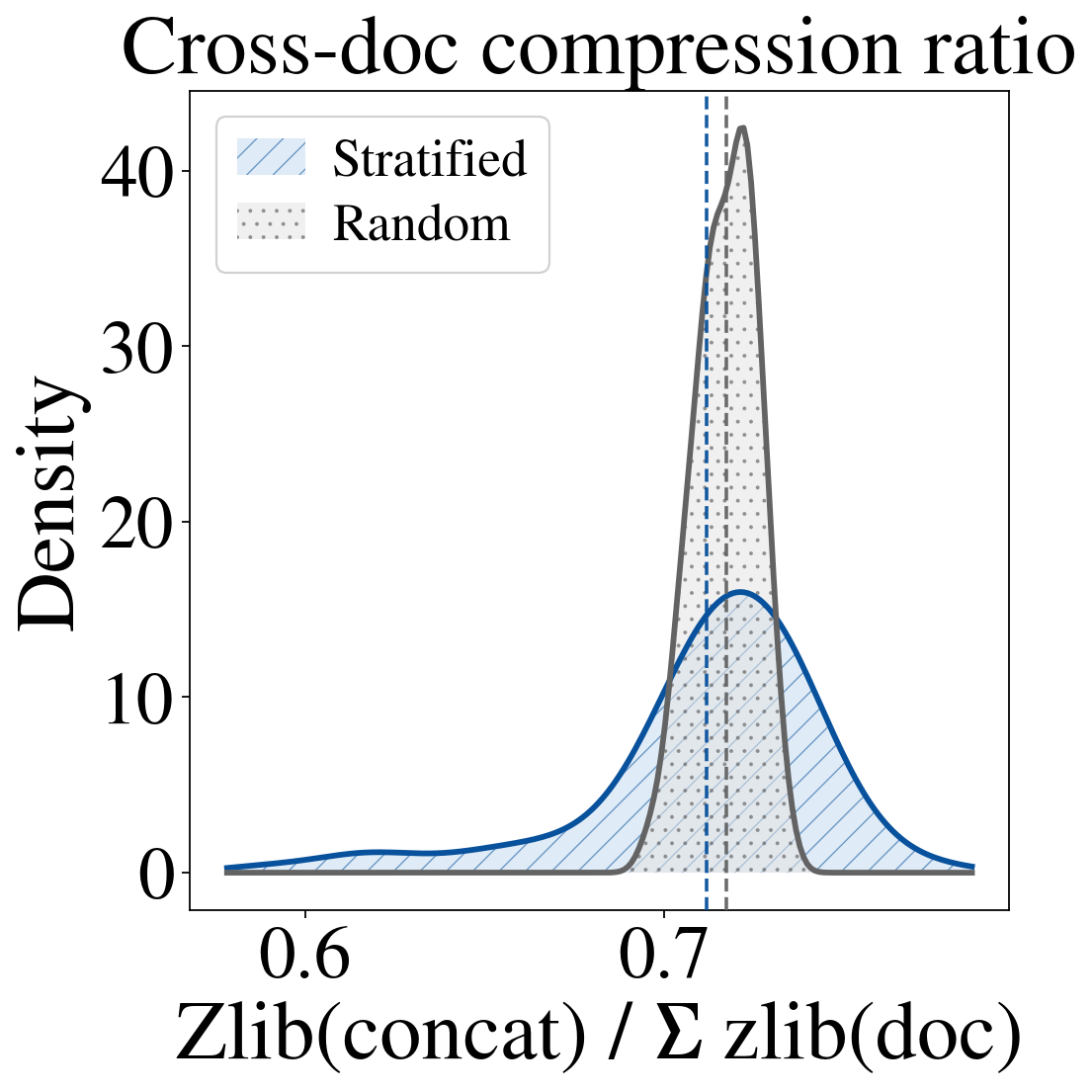}
    \caption{Compression ratio.}
    \label{fig:divmetric_goose_comp_ratio}
\end{subfigure}
\hfill
\begin{subfigure}[t]{0.24\textwidth}
    \centering
    \includegraphics[width=\linewidth]{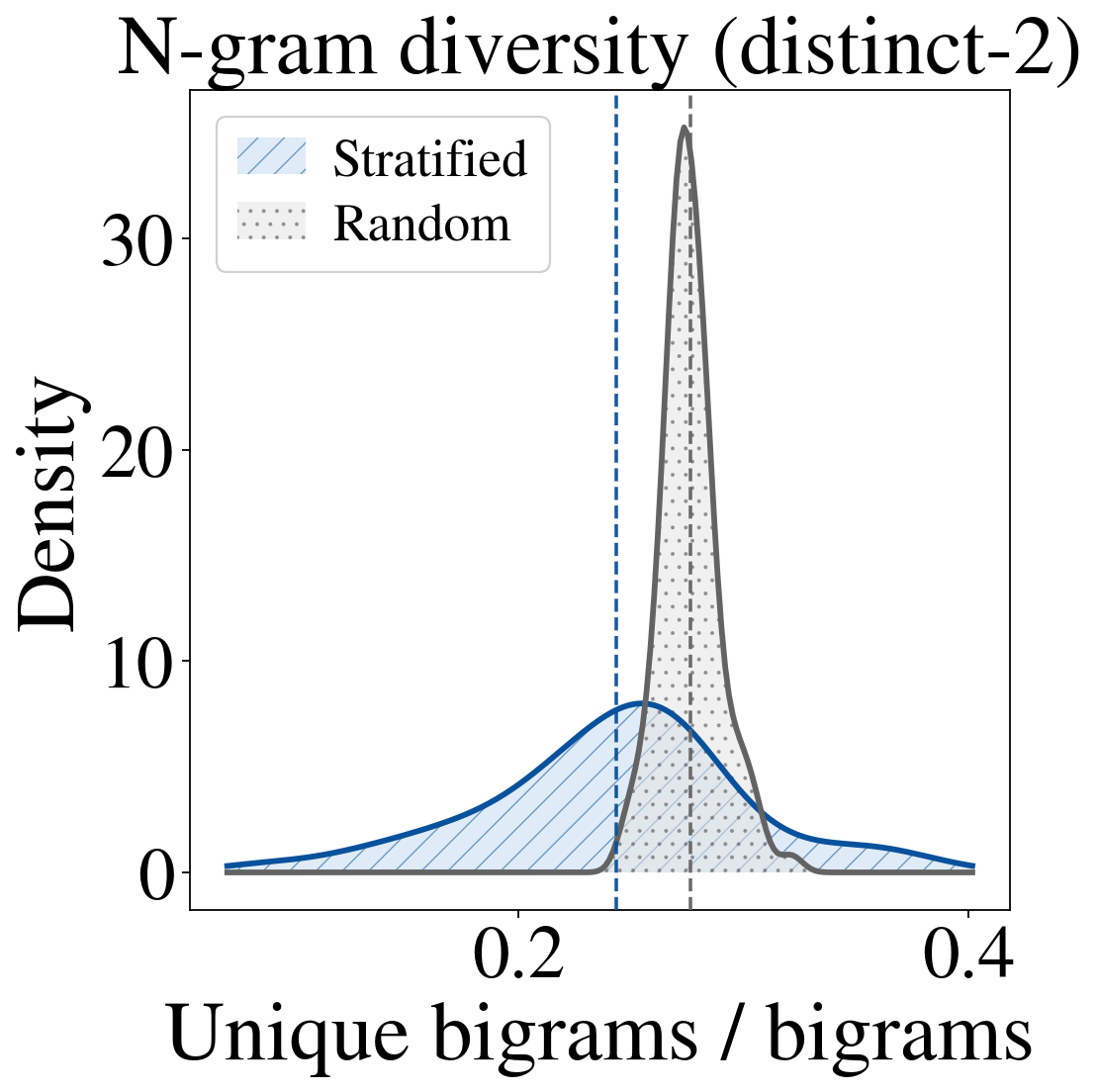}
    \caption{Distinct-2.}
    \label{fig:divmetric_goose_distinct_2}
\end{subfigure}
\hfill
\begin{subfigure}[t]{0.24\textwidth}
    \centering
    \includegraphics[width=\linewidth]{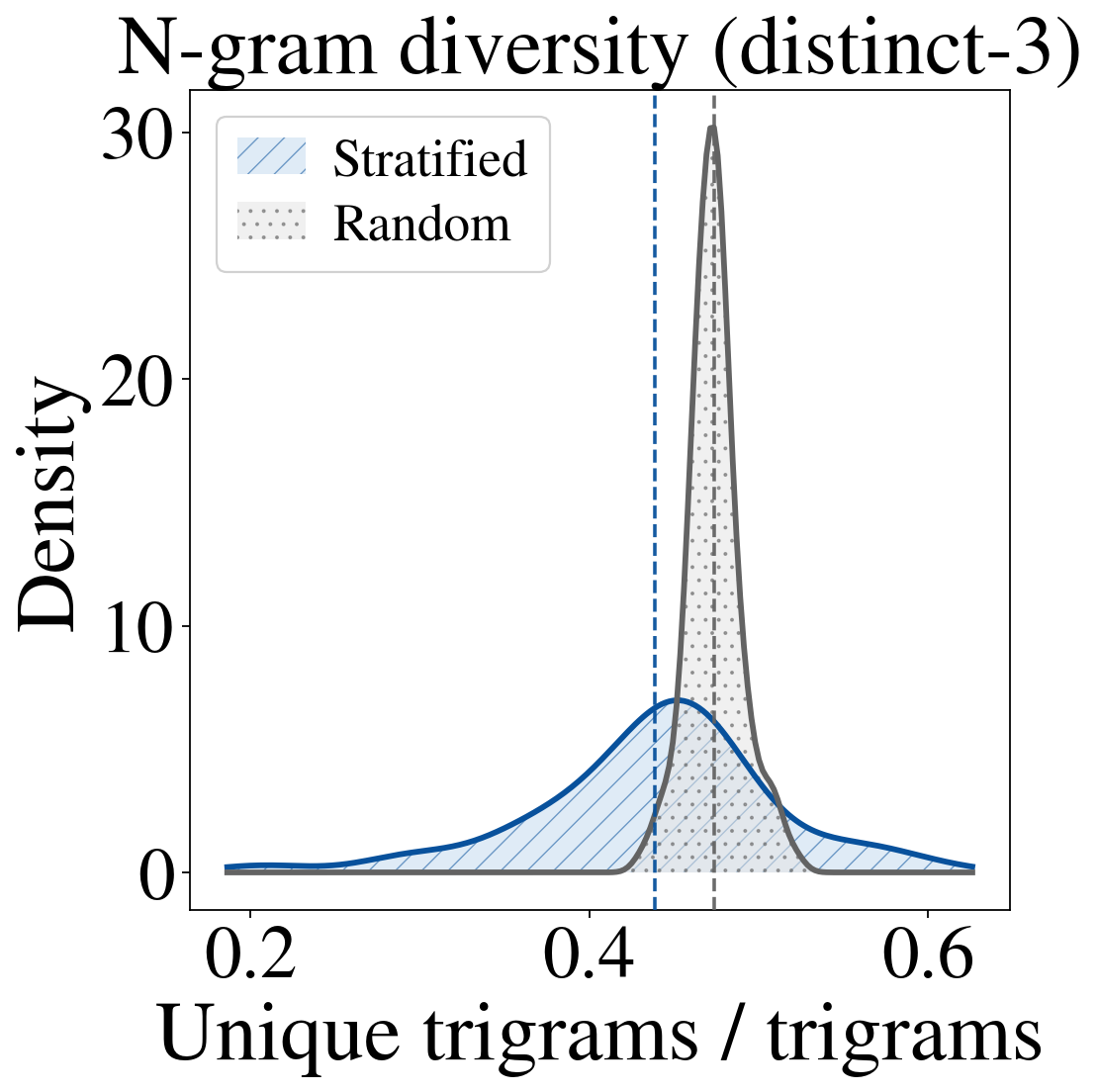}
    \caption{Distinct-3.}
    \label{fig:divmetric_goose_distinct_3}
\end{subfigure}

\vspace{0.5em}
\begin{subfigure}[t]{0.24\textwidth}
    \centering
    \includegraphics[width=\linewidth]{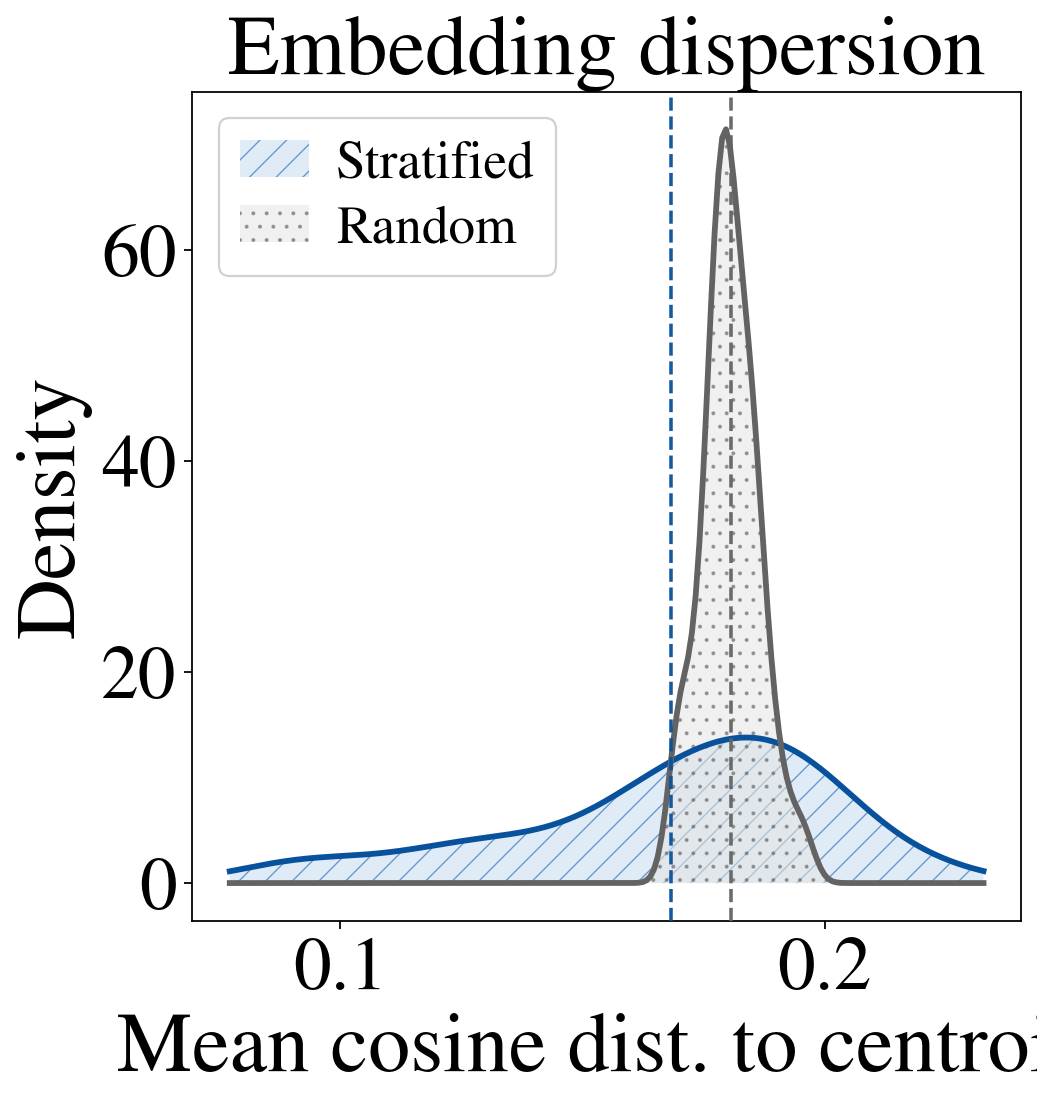}
    \caption{Embedding dispersion.}
    \label{fig:divmetric_goose_emb_dispersion}
\end{subfigure}
\hfill
\begin{subfigure}[t]{0.24\textwidth}
    \centering
    \includegraphics[width=\linewidth]{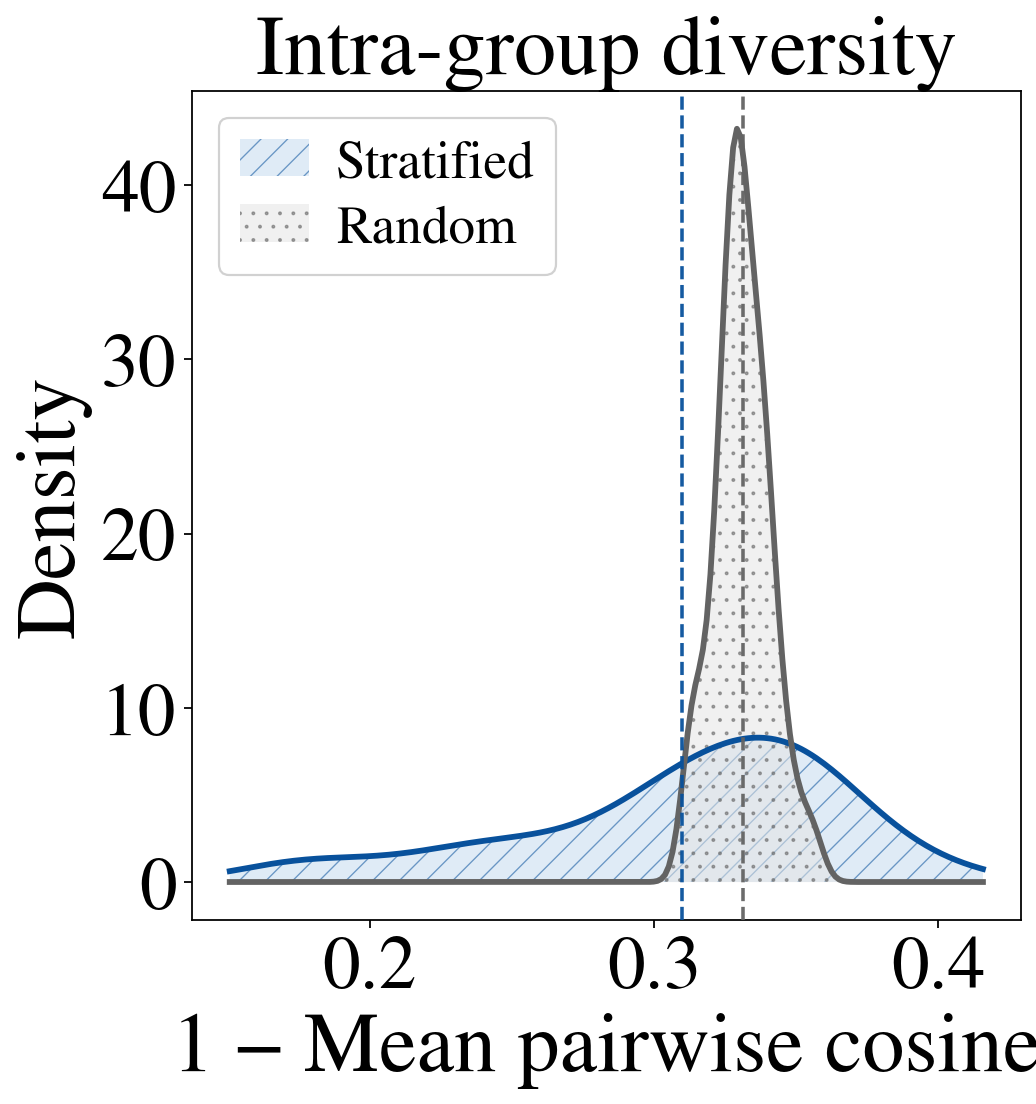}
    \caption{Intra-group diversity.}
    \label{fig:divmetric_goose_intragroup_diversity}
\end{subfigure}
\hfill
\begin{subfigure}[t]{0.24\textwidth}
    \centering
    \includegraphics[width=\linewidth]{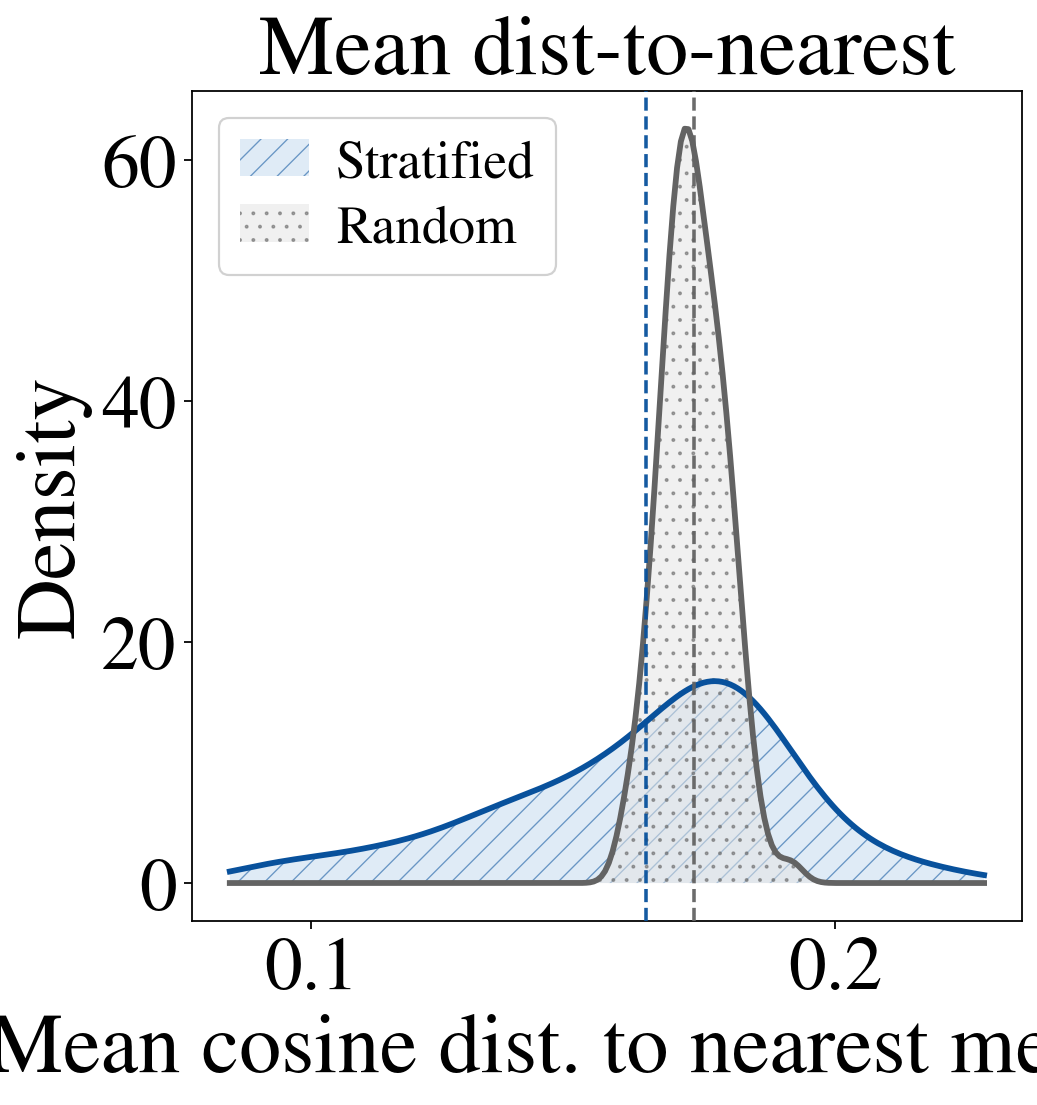}
    \caption{Mean distance to nearest.}
    \label{fig:divmetric_goose_mean_dist_to_nearest}
\end{subfigure}
\hfill
\begin{subfigure}[t]{0.24\textwidth}
    \centering
    \includegraphics[width=\linewidth]{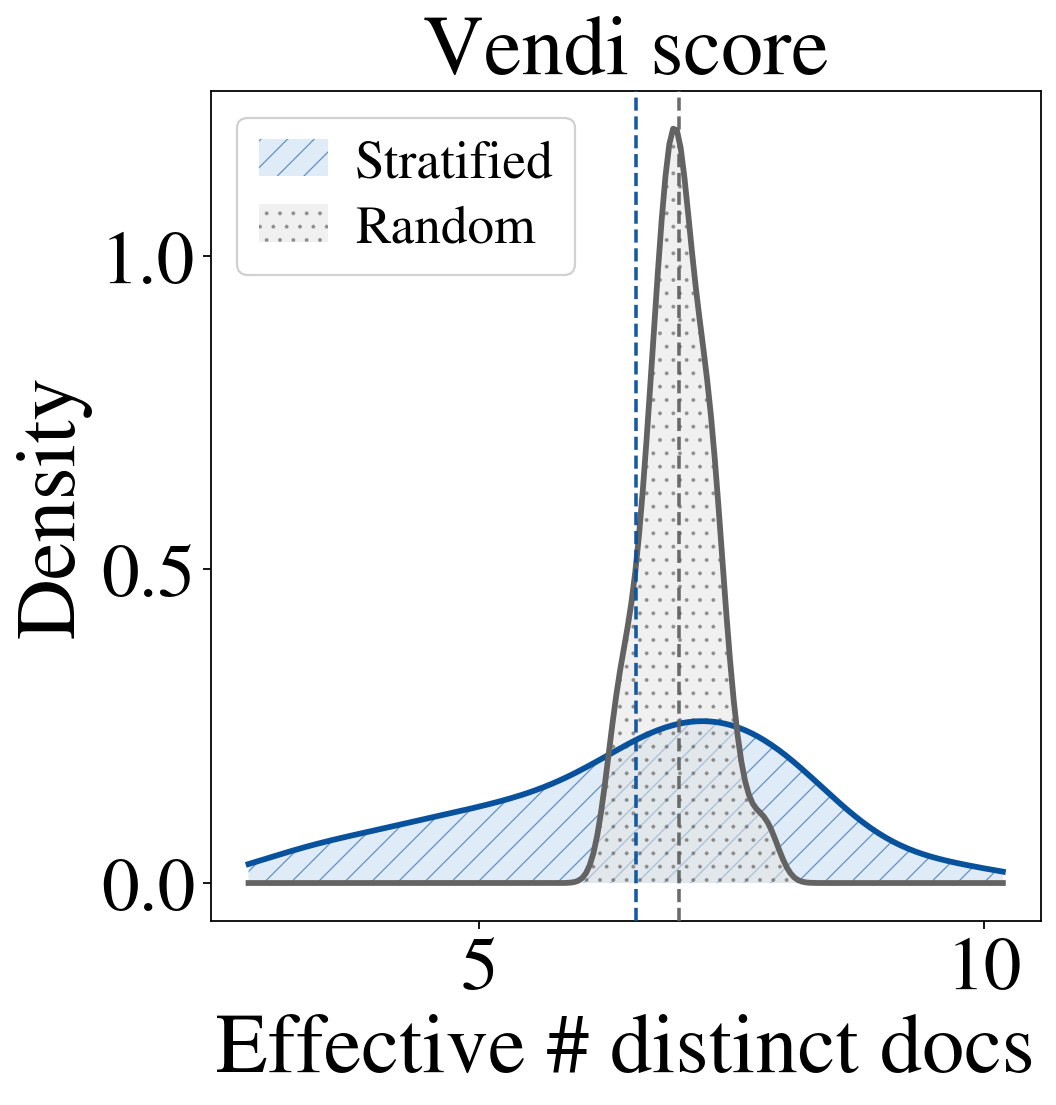}
    \caption{Vendi score.}
    \label{fig:divmetric_goose_vendi}
\end{subfigure}
\caption{Diversity metrics on our constructed groups versus randomly formed data groups in the post-training setting.}
\label{fig:diversity_metrics_goose}
\end{figure*}

\newpage

\subsection{Group Oracle Budget Analysis Plots}\label{appendix:budget_plots}
This section contains the full set of group-oracle budget allocation plots, summarized in Figure~\ref{fig:budget} of the main body.

\paragraph{Choice of diagnostic.} Figure~\ref{fig:allocation_heuristics_all} compares GradDiv against cheap diversity heuristics as the diagnostic used to decide which groups receive a group-oracle call.
\begin{figure}[h]
\centering
\begin{subfigure}[t]{0.24\linewidth}
    \centering
    \includegraphics[width=\linewidth]{images/budget/allocation_heuristics_repro_n5120_recall10.png}
    \caption{Pre-train, top-10\% recall.}
    \label{fig:allocation_heuristics_repro_n5120_recall10_appendix}
\end{subfigure}
\hfill
\begin{subfigure}[t]{0.24\linewidth}
    \centering
    \includegraphics[width=\linewidth]{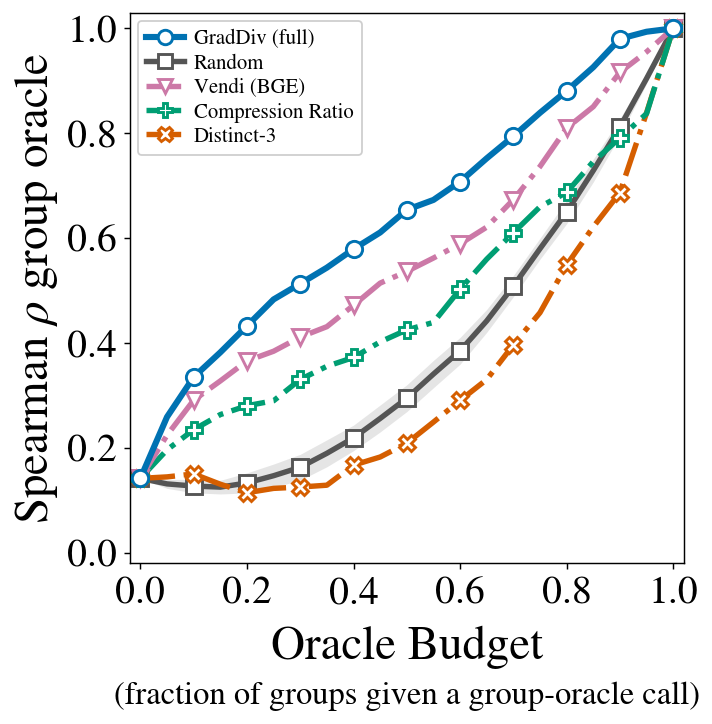}
    \caption{Pre-train, Spearman $\rho$.}
    \label{fig:allocation_heuristics_repro_n5120_spearman}
\end{subfigure}
\hfill
\begin{subfigure}[t]{0.24\linewidth}
    \centering
    \includegraphics[width=\linewidth]{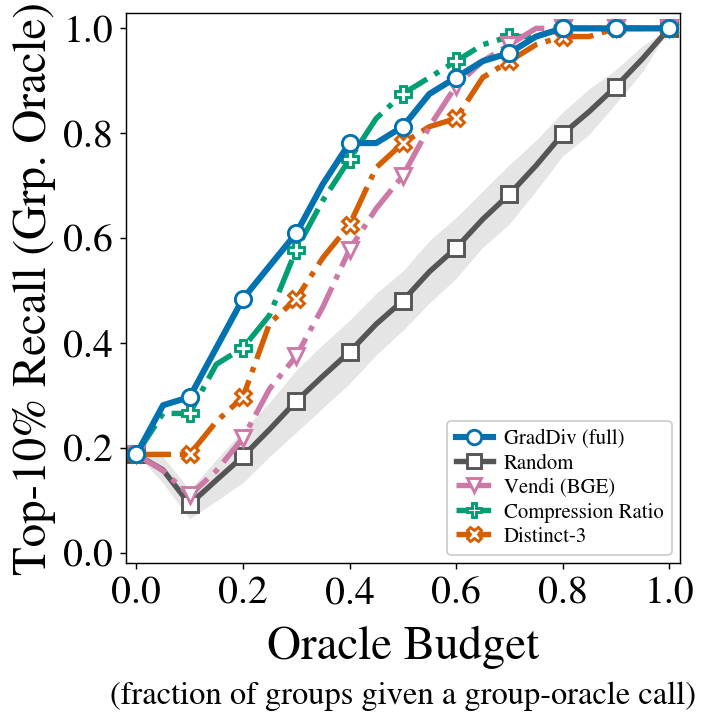}
    \caption{Post-train, top-10\% recall.}
    \label{fig:allocation_heuristics_goose_n512_recall10}
\end{subfigure}
\hfill
\begin{subfigure}[t]{0.24\linewidth}
    \centering
    \includegraphics[width=\linewidth]{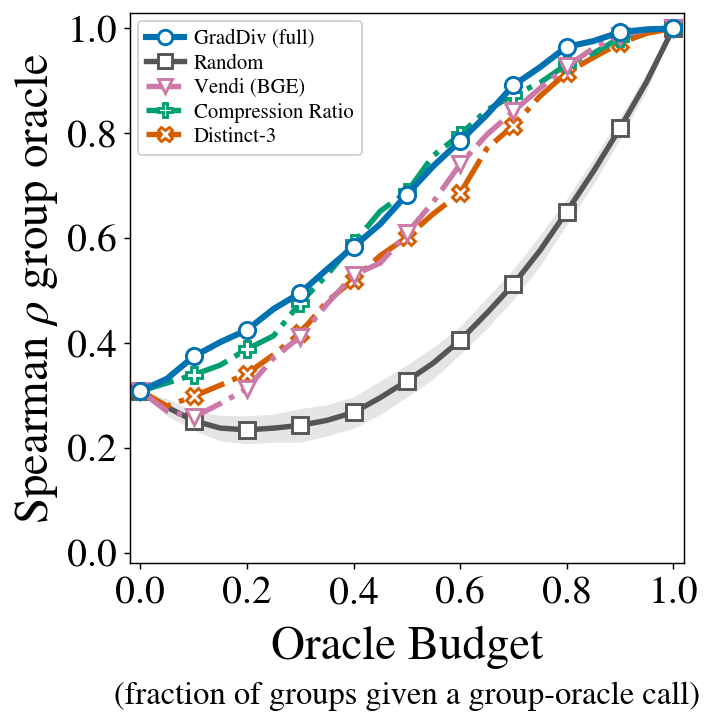}
    \caption{Post-train, Spearman $\rho$.}
    \label{fig:allocation_heuristics_goose_n512_spearman}
\end{subfigure}
\caption{Choice of diagnostic for group-oracle budget allocation, for pre-train and post-train data (group sizes of 5120 and 512 samples, respectively).}
\label{fig:allocation_heuristics_all}
\end{figure}

\paragraph{Gradient subsampling.} Figure~\ref{fig:allocation_fraction_all} shows the effect of subsampling gradients when computing the GradDiv diagnostic.
\begin{figure}[h]
\centering
\begin{subfigure}[t]{0.24\linewidth}
    \centering
    \includegraphics[width=\linewidth]{images/budget/allocation_fraction_repro_n5120_recall10.png}
    \caption{Pre-train, top-10\% recall.}
    \label{fig:allocation_fraction_repro_n5120_recall10_appendix}
\end{subfigure}
\hfill
\begin{subfigure}[t]{0.24\linewidth}
    \centering
    \includegraphics[width=\linewidth]{images/budget/allocation_fraction_repro_n5120_spearman.png}
    \caption{Pre-train (5120), Spearman $\rho$.}
    \label{fig:allocation_fraction_repro_n5120_spearman_appendix}
\end{subfigure}
\hfill
\begin{subfigure}[t]{0.24\linewidth}
    \centering
    \includegraphics[width=\linewidth]{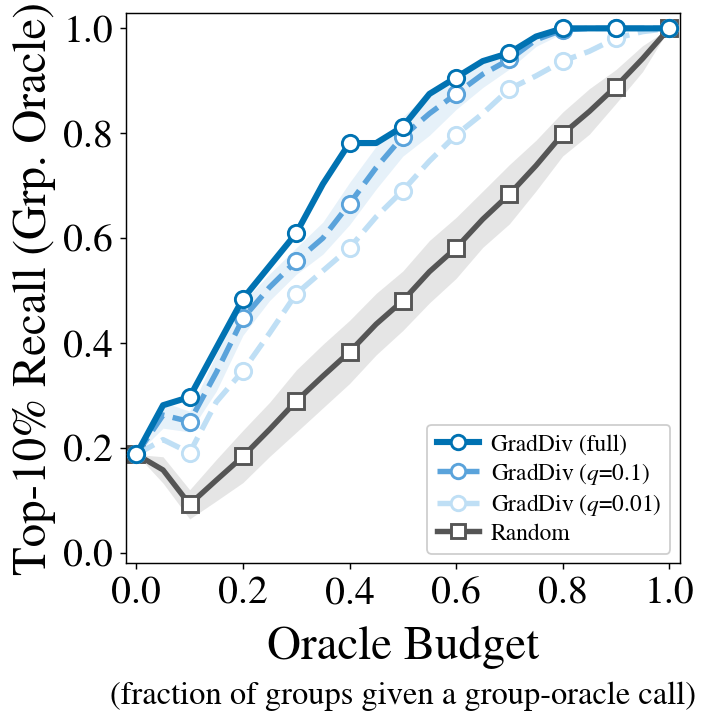}
    \caption{Post-train, top-10\% recall.}
    \label{fig:allocation_fraction_goose_n512_recall10}
\end{subfigure}
\hfill
\begin{subfigure}[t]{0.24\linewidth}
    \centering
    \includegraphics[width=\linewidth]{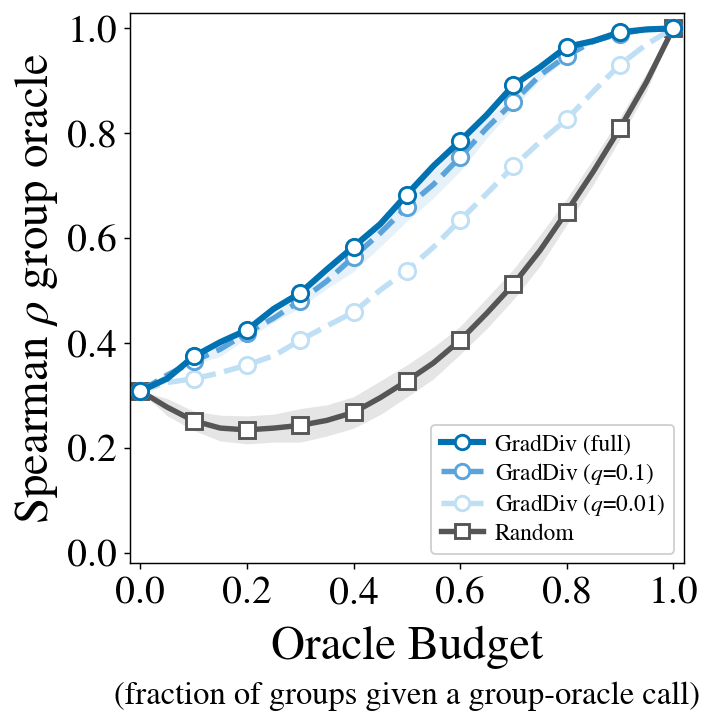}
    \caption{Post-train, Spearman $\rho$.}
    \label{fig:allocation_fraction_goose_n512_spearman}
\end{subfigure}
\caption{Effect of gradient subsampling on group-oracle budget allocation, comparing the full GradDiv diagnostic against subsampled variants ($q = 0.1$, $q = 0.01$) for pre-train and post-train data (group sizes of 5120 and 512 samples, respectively).}
\label{fig:allocation_fraction_all}
\end{figure}

\newpage

\end{document}